\documentclass{article} 
\usepackage{iclr2027_conference,times}

\usepackage{amsmath,amsfonts,bm}

\def\eqref#1{equation~\ref{#1}}

\def\1{\bm{1}}

\DeclareMathAlphabet{\mathsfit}{\encodingdefault}{\sfdefault}{m}{sl}
\SetMathAlphabet{\mathsfit}{bold}{\encodingdefault}{\sfdefault}{bx}{n}

\usepackage{hyperref}

\hypersetup{
	colorlinks=true,
	citecolor=blue!65!black!55,
	linkcolor=blue!55!black!65,
	urlcolor=purple!50!black!80,
}

\usepackage{url}

\usepackage{microtype}
\usepackage{graphicx}
\usepackage{subcaption}
\usepackage{booktabs}
\usepackage{breqn}
\usepackage{algorithm}
\usepackage{algorithmic}
\usepackage[table]{xcolor}
\usepackage{multirow}
\usepackage{textcomp}
\usepackage{amsmath}
\usepackage{amssymb}
\usepackage{mathtools}
\usepackage{amsthm}
\usepackage[capitalize,noabbrev]{cleveref}
\usepackage{colortbl} 

\theoremstyle{plain}
\newtheorem{theorem}{Theorem}[section]
\newtheorem{proposition}[theorem]{Proposition}
\newtheorem{lemma}[theorem]{Lemma}
\newtheorem{corollary}[theorem]{Corollary}
\theoremstyle{definition}
\newtheorem{definition}[theorem]{Definition}
\newtheorem{assumption}[theorem]{Assumption}
\theoremstyle{remark}

\title{UTA: Universal Training Adaptive Learning Rate Decay}

\author{Anda Tang \\
the State Key Lab of General AI\\
School of Intelligence Science and Technology\\
Peking University\\
Beijing, China\\
\texttt{tanganda@pku.edu.cn} \\
\And
Junjie Wang \\
the State Key Lab of General AI\\
School of Intelligence Science and Technology\\
Peking University\\
Beijing, China\\
\texttt{wangjunjie25@stu.pku.edu.cn} \\
\And
Xuanlin Yang\\
the State Key Lab of General AI\\
School of Intelligence Science and Technology\\
Peking University\\
Beijing, China\\
\texttt{2401112159@stu.pku.edu.cn}
\And
Cong Fang \\
the State Key Lab of General AI\\
School of Intelligence Science and Technology\\
Peking University\\
Beijing, China\\
\texttt{fangcong@pku.edu.cn}
\And
Zhouchen Lin \thanks{Corresponding author.} \\
the Pazhou Laboratory (Huangpu)\\
the State Key Lab of General AI\\
School of Intelligence Science and Technology\\
the Institute for Artificial Intelligence\\
Peking University\\
Beijing, China\\
\texttt{zlin@pku.edu.cn}
}

\iclrfinalcopy 
\begin{document}

\maketitle

\begin{abstract}
	Learning rate decay (LRD) is critical for training high-performance models, yet most designs remain empirically driven. Although the Universal Budget Aware (UBA) LRD offers a theoretical framework, its reliance on SGD dynamics limits its applicability to modern Adam-type optimizers and leaves its key `optimization difficulty parameter' opaque. To bridge this gap, we propose the Universal Training Adaptive (UTA) LRD.  
	Unlike UBA's fixed offline trajectory anchored solely to iterations, UTA treats LRD as a closed-loop shape-and-scale problem. By continuous-time formulation, UTA is training-state adaptive with segment-wise optimal shape. This adaptivity links external iterations with internal training states, enabling a near-optimal path in both progression shape and learning rate (LR) scale, a property  unattainable by other LRDs, including UBA. 
	Empirically, we demonstrate consistent improvements across language models (LLaMA 60M to 1B)  and vision models (ResNet50, ViT-B) on ImageNet with plug-and-play ease and negligible overhead. Extensive ablation studies further confirm UTA's robustness across dimensions, including peak LR, final decay ratio, its parameter $\varphi_0$, component efficacy and generalization to fine-tuning scenarios, underscoring its practical implementability.
\end{abstract}

\section{Introduction}

The learning rate decay (LRD) plays  a fundamental role in
training high-performance deep learning models \citep{tang2026stepsize,tissue2025scaling,li2025functional}.
Its design directly governs the convergence behavior, generalization capability, and final performance of models across a wide spectrum of tasks and architectures. 
Notably, the influence of LRDs becomes even more pronounced in resource-constrained or time-sensitive scenarios, where the number of training iterations is limited \citep{hagele2024scaling,tang2026stepsize}. 
These advantages underscore the practical importance of LRDs not only as a tuning parameter but as an essential lever for sustainable and scalable model development.

Research in this area has evolved along two main lines. 
A line of work \citep{he2016deep,smith2017cyclical,smith2019super,li2019budgeted,chen2022rex,pan2022extremebert,zhai2022scaling,hagele2024scaling,hu2024minicpm} focus on designing LRDs primarily through empirically driven approaches, effectively enhancing accuracy.
Another line of research is theoretically driven, mainly seeking to establish frameworks and scaling laws for LRDs to provide theoretical explanation for phenomena \citep{pan2022eigencurve,defazio2023optimal,tissue2025scaling,li2025functional,luo2025multi,tang2026stepsize}. 
Notably, Tang et al. propose a theoretically grounded LRD called UBA  \citep{tang2026stepsize} to narrow this gap. This framework adopts a min‑max formulation that accounts for robustness to variations in loss landscape curvature, thereby guaranteeing performance under worst‑case conditions. The UBA LRD mitigates heuristic design while maintaining strong and consistent performance across diverse tasks, architectures, model scales, and training budgets. Consequently, it removes the costly trial‑and‑error process typically required to select an appropriate LRD for a given network and training setup.

Despite its merits, UBA remains confined to its design framework which is analyzed in the context of Stochastic Gradient Descent (SGD). It raises questions about its generalization to adaptive optimizers such as Adam \citep{kingma2014adam}. 
Moreover, although UBA establishes a connection to optimization difficulty via the local Hessian, its pivotal parameter $\varphi$ is itself invariant to the internal dynamics of networks. Thus it lacks real‑time feedback from the training state.
This disconnect between the predetermined LRD and the dynamic training trajectory represents a key limitation in current practice.
These limitations prompt a natural question: 
\par
{\centering  \textit{Can we explain the pivotal parameter $\varphi$ in universal budget aware LRD, further closing the gap between principle and practice?} \par
} 
In this work, we take a step toward this goal by extending the theoretical foundation of UBA to the Adam optimizer, aiming to combine theoretical soundness with practical, state-aware adaptability.

To achieve this, we propose the \textbf{U}niversal \textbf{T}raining LR under \textbf{A}daptive moment estimation (UTA) LRD.
First, it is theoretically grounded, built upon a principled robust optimization framework.
Moreover, it is dynamically training adaptive, adjusting stepwise learning rates (LRs) based on the real-time optimizer states during training. 
This stands in contrast to LRDs that predefine a fixed curve for LR variation. 
Additionally, it is inherently optimizer-agnostic. Although this paper focuses on its implementation with the Adam, the underlying framework generalizes to other optimizers.

To solve this extended optimization problem, we introduce a segment-wise continuous analysis. This approach offers three advantages. First, it yields the optimal shape of LRD curve. Second, it naturally gives rise to the concept, intrinsic time, defined as the cumulative LR effort. 
This quantity establishes a mapping between the external world of physical iteration count and an internal world of model optimization, serving as a more faithful measure of training progress.
The continuous analysis further describes the energy evolution of optimization progress along intrinsic time, which gives the dynamic setting of the proposed LRD. Consequently, the derived LRD is near‑optimal in shape and energy‑calibrated in scale.
Practically, the proposed LRD is plug‑and‑play, requiring minimal modification to existing model architectures or training pipelines, thus enabling easy adoption across diverse deep‑learning frameworks.

Empirically, UTA delivers consistent improvements across both language (LLaMA-60M to 1B) and vision (ResNet50, ViT-B on ImageNet) tasks. Moreover, comprehensive ablation studies substantiate its practical advantages: (i) robustness to  hyper-parameters including peak LR, final decay ratio, and $\varphi_0$; (ii) the effect of its state-aware mechanism; and (iii) effective generalization to parameter-efficient fine-tuning scenarios.

\textbf{Main contributions}:
\begin{itemize}
	\item We propose UTA, a novel generalized LRD framework, which extends the principled UBA framework from SGD to adaptive optimizers, such as AdamW, thereby overcoming a key limitation of prior work that restricts its applicability. 
	\item We introduce the concept of intrinsic time as a unified measure of training progress, establishing a mapping from external iteration counts to internal optimization dynamics. This perspective not only reveals the fundamental nature of UBA's pivotal parameter $\varphi$ by identifying it as the scale of intrinsic time, but also provides a principled foundation for dynamically determining the LRs.
	\item We construct a theoretically grounded mechanism that combines a segment-wise analysis and an intrinsic-time-based energy evolution, coupling the optimal LRD shape with real-time optimization states. 
	\item We conduct extensive experiments across language (LLaMA-60M to 1B) and vision (ResNet50, ViT-B on ImageNet) tasks, along with comprehensive ablations on hyper-parameter sensitivity, dynamic component efficacy, automatic LR methods, and fine-tuning transferability. These results confirm UTA's consistent performance gains, robustness, and plug-and-play deployability with negligible overhead.
\end{itemize}
The remainder of this paper is organized as follows. Section \ref{Preliminary1} is the preliminary. Section \ref{Methodology} presents the methodology.  
Section \ref{Experiments} reports experimental results. Section \ref{conclusion} concludes with a discussion and directions for future work.
A comprehensive review of related work is provided in Appendix \ref{related-work-appendix}.
Theoretical properties is provided in Appendix~\ref{Theoretical Analysis}.

\section{Preliminary}\label{Preliminary1}
\paragraph{Robust optimization of learning rate.}
Let $f(W,\xi)$ denote a neural network parameterized by parameters $W\in \mathbb{R}^N$ on the data $\xi$. $\xi$ denotes the data sampling on the dataset $D$,  and let $\mathcal{F}$ be a function class that incorporates the variation in multiple factors, such as the network structures and data sets. Let $\mathcal{L}$ be the loss function, $T$ be a maximum number of training iterations, and $t$ be the current learning step. 
UBA \citep{tang2026stepsize} establishes a robust optimization framework for learning rate (ROLR) under SGD, formulated as:
\begin{equation}\label{eq1UBA}
	\begin{aligned}
		\min_{\eta_1  , \eta_2  ,..., \eta_{T-1}} & \max_{f\in \mathcal{F} }  \quad \mathcal{L}(f(W_T,\xi)),\\
		s.t. \quad &W_{t+1}=W_{t}-\eta_t\nabla \mathcal{L}(f(W_{t},\xi)),  \\
	\end{aligned}
\end{equation} 
where $t=0,1,2,\cdots,T-1$.
This min‑max formulation guarantees performance in worst‑case scenarios. 
However, this formulation constrains the feasible region to the standard gradient update process, limiting its direct compatibility with adaptive optimizers like Adam.

\paragraph{UBA-equivalent family.} This family, as defined in Definition \ref{definition:UBA-family}, encapsulates the optimal LRDs for gradient descent under a robust analysis \citep{tang2026stepsize}. Specifically, considering the $k$-th segment $ \left[ T_k, T_{k+1} \right]$  with length $\Delta T_k$, the optimal learning rate at the $t$-th step is given by the following instance of the UBA-equivalent family $\eta_{t}:=g_{a,b}(\varphi,\theta_{t,k})$, where $\theta_{t,k} :=\left( \left[ 2\left(t-T_{k}\right) -1\right] \pi/2\left(T_{k+1}-T_{k}\right)\right) +\left(k-1\right)\pi ,
$ 
the range parameters $a = \eta_{\min}^{(k)}$ and $b = \eta_{\max}^{(k)}$ correspond to the minimal and the maximal step sizes, respectively. $(k-1)\pi$ introduces a phase shift across segments. The parameter $\varphi$ controls the curvature of the LRD, related to optimization difficulty by the Hessian of local minima. However, $\varphi$ is static and expertise prior-driven. 
\begin{definition}[UBA-equivalent family]\label{definition:UBA-family}
	For a curvature parameter $\varphi>0$, range parameters $b,a\geq 0$, a step index $t$ and a period parameter $k$, the \emph{UBA-equivalent family} is defined as 
	\begin{equation}\label{eq:uba-normalized-shape}
		g_{a,b}(\varphi,\theta_{t,k})
		:=b+a	\frac{2\left[1 + \cos(\theta_{t,k})\right]}
		{2\varphi + (2-\varphi)\left[1 + \cos(\theta_{t,k})\right]}.
	\end{equation}
\end{definition}

\section{Methodology}\label{Methodology}

\subsection{Problem statement}\label{section31}


We extend the ROLR \eqref{eq1UBA} to a generalized update rule governed by optimizer-specific dynamics. 
The global min-max view is
\begin{equation}\label{min-max-adam0}
	\begin{aligned}
		&\min_{\eta_1  , \eta_2  ,..., \eta_{T-1}}\;\max_{f\in\mathcal F}\ \mathcal L \left(f(W_T,\xi)\right)\\
		& \quad\text{s.t.}\quad  
		W_{t+1} = W_t - \eta_t \cdot P(t) \nabla\mathcal L(W(t),\xi),
	\end{aligned}
\end{equation}
where $P(t)$ captures the optimizer-specific state dynamics, such as scaling or preconditioning, and may be chosen to represent different optimization algorithms within the same generalized framework. 
For instance, SGD corresponds to $P(t) = I$.
Adam is recovered by defining $P(t)$ such that $P(t) \nabla \mathcal{L}(W(t)) = m_t / (\sqrt{v_t} + \epsilon)$, where $m_t$ and $v_t$ are the first and second moment estimates following Adam update rules.
Other adaptive methods like  AdaGrad or Conda \citep{wang2025s} can be represented analogously through appropriate definitions of $P(t)$ based on their state variables.
This paper primarily focuses on the case where $P(t)$ implements the Adam dynamics (See details in Appendix \ref{Problem Statement-appendix}).

\subsection{Segment-wise continuous analysis}

There is a challenge in solving the problem \eqref{min-max-adam0} due to the exponential moving averages nature of the optimization process, where the loss landscape and gradient statistics of Adam evolve over time. 

To capture this nature yet retain tractability, we analyze training dynamics through their continuous-time limit rather than discrete updates following prior work \citep{li2017stochastic,li2019stochastic,li2020reconciling,li2021validity,li2025functional}. 
This perspective offers crucial advantages. First, we can make small segment-wise approximation of Adam by preconditioned GD.
In detail, we partition the problem into $K$ segments $[T_k,T_{k+1}) $, assuming local approximations for the curvature in each segment.

\paragraph{Segment-wise decomposition.}
We partition the training horizon $[0,\mathcal T]$ into $K$ segments
$
0 = T_0 < T_1 < \cdots < T_K = \mathcal T,
$
and focus on the $k$-th segment $[T_k,T_{k+1}]$. 
For ease of theoretical analysis, we transition from the discrete interval $[T_k,T_{k+1}]$ to a continuous representation. 
The local dynamics could be approximately analyzed, allowing us to model Adam as a preconditioned gradient flow with bounded approximation error (will be further explained in Paragraph Optimal Shape of Problem Solution\ref{Optimal Shape of Problem Solution}). Consequently, the adaptive Adam updates can be analyzed through the simpler lens of preconditioned gradient descent within each short segment.

\begin{definition}[intrinsic time]\label{def:U}
	For a segment $[T_k,T_{k+1})$, define 
	$
	U_k := \int_{T_k}^{T_{k+1}} \eta(t)\, \text{d} t ,
	$
	where $\eta(\tau)$ denotes the continuous-time learning rate.
\end{definition}

\paragraph{Intrinsic time.}
We introduce a key reparameterization of the training timeline via a continuous change of variables. 
The quantity $U_k$ has a natural interpretation as a \emph{time reparameterization}. 
Consider the mapping $t=\bar{T}(\tau):=\int_{T_k}^{\tau} \eta(z)\, \text{d} z$,  which transforms physical time (iteration index $\tau$) into an intrinsic time $t$ that reflects the internal optimization progress of model. Thus, $U_k$ represents the total intrinsic duration of the segment, providing a more accurate measure of training advancement than the raw iteration count.

\paragraph{Optimal shape of problem solution.}\label{Optimal Shape of Problem Solution}
We  characterize the optimal \emph{shape} of the stepsizes within each segment of Problem~\ref{min-max-adam0}.
Lemma~\ref{lem:segment-radius-continuous} controls the segment radius in terms
of the intrinsic time $U_k$. Lemma~\ref{lem:m-g} shows that on such short segments Adam bias-corrected momentum $\hat m_t$ tracks the true gradient with a small error.
These results validate a piecewise preconditioned gradient approximation and enable
Theorem~\ref{thm:uba-optimal-shape}, which identifies the UBA-equivalent family as the optimal segment-wise stepsize shape.

\begin{theorem}[Optimal UBA-equivalent shape on a segment]\label{thm:uba-optimal-shape}
	Consider the min-max problem~\ref{min-max-adam0} on segment $k$, and assume that the step sizes are constrained to lie in the interval $[\eta_{\min}^{(k)},\eta_{\max}^{(k)}]$. There exists an optimal shape of step sizes that minimizes the worst-case loss. This optimal shape belongs to the UBA-equivalent family for some curvature parameter $\varphi_k>0$. 
\end{theorem}\vspace{-0.3em}
The proof can be found in Appendix \ref{theorem-UBA-optimal-shape-appendix}.

For segment $k$, Theorem~\ref{thm:uba-optimal-shape} specifies that the \emph{shape} of the optimal step sizes is a member of UBA-equivalent Family. However, the precise shape of the LRD remains unknown since the key factor $\varphi$ is still free. 
To close this gap, in the next subsection, we determine the overall scale by introducing the intrinsic time, i.e., $U_k := \int_{T_k}^{T_{k+1}} \eta(\tau)\, \text{d} \tau$, to specify $\varphi$ and completing the optimal LRD.

\subsection{Regularized energy evolution via intrinsic time}\label{Regularized Energy Evolution}
This section establishes a theoretical control for scales of segment-wise LRD. 
We first model the training dynamics within each segment as a stochastic dynamical system following \citep{li2020reconciling,li2025functional}. From this stochastic differential equations (SDEs) formulation, 
we derive a bias-variance-style upper bound on the expected energy evolution (Theorem \ref{thm:bias-variance-intrinsic}). Minimizing this bound naturally leads to regularized problems:  the $\ell_1$-regularized formulation (Corollary \ref{cor:l1-Uk}) under a bounded LR constraint; the $\ell_2$-regularized surrogate (Corollary \ref{cor:l2-near-constant}) under a near-constant step-size approximation. 
Both problems admit explicit-form solutions, which yields the optimal intrinsic time  $U_k^{\ast}$. This solution inspires us to design the quantity that governs the scale of the UBA-shaped LR throughout the segment (see Appendix \ref{appendix-Regularized-Energy-Evolution} for details).


Denote $ \nabla\mathcal L(W_t)=\mathbb E_{\xi_t} \nabla\mathcal L(W_t,\xi_t)    $, we have $\nabla\mathcal L(W_t,\xi_t) = \nabla\mathcal L (W_t) + \delta_t $, where the gradient noise $\delta_t$ satisfies $\mathbb E (\delta_t) =0 $.
Since the segment is short, the stochastic preconditioned gradient flow can be modeled in accordance with SGD noise \citep{li2025functional} scaling by $\text{d}W(t)= -\eta (t)\,P(t)\nabla \mathcal L(W(t))\,\text{d}t +\;\eta (t)\,P(t)\Sigma^{1/2}\,\text{d}B(t),$
where $W(t)$ is the weight parameter, 
$\Sigma\succeq 0$ is the covariance of the gradient noise and $B(t)$ is a standard Brownian motion (see Appendix \ref{appendix-Regularized-Energy-Evolution} for details).

Under a quadratic expansion measurement, we define the energy function between the $t$-th step weights and local minimum within segment 
$	E(t):=\frac{1}{2}\|e(t)\|^2$, $e(t):=  H_k^{1/2} \left( W(t)-\bar W^{(k)}\right)     H_k^{1/2}
$ , where $H_k  $ is the abbreviation for $H_f^{(k)}$, the Hessian matrix at local minima (the precise construction is detailed in Appendix \ref{appendix-Regularized-Energy-Evolution}).

With the above setup, we obtain the following bias-variance-style bound expressed in terms of the intrinsic time.
\begin{theorem}
	\label{thm:bias-variance-intrinsic}
	Consider the $k$-th segment $[T_k,T_{k+1}]$ and the preconditioned stochastic
	gradient flow \eqref{eq:sde-seg-appendix}, 
	the expected energy at the segment end satisfies the bias--variance bound
	\begin{equation}\label{eq:bv-general}
		\begin{aligned}
			\mathbb E[E(T_{k+1})]&
			\;\le\;
			\mathbb E[E(T_k)]\,e^{-2\nu_\ell^{(k)}U_k}
			+\;
			\frac12\alpha\int_{T_k}^{T_{k+1}}\eta^2(t)\,\text{d}t,
		\end{aligned}
	\end{equation}
	where $\alpha := \sup_{t\in[T_k,T_{k+1}]}\operatorname{tr}\!\big(H_k^{1/2}P(t)\Sigma P(t)H_k^{1/2}\big)$ is the noise intensity bound, and $\lambda_{\min}(M(t))\ge \nu_\ell^{(k)} >0, M(t):=H_k^{1/2}P(t)H_k^{1/2}$ is the spectral lower bound.
\end{theorem}
The proof can be found in Appendix \ref{appendix-Regularized-Energy-Evolution}.

\begin{corollary}[$\ell_1$-regularized intrinsic-time problem]
	\label{cor:l1-Uk}
	Under Theorem~\ref{thm:bias-variance-intrinsic} and the constraint	$0\le \eta(t)\le \eta_{\max}$ on $[T_k,T_{k+1}]$, the intrinsic-time selection problem
	reduces to the $\ell_1$-regularized problem, 
	\begin{equation}\label{eq:l1-problem}
		\min_{U\in[0,\eta_{\max}\Delta T_k]}
		\Big\{
		\mathbb E[E(T_k)]e^{-2\nu_\ell^{(k)}U} + \frac12\alpha\eta_{\max}U
		\Big\}.
	\end{equation}
	Its minimizer admits the closed form  $\min \left(   \max \left(0, \frac{1}{2\nu_\ell^{(k)}}
	\ln\!\left(\frac{4\nu_\ell^{(k)}\mathbb E[E(T_k)]}{\alpha\eta_{\max}}\right) \right),  \eta_{\max}\Delta T_k  \right))$.

\end{corollary}
The proof can be found in Appendix \ref{appendix-Regularized-Energy-Evolution}.

\begin{corollary}[$\ell_2$ surrogate via near-constant stepsizes]
	\label{cor:l2-near-constant}
	Under Theorem~\ref{thm:bias-variance-intrinsic} and continuity of UBA-shape function, the intrinsic-time selection problem can be approximated as $\ell_2$-regularized problem, 
	\begin{equation}\label{eq:l2-obj-const}
		\min_{U_k\in[0,\eta_{\max}\Delta T_k]}
		\mathbb E[E(T_k)]e^{-2\nu_\ell^{(k)}U_k}
		+ \frac12\alpha\frac{U_k^2}{\Delta T_k},
	\end{equation}
	The minimizer is given by the Lambert-$W$ \citep{corless1996lambert} expression $\min \left(   \max \left(0, \bar{U} \right), \eta_{\max}\Delta T_k  \right)$, 
	where $\bar{U}:= \frac{1}{2\nu_\ell^{(k)}}\,W\!\left(
	\frac{2(\nu_\ell^{(k)})^2 \mathbb E[E(T_k)] \Delta T_k}{\alpha}
	\right)$.
\end{corollary}
The proof can be found in Appendix \ref{appendix-Regularized-Energy-Evolution}.

\paragraph{A computable plug-in estimate of $U_k^\ast$.}
Both closed-form solutions for $U_k^\ast$ in
Corollaries~\ref{cor:l1-Uk} and~\ref{cor:l2-near-constant} are dominated by the
ratio
\(
\frac{\nu_\ell^{(k)} \mathbb E[E(T_k)]}{\alpha}
\).
We retain this dominant part to preserve the leading-order asymptotic behavior
of the theoretical solutions while relying only on quantities observable during
training.

Since \(\nu_\ell^{(k)}\), \(\mathbb E[E(T_k)]\), and \(\alpha\) are not directly
observable, we adopt a plug-in estimate with measurable surrogates. The energy
term is approximated by the current training loss,
\(\mathbb E[E(T_k)]\approx \mathcal L(W_{T_k})\), following the local quadratic
definition of \(E\). The spectral lower bound is approximated by the Adam
preconditioned gradient,
\(
\nu_\ell^{(k)} T_k \approx
\frac{2E(T_k) }{\|\hat{m}_t/(\sqrt{\hat{v}_t}+\epsilon)\|^2}
\).
The noise intensity \(\alpha\) is treated as a tunable hyper-parameter. 
A direct substitution yields an inverse-quadratic dependence on the
preconditioned update norm, which becomes numerically unstable when
\(\|\cdot\|\) is small. We therefore replace \(\|\cdot\|^{-2}\) with the surrogate \((1+\|\cdot\|)^{-1}\), which preserves monotonicity
and asymptotic behavior while remaining bounded in \((0,1]\). This gives the
stabilized plug-in estimate
\begin{equation}\label{varphi_0}
	U_t^{\ast}
	=
	1/{
		\left(1 + \sqrt{\varphi_0/2}\,
		\left\|\frac{\hat{m}_t}{\sqrt{\hat{v}_t} + \epsilon}\right\| \right)
	},
\end{equation}
where \(\varphi_0\) is a hyper-parameter and
\(
U_k^{\ast}
= \inf_{t\in [T_k,T_{k+1})}
1/{
	(1 + \sqrt{\varphi_0/2}\,
	\left\|\frac{\hat{m}_t}{\sqrt{\hat{v}_t} + \epsilon}\right\|)
}
\).
This estimate preserves the dominant asymptotic behavior of the theoretical
solutions while depending only on quantities available in standard adaptive
optimizers (see Appendix
\ref{appendix-Regularized-Energy-Evolution-computable plug-in estimate} for
details).

	\subsection{Derived LRD}
	The optimal solution of problem \eqref{min-max-adam0} is given by a member of  UBA-equivalent family (Definition~\ref{definition:UBA-family}) where $\varphi_k$ is unknown. In this subsection, we complete the UTA design by \emph{matching} the energy-optimal intrinsic time $U_t^\ast$ \eqref{varphi_0}, with the intrinsic time implied by the UBA-equivalent shape , i.e., $U_{\varphi_k}=\int_{T_k}^{T_{k+1}} g_{a,b}\left(\varphi,\theta_{t,k}\right) \text{d}t$. Concretely, for the  desired parameterization of the shape, such as fixing $b=\eta_{\min}^{(k)}$ and $a=\left(\eta_{\max}^{(k)}- \eta_{\min}^{(k)} \right)$, solving for $\varphi_t$ under iteration-wise extension, we obtain the scalar equation $
	U_{\varphi_t}
	= U_t^\ast$ (the details of integral are in Appendix~\ref{Derived-LRD-appendix}). Then we have  
	$
		\varphi_t
		= \varphi_0 \left\|\frac{\hat{m}_t}{\sqrt{\hat{v}_t} + \epsilon}\right\|^2
	$.
	This yields a per-segment effective LRD that is (i) near-optimal in shape for ~\ref{min-max-adam0}, by Theorem~\ref{thm:uba-optimal-shape}, and (ii) energy-calibrated in scale, by subsection~\ref{Regularized Energy Evolution}. Furthermore, we extend the LRD on entire horizon and keep a controlled approximation for the original dynamics (see Appendix \ref{Global LRD by segment-wise concatenation}).
	The derived LRD on the entire horizon $[0,T]$ admits the explicit form
	\begin{equation}\label{UTA-final-function}
		\begin{aligned}
			\eta(t)
			= g_{a,b}(\varphi_t,\theta(t)),
		\end{aligned}
	\end{equation}
	where $a=\eta_{\min}$, $b=\eta_{\max} - \eta_{\min}$, 
	$
	\theta(t)
	:= \frac{\pi \cdot t  }{T},
	$
	the curvature parameter	
	$
	\varphi_t
	= \varphi_0 \left\|\frac{\hat{m}_t}{\sqrt{\hat{v}_t} + \epsilon}\right\|^2
	$.

	\subsection{Complexity analysis}
	
	The computational overhead of UTA is negligible in practice. Each update involves merely 8‑16 scalar arithmetic operations and a single norm computation. The cost is orders of magnitude lower than the FLOPs required for standard forward and backward propagation in modern neural networks.

	\begin{figure}[htbp]
		\centering
		\begin{subfigure}[b]{0.245\textwidth}
			\centering
			\includegraphics[width=\textwidth]{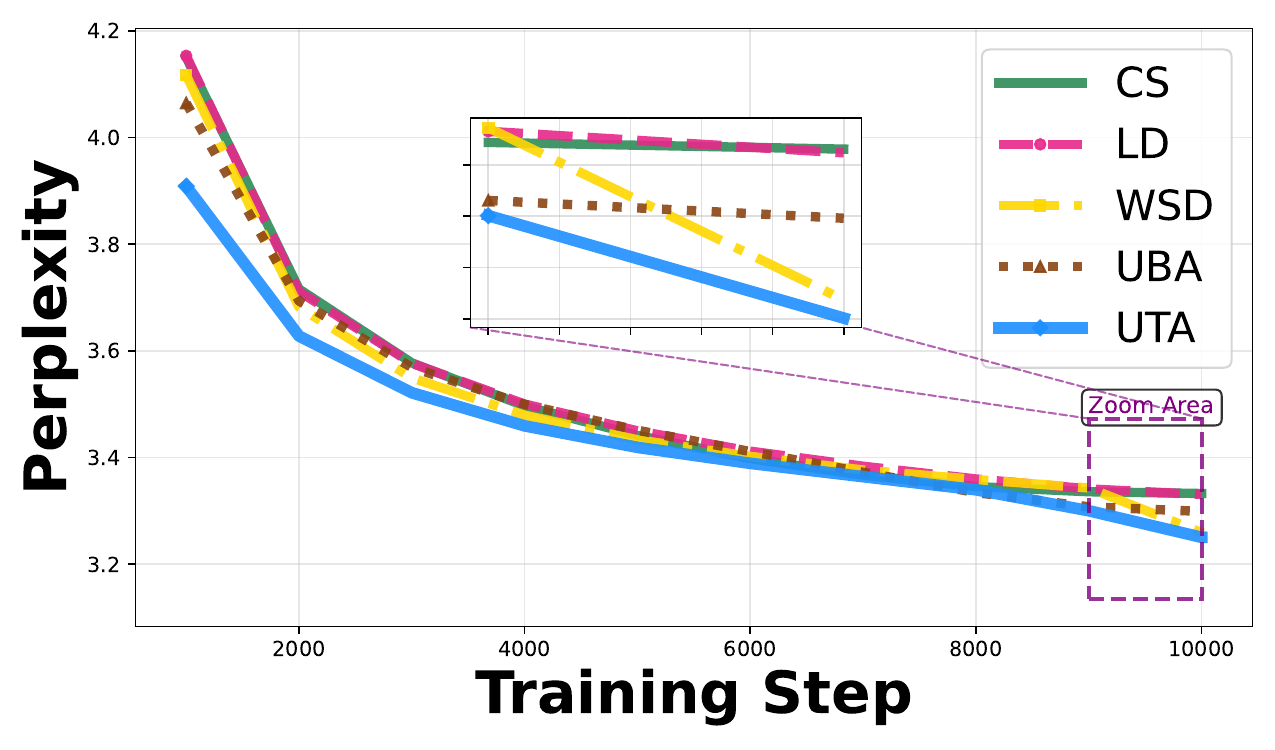}
			\caption{LLaMA-60M.}
			\label{Performance on Language Model-LLaMA60M}
		\end{subfigure}
		\begin{subfigure}[b]{0.245\textwidth}
			\centering
			\includegraphics[width=\textwidth]{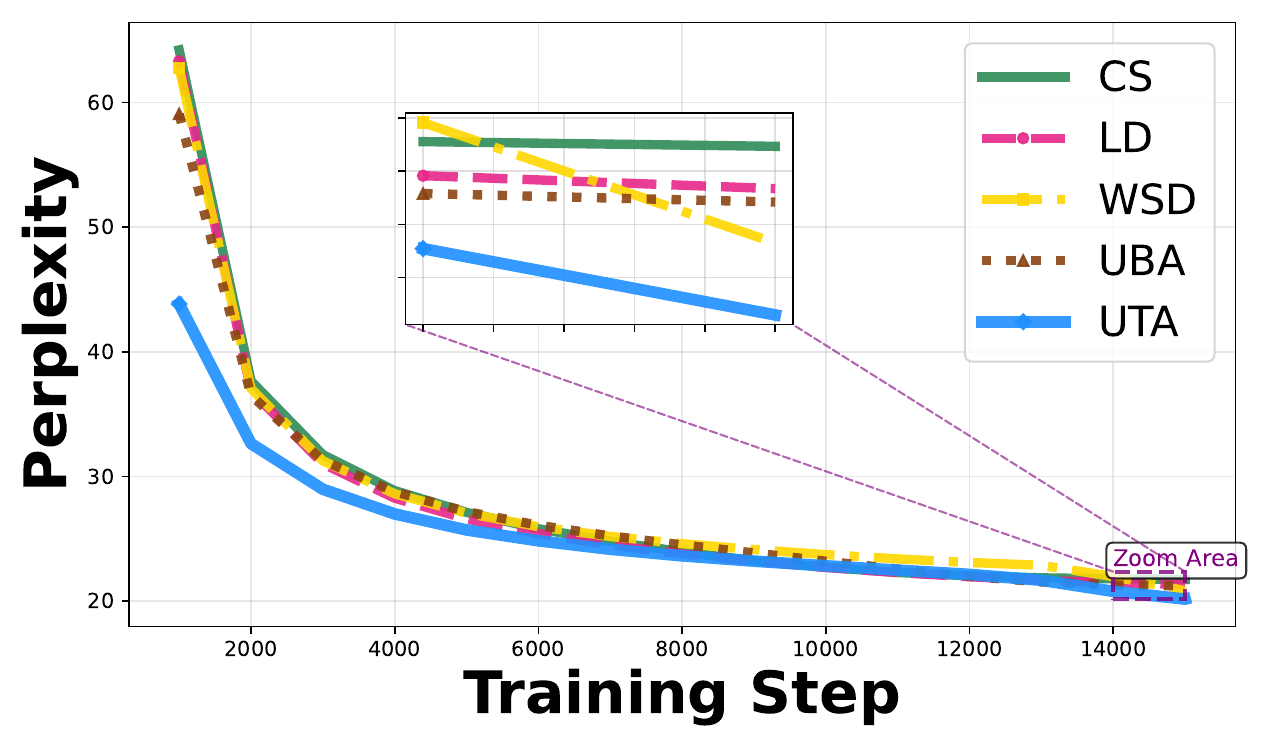}
			\caption{LLaMA-130M.}
			\label{Performance on Language Model-LLaMA130M}
		\end{subfigure}
		\begin{subfigure}[b]{0.245\textwidth}
			\centering
			\includegraphics[width=\textwidth]{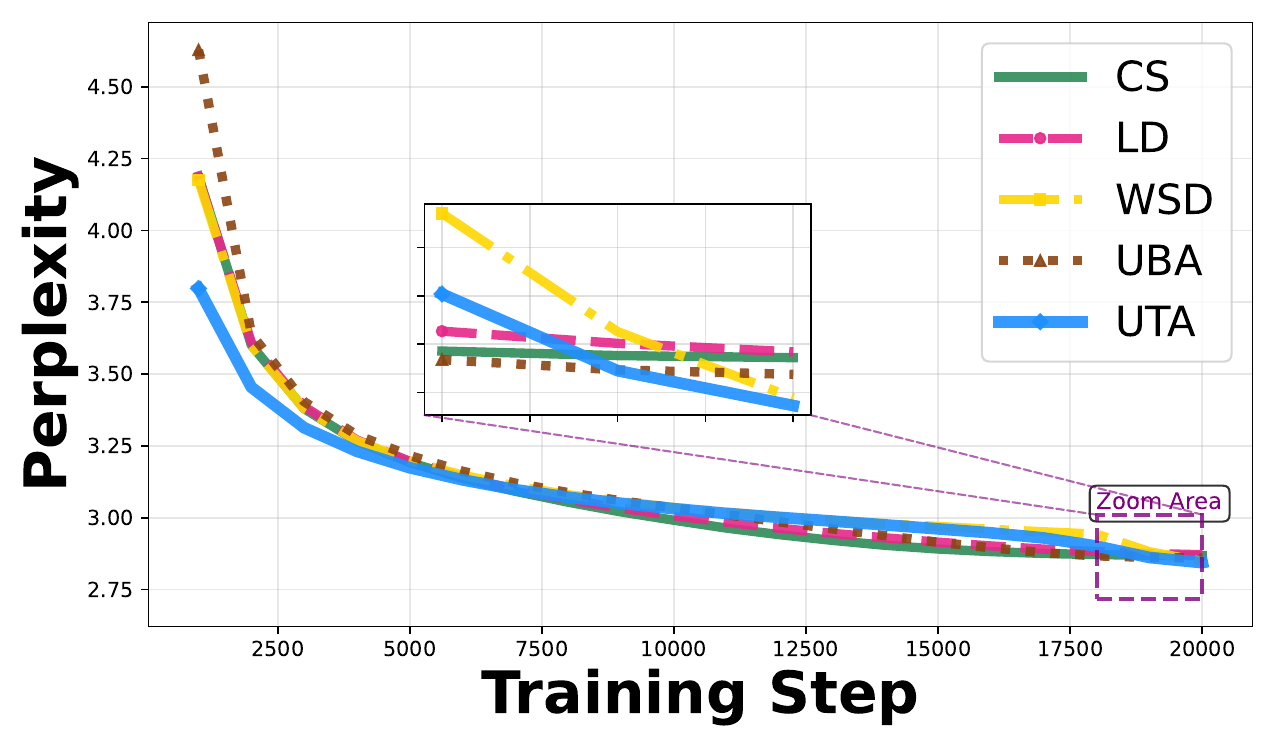}
			\caption{LLaMA-350M.}
			\label{Performance on Language Model-LLaMA350M}
		\end{subfigure}
		\begin{subfigure}[b]{0.245\textwidth}
			\centering
			\includegraphics[width=\textwidth]{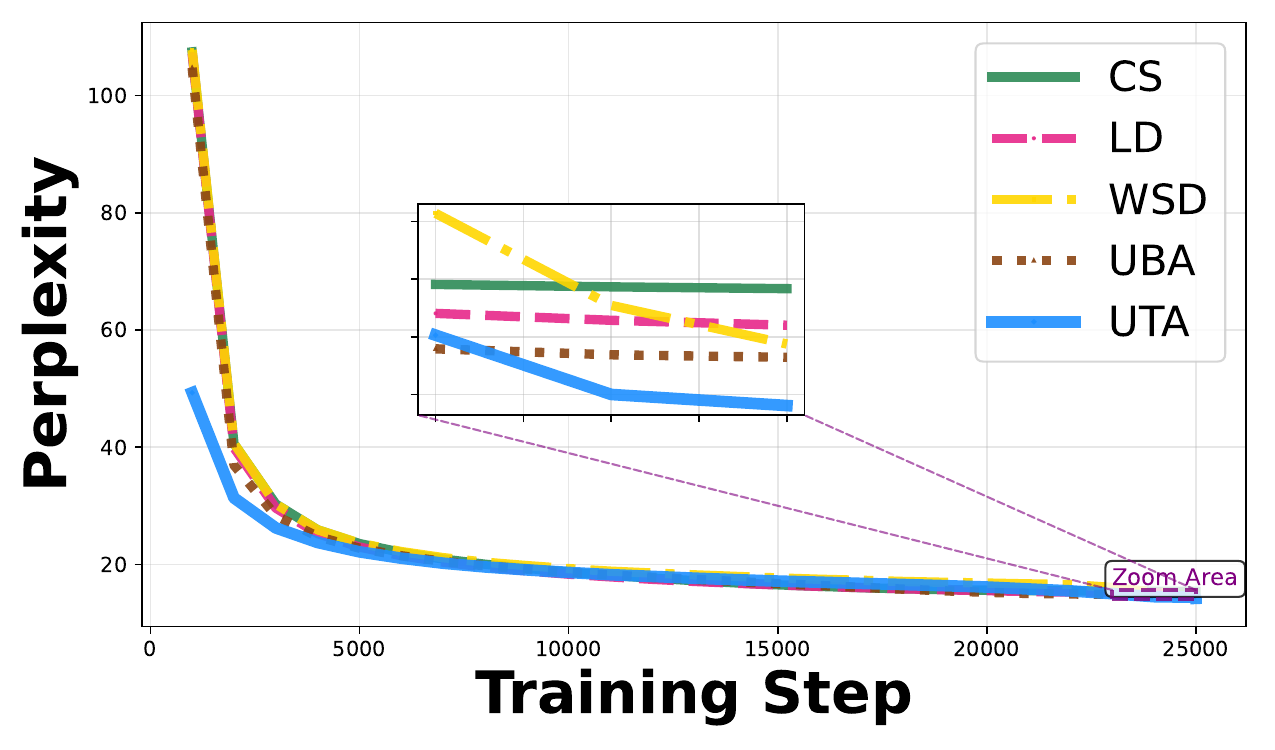}
			\caption{LLaMA-1B.}
			\label{Performance on Language Model-LLaMA1B}
		\end{subfigure}
		\begin{subfigure}[b]{0.245\textwidth}
			\centering
			\includegraphics[width=\textwidth]{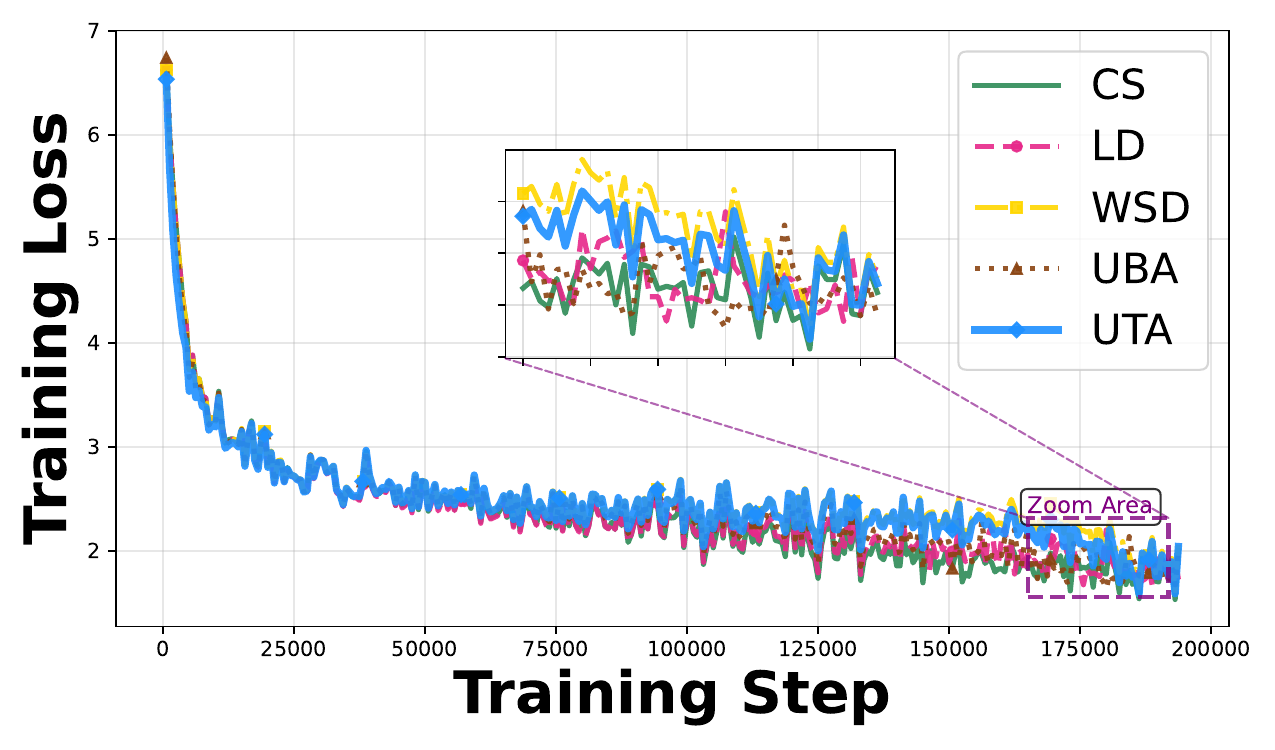}
			\caption{ImageNet-ResNet50}
			\label{fig:sub1}
		\end{subfigure}
		\begin{subfigure}[b]{0.245\textwidth}
			\centering
			\includegraphics[width=\textwidth]{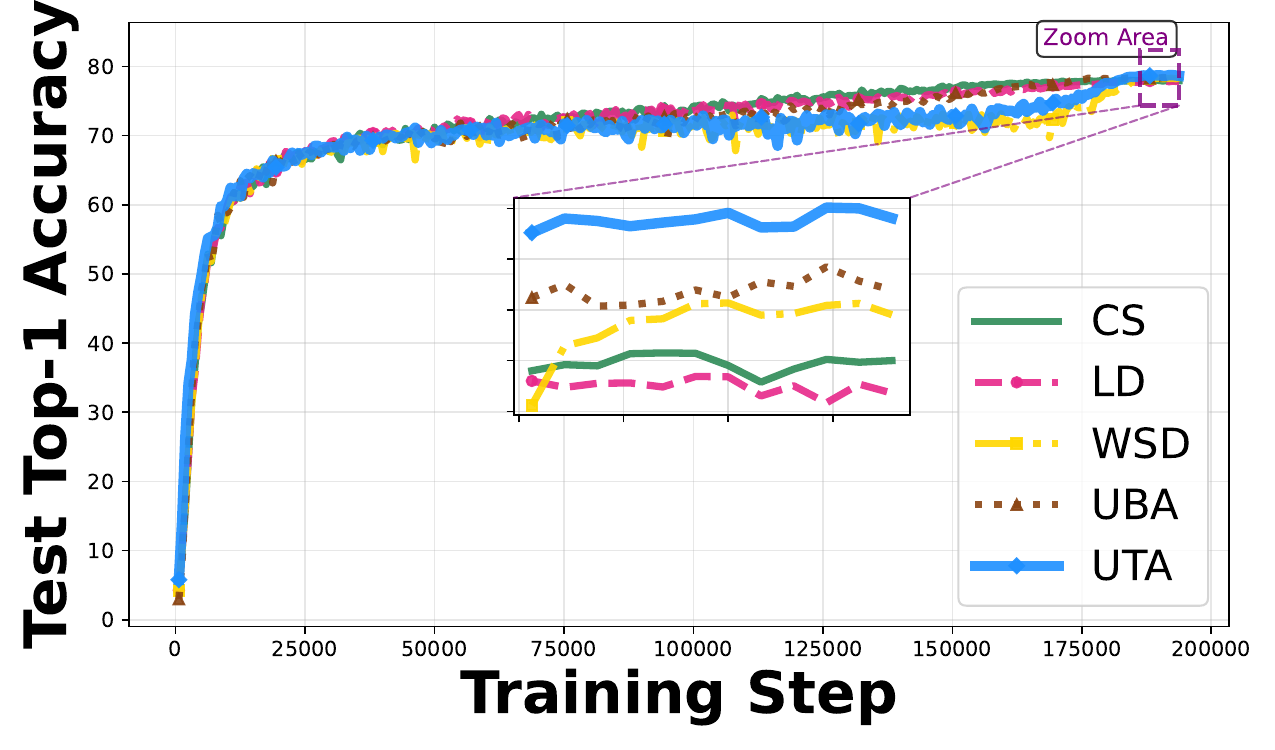}
			\caption{ImageNet-ResNet50.}
			\label{fig:sub2}
		\end{subfigure}
		\begin{subfigure}[b]{0.245\textwidth}
			\centering
			\includegraphics[width=\textwidth]{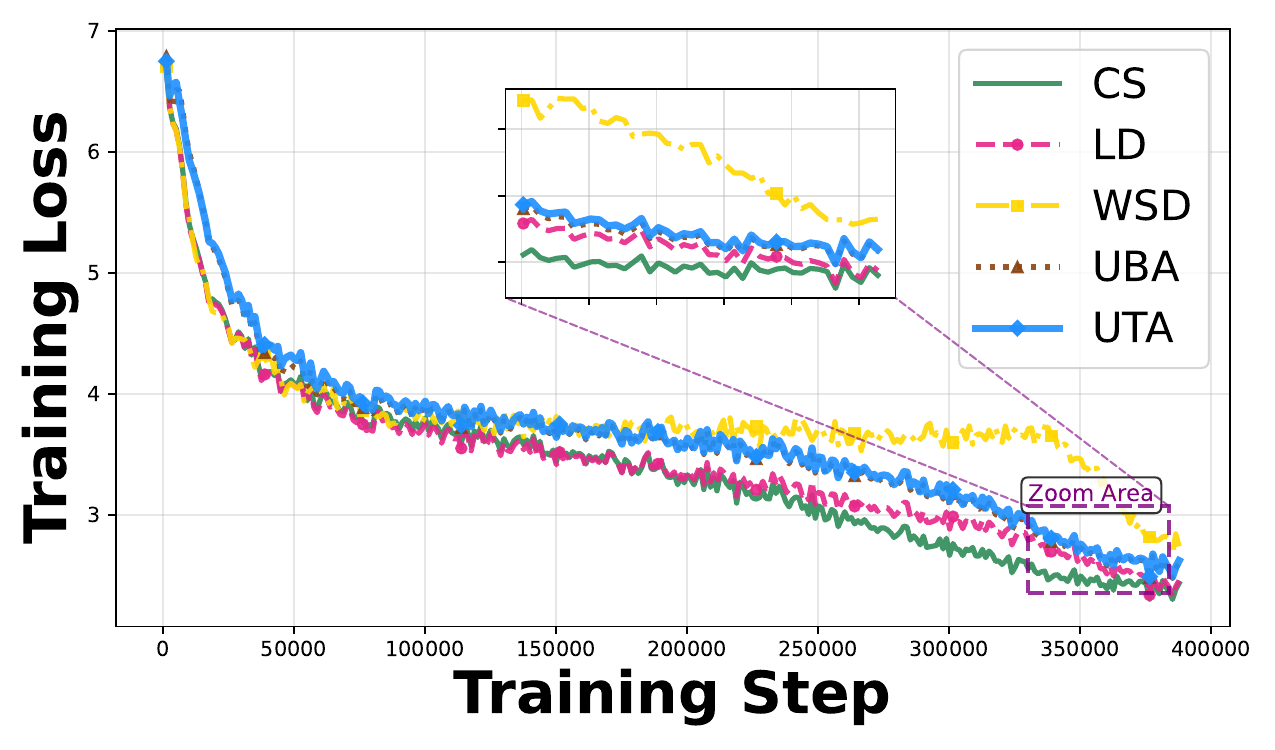}
			\caption{ImageNet-ViT-B.}
			\label{fig:sub3}
		\end{subfigure}
		\begin{subfigure}[b]{0.245\textwidth}
			\centering
			\includegraphics[width=\textwidth]{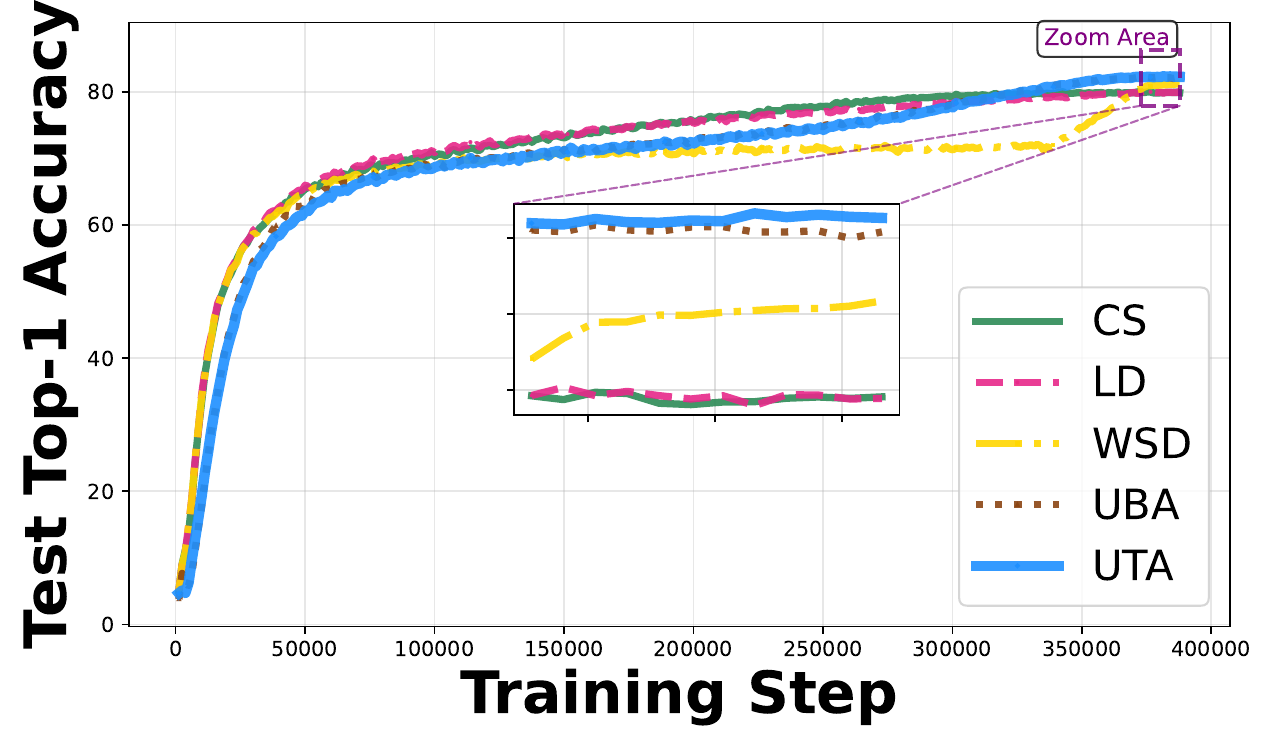}
			\caption{ImageNet-ViT-B.}
			\label{fig:sub4}
		\end{subfigure}
		\caption{Training curves of LLaMA models, ResNet-50 and ViT-B. }
		\label{perplexity-llama}
	\end{figure}

	\section{Experiments}\label{Experiments}
	We conduct a systematic evaluation of the proposed LRD along three key axes of diversity: (i) \textbf{model architecture}, covering both convolutional (ResNet\citep{he2016deep}) and attention-based (ViT\citep{dosovitskiy2020image}) vision models, as well as language models (LLaMA); (ii) \textbf{task modality}, encompassing vision and language tasks; and (iii) \textbf{training scale}, covering the large-scale training (ImageNet and LLM pre-training from 60M to 1B). This multi-faceted evaluation is designed to rigorously assess the generality, robustness, and scalability of our method.
	\paragraph{Baselines.}
	Our baseline selection is informed by recent systematic evaluations of LRDs. The study \citep{tang2026stepsize} performs comprehensive empirical comparisons of LRDs. It shows that the UBA is a consistently top-performing LRD, and identifies cosine and linear decay as competitive approaches. Accordingly, we include \textbf{UBA} \citep{tang2026stepsize} as a strong latest baseline, and include \textbf{Cosine(CS)} \citep{loshchilov2017sgdr} and \textbf{Linear decay(LD)} \citep{li2019budgeted} as representative classical baselines in our evaluation. To further strengthen the comparison, we also include the \textbf{WSD} \citep{zhai2022scaling,hu2024minicpm} LRD, a widely used and empirically validated strategy in large-scale language model pre-training, as an additional strong baseline representative of modern practices.
	The LRD-specific hyper-parameters and details are provided in Appendix \ref{baselines-appendix} and \ref{LRD-Setting-appendix}.
	\paragraph{Experimental details.} 
	All baselines follow \textbf{the same standard configurations} as established in prior work on LLaMA~\citep{lialin2023stack,zhao2024galore}, ResNet~\citep{wightman2021resnet}, and ViT~\citep{touvron2021training,yu2022metaformer}, ensuring a fair comparison. These configurations are widely adopted defaults that have been empirically validated as the baseline. 
	
	A key advantage of UTA lies in its parameter $\varphi_0$. While an optimal $\varphi_0$ can further improve model performance, we intentionally fix its value across all tasks and architectures to ensure fair evaluation (Appendix \ref{LRD-Setting-appendix}).

	Comprehensive experimental details, such as computational resources, configurations, implementation specifics, are elaborated in Appendix \ref{appendix-implement-details}.
	Our implementation strictly follows the PyTorch official API. It is inherited from the base class
	\verb|torch.optim.lr_LRDr.LRDcheduler|. This design ensures native compatibility with the optimizer in \verb|torch.optim|. 
	Fully functional, production-ready code is publicly released (see Appendix \ref{appendix-implement-details} for link and implement details).

	\subsection{Performance on language task with LLaMA}
	We evaluate on \textbf{four scales} of the LLaMA architecture: 60M, 130M, 350M, and 1B parameters. For each model, we consider \textbf{four distinct training budgets}: 10K, 15K, 20K, and 25K training steps, covering regimes from extremely limited to moderately long training. 
	To provide a granular view of the optimization dynamics, we report the validation perplexity at around 30\%, 60\%, and 100\% of the total budget for each scale (e.g., for 1B model, at 8k, 16k, and final 25k steps). 
	Performance is measured by validation perplexity, which directly reflects the model generative uncertainty and language modeling quality.
	Results are summarized in Tab.~\ref{llama-model-Perplexity}.

	\begin{table}[!ht]
		\centering
		\small
		\caption{Perplexity ($e^{\text{validation loss}}$) on LLaMA 60M, 130M, 350M and 1B models. Lower is better (\textdownarrow).}\label{llama-model-Perplexity}
		
		\setlength{\tabcolsep}{4pt}
		\renewcommand{\arraystretch}{1.1}
		\begin{tabular}{@{}llcccccccccccc@{}}
			\toprule
			\multirow{2}{*}{LRD}& Model & \multicolumn{3}{c}{60M} & \multicolumn{3}{c}{130M} & \multicolumn{3}{c}{350M} & \multicolumn{3}{c}{1B} \\
			\cmidrule(lr){3-5} \cmidrule(lr){6-8} \cmidrule(lr){9-11} \cmidrule(lr){12-14}
			& Step & 3k& 6k & final& 5k  & 10k & final & 7k& 14k & final & 8k& 16k& final\\
			\midrule
			CS &  & 35.79 & 29.91 & 28.03 & 27.12 & 22.71 & 21.76 & 22.03 & \textbf{18.24} & 17.62 & 20.14 & \textbf{16.32} & 15.41 \\
			LD &  & 35.76 & 30.34 & 27.99 & 26.60 & 22.71 & 21.43 & 22.16 & 18.72 & 17.68 & 19.96 & 16.49 & 15.10 \\
			WSD &  & 34.80 & 30.03 & 26.03 & 27.05 & 23.71 & 20.83 & 22.41 & 19.66 & 17.25 & 20.60 & 17.94 & 14.94 \\
			UBA &  & 35.48 & 30.30 & 27.01 & 27.18 & 23.18 & 21.12 & 22.24 & 18.68 & 17.42 & 19.78 & 16.41 & 14.83 \\
			\rowcolor{cyan!8} UTA&  & \textbf{33.72} & \textbf{29.57} & \textbf{25.80} & \textbf{25.74} & \textbf{22.83} & \textbf{20.16} & \textbf{22.14} & 19.58 & \textbf{17.17} & \textbf{19.56} & 16.97 & \textbf{14.53} \\
			\bottomrule
		\end{tabular}
	\end{table}

	UTA consistently outperforms all baselines across model scales and most training checkpoints. This is not merely a marginal improvement at the final step, but a systematic acceleration throughout the entire training process. 
	Firstly, \textbf{UTA exhibits remarkably fast convergence in the early stages}. At the 30\% checkpoint, UTA achieves the lowest perplexity across all model sizes. These early-stage gains demonstrate UTA's ability to establish a proper initial trajectory, which is critical for large-scale training where early decisions profoundly impact final performance.
	Most critically, \textbf{UTA achieves best final perplexity across all four model scales}, with absolute improvements from 0.23 (60M), 0.67 (130M), 0.08 (350M) and 0.30 (1B)  perplexity units over the best baseline. 
	Although the absolute final gains range from 0.23 to 0.67 perplexity, the relative improvement becomes increasingly meaningful as model size grows. Because for large-scale foundation models, even a 0.2-unit reduction in perplexity translates into a statistically significant leap in generation quality and reliability.
	We analyze that these consistent gains may be attributed to (i) near-optimal in \textit{shape} by Theorem~\ref{thm:uba-optimal-shape}, and (ii) \textit{energy-calibrated} in scale by subsection~\ref{Regularized Energy Evolution}. 
	It gives UTA the ability to dynamically \textit{balance exploration and stabilization} during training.

	We also illustrate the perplexity evolution in Fig. \ref{perplexity-llama}. 
	All curves show a characteristic rapid initial descent followed by gradual convergence. UTA consistently outperforms throughout training, demonstrating both faster initial convergence and superior final performance, while WSD and UBA exhibit competitive performance. These patterns align with theoretical expectations in Appendix \ref{Theoretical Analysis} .

	Moreover, we emphasize that the observed performance gains are not simply due to an increased average LR. In our experiments, the average LR of the WSD baseline was notably higher than that of our UTA, yet UTA still achieves better perplexity. The advantage stems from the optimal selection of intrinsic time in UTA, which promotes more effective learning and reduces the accumulation of noise during training, by Theorem \ref{theorem4-2}. This leads to improved generalization, ultimately resulting in superior model performance across different scales.

	\subsection{Performance on vision task ImageNet}
	We further validate the generalization of UTA on the ImageNet classification benchmark using ResNet-50 and ViT-B architectures. This selection tests generalization across model depth, width, and fundamental architectural biases. 
	Fig. \ref{fig:sub1}, \ref{fig:sub2}, \ref{fig:sub3} and \ref{fig:sub4} depict training curves of ResNet-50 and ViT-B on ImageNet under different LRDs. For both architectures, UTA consistently achieves the highest final accuracy.  UBA demonstrates competitive performance,  while WSD and CS show similar but suboptimal convergence. 
	Tab.~\ref{imagenet-tab} reports the test top-1 accuracy. UTA consistently achieves the highest accuracy across both models. 
	Specifically, it improves ResNet-50 accuracy to 78.73\%, outperforming the strongest baseline UBA (78.38\%) by 0.35\%. For ViT-B, UTA reaches 82.30\%, surpassing the competitive UBA LRD (82.05\%) by 0.25\%. 
	These gains are statistically meaningful given the scale and complexity of ImageNet. 
	These consistent outperformance demonstrates that UTA generalizes beyond language modeling to vision tasks. It offers a robust and architecture-agnostic advantage in large-scale, real-world training scenarios, generalizing beyond language modeling to vision tasks.

	\begin{table}[!t]
		\centering
		\small
		\caption{ImageNet test top-1 Acc. (\%) of different LRDs on ResNet50 and ViT-B for 300 epochs. 
			Higher values indicates better performance (\textuparrow). Best results in bold.}\label{imagenet-tab}
		\renewcommand{\arraystretch}{1.15}
		\begin{tabular}{lccccc}
			\toprule
			LRD & CS & LD & WSD & UBA & UTA\\
			\midrule
			ResNet-50& 78.34$_{\pm 0.024}$  & 77.90$_{\pm 0.028}$ & 78.30$_{\pm 0.024}$ & 78.38$_{\pm 0.014}$ & \multicolumn{1}{>{\columncolor{cyan!8}}c}{\textbf{78.73}$_{\pm 0.029}$}\\
			ViT-B& 79.92$_{\pm 0.057}$ & 79.96$_{\pm 0.104}$ & 81.14$_{\pm 0.048}$ & 82.05$_{\pm 0.041}$ & \multicolumn{1}{>{\columncolor{cyan!8}}c}{\textbf{82.30}$_{\pm 0.037}$ }\\
			\bottomrule
		\end{tabular}
	\end{table}

	\subsection{Ablation study}\label{Ablation-Study1234}


	\begin{table*}[t]
		\centering
		\footnotesize
		\setlength{\tabcolsep}{3.5pt}
		\renewcommand{\arraystretch}{1.1}
		\caption{
			Perplexity ($e^{\text{validation loss}}$) of LLaMA-60M and LLaMA-130M under different peak LRs.
			Lower is better. 
		}
		\label{tab:llama-60m-130m-full-lr}
		\begin{tabular}{@{}lcccccccc@{}}
			\toprule
			Model & LRD
			& $5{\times}10^{-4}$ & $1{\times}10^{-3}$
			& $3{\times}10^{-3}$  & $5{\times}10^{-3}$
			& $7{\times}10^{-3}$ & $1{\times}10^{-2}$ & Best from LR Sweep
			\\
			\midrule
			\multirow{5}{*}{60M}
			& CS   & 32.61 & 28.03 & 26.91   & 26.91   & 26.31 & 26.57 &  26.31\\
			& LD   & 32.83 & 27.99 & 26.80  & 26.77   & 26.25 & 26.46 & 26.25 \\
			& WSD  & 29.40 & 26.03 & 26.56   & \textbf{26.34}     & 26.06 & \textbf{26.28} & 26.03\\
			& UBA  & 30.65 & 27.01 &\textbf{ 26.51}   & 26.42     & 26.26 & 26.57 & 26.26\\
			\rowcolor{cyan!8}& UTA  & \textbf{29.28} &\textbf{ 25.80} & 26.86   & 27.00  & \textbf{25.81} & 27.11 & \textbf{25.80} \\
			\midrule
			\multirow{5}{*}{130M}
			& CS   & -- & 21.76 & 21.32   & 20.94   & 21.27 & --  &20.94\\
			& LD   & -- & 21.43 & 21.18  &  20.90    & \textbf{21.24} & --  & 20.90\\
			& WSD  & -- & 20.83 & 21.16   &  \textbf{20.82}   & 21.39 & -- & 20.82\\
			& UBA  & -- & 21.12 & \textbf{20.78}   & 21.25  & 21.92 & --  & 20.78\\
			\rowcolor{cyan!8}& UTA  & -- & \textbf{20.16} & 21.36   & 21.65 &  22.39 & --  &\textbf{20.16}\\
			\bottomrule
		\end{tabular}
	\end{table*}
	
	\paragraph{Peak LR sensitivity analysis.}\label{}
	
	To address the concern that a fixed peak LR may disadvantage baselines with different optimal LRs, we conduct a LR sweep on LLaMA-60M and 130M (Tab.~\ref{tab:llama-60m-130m-full-lr}). As expected, optimal peak LRs vary across LRDs. Nevertheless, under each LRD's optimal LR, \textbf{UTA consistently outperforms all baselines on both scales}. On 60M, UTA achieves a best perplexity of \textbf{25.80}, surpassing the best baseline UBA (26.26) by \(0.46\) and WSD (26.03) by \(0.23\). On 130M, UTA's advantage is more pronounced, achieving \textbf{20.16}, which outperforms WSD (20.82) by \(0.66\) and UBA (20.78) by \(0.62\). These margins, while modest in absolute terms, are meaningful in the context of language modeling validation loss.

	\paragraph{Parameter analysis of $\varphi_0$.}\label{Ablation-Study-d}
	We analyze UTA's sensitivity to $\varphi_0$ (\eqref{varphi_0}) on LLaMA-60M and 130M (Tab.s~\ref{Cross-d-table-appendix1}--\ref{Cross-d-table-appendix2}). UTA exhibits robustness. On 60M, it achieves perplexity below 26.10 for 7 out of 10 tested values, with variation of only 0.24 across a 10-fold range; on 130M, performance remains nearly flat (20.63--20.67) across $\varphi_0 \in [0.15, 0.5]$. Performance improves modestly as $\varphi_0$ decreases, consistent with our theory that smaller $\varphi_0$ yields larger intrinsic steps. Crucially, UTA demonstrates a broader effective range than UBA, which degrades for $\varphi_0 < 0.5$. We attribute this to UTA's adaptive mechanism, which dynamically compensates for suboptimal $\varphi_0$ choices, a capability absent in UBA. These results confirm that UTA is both performant and easy to use.
	
	\paragraph{The effect of final LR in decay.} Regarding the LRD  implementations, some practices use decay to 10\% of the peak LR, while others use 0\%. 
	Thus, we ablate the final LR ratio (10\% vs. 0\%) on LLaMA-60M and 130M (Tab.~\ref{llama-model-Perplexity-appendix-min_lr_ratio_0}). Results show that decaying to 0\% underperforms 10\% across all methods and scales. Specifically, on 130M, LD degrades from 21.43 to 21.89, and UTA from 20.16 to 20.68. This suggests that aggressive decay prematurely saturates optimization under our moderate training budgets. Importantly, UTA remains the top performer under both 0\% and 10\% decay, surpassing the best baseline by 0.16 (60M) and 0.32 (130M). This confirms that UTA is robust to the choice of final LR, and that the default 10\% decay setting does not unfairly advantage UTA.
	
	\paragraph{The role of dynamic state-Aware mechanism.}
	To validate the component efficacy of the dynamic mechanism in our UTA LRD, we conduct an ablation study comparing the full version (Ours-Full) against a static variant (Ours-Static). We conduct the ablation study on the CIFAR100-ResNet34, ImageNet-ResNet50, LLaMA60M and 130M. In the static variant, the dynamic part is removed. All other hyper-parameters and training settings remain identical. 
	As shown in Tab.~\ref{ablation-tab-1-appendix}, UTA (Full) consistently outperforms its static counterpart. Moreover, the improvement is pronounced, suggesting that the dynamic state adaptation mechanism is a crucial component for LRD performance. More details and analyses can be found in Appendix \ref{ablation-Dynamic State-Aware Mechanism-appendix} and \ref{Curvature-alignment-UTA-UBA-appendix}.
	
	\paragraph{Comparison with automatic LR methods.}
	We additionally compare UTA against automatic LR methods, such as D-Adaptation~\citep{defazio2023learning} and Prodigy~\citep{mishchenko2023prodigy}, which adjust LRs dynamically without predefined schedule shapes, in contrast to our main baselines.
	Results on LLaMA-60M, 130M, and 350M reported in Tab.~\ref{tab:vs-auto-lrs-appendix} in Appendix. 
	UTA matches D-Adaptation on 60M and outperforms it on 130M and 350M by 0.49 and 0.47, respectively. Against Prodigy, UTA achieves consistent improvements of 1.57, 1.10, and 0.80 across the three scales. We attribute UTA's advantage to its shape design, which explicitly accounts for the training horizon, unlike methods that rely solely on gradient statistics.

	\paragraph{Parameter-efficient fine-tuning scenario.}
	To demonstrate that our LRD generalizes beyond full-parameter training, we evaluate it in a parameter-efficient fine-tuning (PEFT) scenario. We apply UTA to fine-tune a ViT-B model on CIFAR-100  and OxfordPets , using LoRA\citep{hu2022lora} ($r=16$) as the PEFT method. 
	Tab.~\ref{ablation-tab-lora-appendix} demonstrates UTA effectiveness in PEFT settings. UTA outperforms UBA on CIFAR100 and OxfordPets. 
	These results confirm that UTA generalizes effectively to parameter-constrained optimization (see details in Appendix \ref{PEFT-scenario-appendix}).

	\section{Conclusion} \label{conclusion}

	We introduced \textbf{UTA}, a novel LRD framework that bridges the gap between theoretically principled shape design and practical, state-aware adaptation. 
	Methodologically, UTA extends the ROLR from SGD to  adaptive optimizers, particularly AdamW, overcoming a key limitation of prior work. Through the concept of intrinsic time and energy-calibrated adaptation, UTA automatically determines the LR in response to real-time optimization states.
	Experiments on vision and language models confirm that UTA consistently outperforms existing LRDs with minimal overhead. Comprehensive ablations confirm its robustness to hyper-parameters (peak LR, final decay ratio, $\varphi_0$), the necessity of its dynamic mechanism, and its superiority over automatic LR methods. 
	Overall, UTA not only delivers a plug‑and‑play LRD solution but also advances the broader understanding of state dependent optimization. 
	In the future, potential extensions include adapting the framework to optimizers such as Conda \citep{wang2025conda} and SOAP \citep{vyas2025soap}. Moreover, the optimal amplitude of LR fluctuations remains an open problem. A rigorous analysis could help choosing optimal fluctuation range, further improve model performance.

	\newpage

	\newpage

\bibliography{iclr2027_conference}
\bibliographystyle{iclr2027_conference}

\newpage

\appendix
\addtocontents{toc}{\protect\setcounter{tocdepth}{-1}}  
\tableofcontents
\addtocontents{toc}{\protect\setcounter{tocdepth}{2}}   

\section*{Appendix}

\section{Notation}\label{app:notation}

For ease of reference, we summarize the main notation used throughout the paper.

\paragraph{Optimization and training.}
\begin{center}
	\small
	\setlength{\tabcolsep}{6pt}
	\renewcommand{\arraystretch}{1.15}
	\begin{tabular}{@{}ll@{}}
		\toprule
		Symbol & Meaning \\
		\midrule
		$W \in \mathbb{R}^N$ & model parameters (a single vector) \\
		$t$ & current training step (physical iteration index) \\
		$T$ & total number of training iterations \\
		$\mathcal{T}$ & training horizon, $[0,\mathcal{T}]$ \\
		$\mathcal{L}(\cdot)$ & training loss function \\
		$\xi$ & data sample; $\xi_t$ denotes the sample at step $t$ \\
		$D$ & training dataset \\
		$\mathcal{F}$ & function class (network structures, datasets, etc.) \\
		$f(W,\xi)$ & neural network parameterized by $W$ on data $\xi$ \\
		$\nabla \mathcal{L}(W_t)$ & full gradient at step $t$ \\
		$\nabla \mathcal{L}(W_t,\xi_t)$ & stochastic gradient at step $t$ \\
		$\delta_t$ & gradient noise, $\delta_t := \nabla\mathcal L(W_t,\xi_t) - \nabla\mathcal L(W_t)$ \\
		\bottomrule
	\end{tabular}
\end{center}

\paragraph{Optimizer and learning rate.}
\begin{center}
	\small
	\setlength{\tabcolsep}{6pt}
	\renewcommand{\arraystretch}{1.15}
	\begin{tabular}{@{}ll@{}}
		\toprule
		Symbol & Meaning \\
		\midrule
		$\eta_t$ & learning rate at step $t$ \\
		$\eta(\cdot)$ & continuous-time learning rate \\
		$\eta_{\min},\eta_{\max}$ & minimum and maximum learning rates \\
		$\eta_{\min}^{(k)},\eta_{\max}^{(k)}$ & range parameters on segment $k$ \\
		$\phi(\cdot)$ & learning rate decay (LRD) as a function of physical time \\
		$P(t)$ & optimizer-specific preconditioner (e.g., $P(t)=I$ for SGD) \\
		$m_t, v_t$ & first and second moment estimates in Adam \\
		$\hat m_t, \hat v_t$ & bias-corrected moments in Adam \\
		$\beta_1, \beta_2$ & Adam momentum coefficients \\
		$\epsilon$ & small constant for numerical stability \\
		$\lambda$ & weight decay coefficient \\
		$B$ & batch size \\
		\bottomrule
	\end{tabular}
\end{center}

\paragraph{Segments and intrinsic time.}
\begin{center}
	\small
	\setlength{\tabcolsep}{6pt}
	\renewcommand{\arraystretch}{1.15}
	\begin{tabular}{@{}ll@{}}
		\toprule
		Symbol & Meaning \\
		\midrule
		$K$ & number of segments \\
		$[T_k,T_{k+1}]$ & the $k$-th segment \\
		$\Delta T_k$ & segment length, $\Delta T_k := T_{k+1} - T_k$ \\
		$k$ & segment index; under iteration-wise extension, $k$ coincides with $t$ \\
		$U_k$ & intrinsic time of segment $k$, $U_k := \int_{T_k}^{T_{k+1}} \eta(t)\,\mathrm{d}t$ \\
		$U_t^\ast$ & energy-optimal intrinsic time at step $t$ \\
		$\bar T(\tau)$ & intrinsic-time reparameterization, $\bar T(\tau) := \int_{T_k}^{\tau} \eta(z)\,\mathrm{d}z$ \\
		$\bar T^{-1}(t)$ & inverse intrinsic-time map \\
		$\varphi$ & curvature parameter of the UBA-equivalent family \\
		$\varphi_k, \varphi_t$ & curvature parameter on segment $k$ / at step $t$ \\
		$\varphi_0$ & hyper-parameter in the UTA plug-in estimate \\
		$\psi$ & normalized intrinsic time, $\psi := U_k^\ast / \Delta T_k$ \\
		$\theta_{t,k}$ & phase variable of the UBA-equivalent family \\
		$\theta(t)$ & phase variable in the global UTA LRD \\
		\bottomrule
	\end{tabular}
\end{center}

\paragraph{Energy evolution and curvature.}
\begin{center}
	\small
	\setlength{\tabcolsep}{6pt}
	\renewcommand{\arraystretch}{1.15}
	\begin{tabular}{@{}ll@{}}
		\toprule
		Symbol & Meaning \\
		\midrule
		$\bar W^{(k)}$ & local minimum in segment $k$ \\
		$H_k$ & Hessian at the local minimum, $H_k := H_f^{(k)}$ \\
		$E(t)$ & energy function, $E(t) := \tfrac12 \|e(t)\|^2$ \\
		$e(t)$ & transformed deviation, $e(t) := H_k^{1/2}(W(t)-\bar W^{(k)})$ \\
		$\nu_\ell^{(k)}$ & spectral lower bound of $H_k^{1/2} P(t) H_k^{1/2}$ \\
		$\alpha$ & noise intensity bound \\
		$\Sigma$ & covariance of the gradient noise \\
		$B(t)$ & standard Brownian motion \\
		$\mu(t)$ & integrating factor in the SDE analysis \\
		$\varepsilon$ & target accuracy in the plug-in estimate \\
		\bottomrule
	\end{tabular}
\end{center}

\paragraph{Functional scaling law (FSL).}
\begin{center}
	\small
	\setlength{\tabcolsep}{6pt}
	\renewcommand{\arraystretch}{1.15}
	\begin{tabular}{@{}ll@{}}
		\toprule
		Symbol & Meaning \\
		\midrule
		$M$ & model size (number of parameters) \\
		$\beta$ & model capacity exponent \\
		$s$ & task difficulty exponent \\
		$\mathcal{E}_t$ & expected loss at intrinsic training time $t$ \\
		$\sigma^2$ & label noise variance \\
		$\mathcal{K}(\cdot)$ & forgetting kernel \\
		$e(t)$ & learning term, $e(t) := (1+t)^{-s}$ \\
		$\gamma(\cdot)$ & noise intensity as a function of intrinsic time \\
		$T_{\max}$ & total number of training steps in the FSL analysis \\
		$T_1, T_2$ & durations of the stable and decay stages in the UTA FSL analysis \\
		$a, b$ & auxiliary constants, $a := \tfrac12(\eta_{\max}+\eta_{\min})$, $b := \tfrac12(\eta_{\max}-\eta_{\min})$ \\
		$\hat a, \hat b$ & scaled constants in the noise-term estimate \\
		$\eqsim$ & equivalence up to a constant factor \\
		\bottomrule
	\end{tabular}
\end{center}

\paragraph{Baselines and models.}
\begin{center}
	\small
	\setlength{\tabcolsep}{6pt}
	\renewcommand{\arraystretch}{1.15}
	\begin{tabular}{@{}ll@{}}
		\toprule
		Symbol & Meaning \\
		\midrule
		UBA & Universal Budget Aware LRD \\
		UTA & Universal Training Adaptive LRD \\
		CS & cosine decay LRD \\
		LD & linear decay LRD \\
		WSD & warmup--stable--decay LRD \\
		LLaMA-60M/130M/350M/1B & model scales used in the language experiments \\
		ResNet-50 & convolutional vision backbone \\
		ViT-B & vision transformer backbone \\
		LoRA & low-rank adaptation for PEFT \\
		PEFT & parameter-efficient fine-tuning \\
		\bottomrule
	\end{tabular}
\end{center}

\section{Assumption}\label{Assumption}

We adopt assumptions including local quadratic(inherited from the framework \citep{tang2026stepsize})  and continuous-time (standard for analyzing optimizer dynamics \citep{li2017stochastic,li2025functional}) .
These assumptions provide intuition and design principles that motivate UTA in Section \ref{Methodology}.

We also know that deep learning training landscapes are often highly non-convex and dynamic, so these assumptions may not fully reflect real-world behavior. Importantly, UTA’s robust optimization formulation mitigates approximation errors by optimizing for worst-case scenarios. This ensures the method remains effective even when assumptions are partially violated, as supported by our empirical results.

\section{Problem statement}\label{Problem Statement-appendix}

Let $f(W,\xi)$ denote neural network parameterized by parameters $W\in \mathbb{R}^N$ on the data $\xi$. $\xi$ denote the data sampling on the dataset $D$,  and let $\mathcal{F}$ be a function class. Let $\mathcal{L}$ be the loss function, $T$ be a maximum number of training iterations, and $t$ be the current learning step. 

We seek learning rates $\{ \eta_{t} \}_{t=0}^{T-1} $ that minimize loss of worst-case among the training process under a generalized update rule governed by optimizer-specific dynamics. The global min-max view is
\begin{equation}\label{min-max-adam-appendix}
	\begin{aligned}
		&\min_{\eta_1  , \eta_2  ,..., \eta_{T-1}}\;\max_{f\in\mathcal F}\ \mathcal L\!\big(f(W_T,\xi)\big)\\
		& \quad\text{s.t.}\quad  
		W_{t+1} = W_t - \eta_t \cdot P(t) \nabla\mathcal L(W(t)),
	\end{aligned}
\end{equation}
where $P(t)$ captures the optimizer-specific state dynamics, such as scaling or preconditioning, and may be chosen to represent different optimization algorithms within the same generalized framework. 
For instance, SGD corresponds to $P(t) = I$.
Adam is recovered by defining $P(t)$ such that $P(t) \nabla \mathcal{L}(W(t)) = m_t / (\sqrt{v_t} + \epsilon)$, where $m_t$ and $v_t$ are the first and second moment estimates following Adam update rules.
Other adaptive methods like RMSProp or AdaGrad can be represented analogously through appropriate definitions of $P(t)$ based on their state variables.
This paper primarily focuses on the case where $P(t)$ implements the Adam dynamics.

This paper specializes the framework to \textit{the Adam optimizer} for two compelling reasons. First, the  unique $P(t)$ of Adam encapsulates a complex yet representative interplay of momentum and adaptive scaling mechanisms. It combines first-moment momentum with second-moment adaptive scaling, represents a complex and widely-used class of optimization dynamics. Second, as the default optimizer in much of current deep learning practice, its performance is of paramount practical importance. Analyzing Adam thus serves as a highly relevant and sufficiently comprehensive case study for our theory.
Then the optimization problem is as follows,
\begin{equation}\label{min-max-adam}
	\begin{aligned}
		&\min_{\eta_1  , \eta_2  ,..., \eta_{T-1}}\;\max_{f\in\mathcal F}\ \mathcal L\!\big(f(W_T,\xi)\big)\\
		& \quad\text{s.t.}\quad  
		W_{t+1} = W_t - \eta_t \cdot \frac{\hat{m}_t}{\sqrt{\hat{v}_t} + \epsilon},
	\end{aligned}
\end{equation}
where $
m_t = \beta_1 m_{t-1} + (1-\beta_1)\nabla\mathcal L(W_t),\quad
v_t = \beta_2 v_{t-1} + (1-\beta_2)\nabla\mathcal L(W_t)^{\odot 2},
$
and $\hat m_t = m_t/(1-\beta_1^t)$, $\hat v_t = v_t/(1-\beta_2^t)$ denote the bias-corrected moments, $\odot$ is the element-wise operation.

\section{ Optimal shape of problem solution}\label{Optimal Shape of Problem Solution-appendix}

\begin{lemma}
	\label{lem:segment-radius-continuous}
	Under the generalized gradient flow $\dot W(\tau)=-\eta(\tau) P(\tau) \nabla\mathcal L(W(\tau))$, where $P(\tau)$ denotes precondition. 
	For all $t\in[T_k,T_{k+1}]$,
	\begin{equation}\label{eq:Wt-drift-1}
		\|W(t)-W(T_k)\|
		\le U_k \cdot C_P^{(k)} \cdot G_k\le \eta_{\max,k} \Delta T_k \cdot C_P^{(k)} \cdot G_k,
	\end{equation}
	where $G_k:=\sup_{\tau\in[T_k,T_{k+1})}\|  \nabla\mathcal L(W(\tau))\|$, $C_P^{(k)}:=\sup_{\tau\in[T_k,T_{k+1})} \|P(\tau)\|$,   $\eta_{\max,k} := \sup_{t\in[T_k,T_{k+1}]}\eta(t)$.
\end{lemma}

This Lemma shows that segment radius can be controlled by segment length.

\begin{proof}
	By integrating the flow $\dot W(\tau)=-\eta(\tau) P(\tau) \nabla\mathcal L(W(\tau))$, we have
	\begin{equation}
		W(t)-W(T_k) = \int_{T_k}^t \dot W(s)\,\text{d}s,
	\end{equation}
	so
	\begin{equation}\label{eq:Wt-drift}
		\|W(t)-W(T_k)\|
		\le \int_{T_k}^{t}\eta(\tau)    \|   P(\tau) \nabla\mathcal L(W(\tau))\|\,\text{d}\tau
		\le U_k \cdot C_P^{(k)} \cdot G_k, \qquad t \in [T_k,T_{k+1}),
	\end{equation}
	The segment-radius bound then follows from the triangle inequality:
	\(
	\|W(t)-\bar W^{(k)}\|
	\le \|W(t)-W(T_k)\| + \|W(T_k)-\bar W^{(k)}\|
	\le V_k\Delta T_k + U_k \cdot C_P^{(k)} \cdot G_k.
	\)
\end{proof}

\begin{assumption}[Locally Lipschitz gradients]\label{ass:Locally smooth gradients}
	On the segment $[T_k,T_{k+1})$, the gradient is Lipschitz with constant $L_g^{(k)}$, i.e., $\forall\, W,W' \in \{W(\tau):\tau\in[T_k,T_{k+1})\}, \|\nabla\mathcal L(W)-\nabla\mathcal L(W')\|\le L_g^{(k)}\|W-W'\|$.
\end{assumption}

Combining Assumption~\ref{ass:Locally smooth gradients} with \eqref{eq:Wt-drift-1}, we obtain the segment-wise gradient drift bound
\begin{equation}\label{eq:grad-drift}
	\begin{aligned}
		&\sup_{\tau\in[T_k,T_{k+1})}\|\nabla\mathcal L(W(\tau))-\nabla\mathcal L(W(T_k))\|\\
		\le &L_g^{(k)}\sup_{\tau\in[T_k,T_{k+1})}\|W(\tau)-W(T_k)\|
		\le L_g^{(k)}G_k C_P^{(k)} U_k .
	\end{aligned}
\end{equation}
Since $U_k\leq \Delta T_k \cdot \max\limits_{t\in [T_k,T_{k+1})} \eta_{t}$, if the segment is sufficiently small, the difference of gradients can be bounded.

\begin{lemma}
	\label{lem:m-g}
	On segment $[T_k,T_{k+1})$ with $\Delta T_k:=T_{k+1} -T_k$, let $\hat m_t$ be the bias-corrected first moment
	in Adam and assume:(1) the gradient along the trajectory is $L_g^{(k)}$-Lipschitz in time;(2) the gradient norm is bounded by $G_k$, i.e.,
	$\|\nabla\mathcal L_t\|\le G_k$ for all $t\in[T_k,T_{k+1})$.
	
	Then the bias-corrected momentum $\hat m_t$ tracks the true gradient
	$\nabla\mathcal L_t$ with a segment-wise error bound
	\begin{equation}
		\|\nabla\mathcal L_t-\hat m_t\|
		\le L_g^{(k)} C_P^{(k)} G_k \cdot\Delta T_k \cdot \max\limits_{t\in [T_k,T_{k+1})} \eta_{t} \frac{\beta_1 }{1-\beta_1}.
	\end{equation}
\end{lemma}

\begin{proof}
	It is noted that the first-order moment $\hat{m}_t$ is a convex combination of recent gradients as follows, 
	\begin{equation}
		\hat{m}_t=\sum\limits_{i=0}^{t-1} \frac{(1-\beta_1) \beta_1^i}{1-\beta_1^t} \nabla \mathcal{L}_(t-i), \quad \sum\limits_{i=0}^{t-1} \frac{(1-\beta_1) \beta_1^i}{1-\beta_1^t} =1.
	\end{equation}
	Combining with \eqref{eq:grad-drift}, we have
	\begin{equation}
		\begin{aligned}
			\left\| \nabla \mathcal{L}_t - \hat{m}_t \right\|&=\left\| \nabla \mathcal{L}_t - \sum\limits_{i=0}^{t-1} \frac{(1-\beta_1) \beta_1^i}{1-\beta_1^t} \nabla \mathcal{L}_(t-i) \right\|   \\
			&= \left\|  \sum\limits_{i=0}^{t-1}   \frac{(1-\beta_1) \beta_1^i}{1-\beta_1^t}    (  \nabla \mathcal{L}_t   -\nabla \mathcal{L}_(t-i)   )  \right\|   \\
			&\leq \sum\limits_{i=0}^{t-1}    \frac{(1-\beta_1) \beta_1^i}{1-\beta_1^t}  \cdot L_g^{(k)}G_k \cdot\Delta T_k \cdot \max\limits_{t\in [T_k,T_{k+1})} \eta_{t}\cdot i \\
			&=L_g^{(k)}  C_P^{(k)} G_k \cdot\Delta T_k \cdot \max\limits_{t\in [T_k,T_{k+1})} \eta_{t} \cdot \frac{(1-\beta_1) }{1-\beta_1^t}  \frac{\beta_1 - t \beta_1^t + (t-1) \beta_1^{t+1} }{ (1-\beta_1)^2}\\
			&\leq L_g^{(k)}  C_P^{(k)} G_k \cdot\Delta T_k \cdot \max\limits_{t\in [T_k,T_{k+1})} \eta_{t} \frac{\beta_1 }{1-\beta_1}.
		\end{aligned}
	\end{equation} 
	It yields the claim.
	Hence on a short segment, $\hat{m}_t  \approx \nabla \mathcal{L}_t $.		
\end{proof}

\textbf{Theorem~\ref{thm:uba-optimal-shape}} [Optimal UBA-equivalent shape on a segment]\textit{
	Consider the min-max problem~\ref{min-max-adam} on segment $k$, and assume that the step sizes are constrained to lie in the interval $[\eta_{\min}^{(k)},\eta_{\max}^{(k)}]$. There exists an optimal shape of step sizes that minimizes the worst-case loss. This optimal shape belongs to the UBA-equivalent family for some curvature parameter $\varphi_k>0$. 
}\label{theorem-UBA-optimal-shape-appendix}

\begin{proof}
	By Lemma~\ref{lem:m-g}, it makes precise the intuition that on a short and smooth segment the Adam momentum is close to the current gradient.
	
	The bound 
	\begin{equation}
		\varepsilon_g^{(k)}
		:=
		L_g^{(k)} C_P^{(k)} G_k \cdot\Delta T_k
		\max_{t\in[T_k,T_{k+1})}\eta_t\,
		\frac{\beta_1}{1-\beta_1}.
	\end{equation}
	depends linearly on the segment length
	$\Delta T_k$ and on the maximal stepsize on the segment. Thus, by choosing $\Delta T_k$ small and avoiding excessively large $\eta_t$, we can make $\varepsilon_g^{(k)}$ small. In other words, on sufficiently short segments we may treat the Adam momentum $\hat m_t$ as a perturbation of the
	true gradient $\nabla\mathcal L_t$, with a perturbation norm controlled by $\varepsilon_g^{(k)}$.
	
	Thus, by Lemma~\ref{lem:m-g} and following the local quadratic modeling in prior work \citep{tang2026stepsize}, we can approximate the constraint in UBA framework\citep{tang2026stepsize} as follows, 
	\[
	W_{t+1}-\bar W^{(k)}
	\approx \big(I-\eta_t P_t H_k\big)(W_t-\bar W^{(k)}) ,
	\]
	where $P_t:=\frac{1}{\sqrt{\hat{v}_t} + \epsilon} $. 
	Denote $P_t H_k:=\bar{H_k}$, by iteratively applying this equation, we obtain 
	\begin{equation}\label{eq4}
		\begin{aligned}
			&W_{T_{k+1}} -\bar{W}^{(k)}
			=\left(I-\eta_{T_{k+1}-1} \bar{H_k}\right)\cdots \left(I-\eta_{T_{k}+1} \bar{H_k}\right)  \left( W_{T_{k}+1 } -\bar{W}^{(k)}\right).
		\end{aligned}
	\end{equation}

	Since the matrix $\bar{H_k} \in \mathbb{R}^{N\times N} $  possesses $N$ eigenvalues $\lambda_1^{(k)}(f(\xi)), \lambda_2^{(k)} (f(\xi)), \cdots , \lambda_N^{(k)}(f(\xi)) $ abbreviated as $\lambda_i^{(k)}$, which satisfy $\lambda_{l}^{(k)}\leq\lambda_i^{(k)} (f(\xi))\leq  \lambda_{u}^{(k)}$ for all $i$. 
	Here, $\lambda_{l}^{(k)}$ and $\lambda_{u}^{(k)}$ denote the bounds on the non-zero eigenvalues of the $\bar{H_k} $ of  optimum in $k$-th segment.  
	Moreover, there exist $N$ linearly independent eigenvectors $u_1^{(k)}(f(\xi)), u_2^{(k)}(f(\xi)), \cdots , u_N^{(k)} (f(\xi)) $ abbreviated as $u_i^{(k)}$, which can form the basis for the $N$ dimension vector space. Then we have 
	$\bar{H_k} u_i^{(k)}= \lambda_i^{(k)}  u_i^{(k)} \quad (i=1,2,\cdots, N )$ 
	and the initial deviation from the optimal solution can be expressed as $W_{T_{k}} -\bar{W}^{(k)} = \sum\limits_{i=1}^{N}   s_i^{(k)} u_i^{(k)}$ 
	where $s_i^{(k)}$ are the coefficients corresponding to the eigenvector components. 
	Thus, we obtain 
	\begin{equation}\label{eq7}
		\begin{aligned}
			\left\| W_{T_{k+1}} -\bar{W}^{(k)} \right\|_2^2 =& \sum\limits_{i=1}^{N}   (s_i^{(k)}  )^2  \|u_i^{(k)}\|_2^2  \left[ \prod\limits_{t=1+T_{k} }^{T_{k+1}} \left( 1-\eta_{t}\lambda_i^{(k)} \right)  \right]^2 \\
			\leq & \sum\limits_{i=1}^{N}   (s_i^{(k)}  )^2  \|u_i^{(k)}\|_2^2  \max_{\lambda_{l}^{(k)} \leq \lambda_i^{(k)} \leq \lambda_{u} ^{(k)}} \left[ \prod\limits_{t=1+T_{k} }^{T_{k+1}} \left( 1-\eta_{t}\lambda_i^{(k)} \right)  \right]^2\\
			=&   \max_{\lambda_{l}^{(k)} \leq \lambda_i^{(k)} \leq \lambda_{u} ^{(k)}} \left[ \prod\limits_{t=1+T_{k} }^{T_{k+1}} \left( 1-\eta_{t}\lambda_i^{(k)} \right)  \right]^2  \left\| W_{T_{k}} -\bar{W}^{(k)} \right\|_2^2
		\end{aligned}
	\end{equation}
	We observe that this formulation shares the same optimization form as the UBA framework\citep{tang2026stepsize}.
	Thus, in each short segment, the problem \eqref{min-max-adam} reduces to UBA problem. 
	The solution of problem \eqref{min-max-adam} belongs to the UBA-equivalent family.
\end{proof}

\section{Regularized energy evolution via intrinsic time}\label{appendix-Regularized-Energy-Evolution}

Denote $ \mathbb E_{\xi_t} \nabla\mathcal L(W_t,\xi_t)  = \nabla\mathcal L(W_t) $, we have $\nabla\mathcal L(W_t,\xi_t) = \nabla\mathcal L (W_t) + \delta_t $, where the gradient noise $\delta_t$ satisfies $\mathbb E (\delta_t) =0 $.
Since the segment is short, the stochastic preconditioned gradient flow be modeled in accordance with SGD noise \citep{li2025functional} scaling by
\begin{equation}\label{eq:sde-seg-appendix}
	\text{d}W(t)= -\eta (t)\,P(t)\nabla \mathcal L(W(t))\,\text{d}t
	\;+\;\eta (t)\,P(t)\Sigma^{1/2}\,\text{d}B(t),
\end{equation}
where $W(t)$ is the weight parameter, $\eta(t)\ge 0$, $P(t)\succ 0$ is diagonal, and $\Sigma\succeq 0$ is the covariance of the gradient noise and $B(t)$ is a standard Brownian motion.

According to the second-order necessary condition for the strict minimum $\bar{W}^{(k)}$, this approximation is achieved through the positive semi-definite Hessian matrix $H_f^{(k)} (\xi) \in \mathbb{R}^{N\times N} $, $H_k\succeq 0$. 
By quadratic expansion, the loss function can be formulated as $\mathcal{L}(f(\bar{W}^{(k)},\xi))+ \nabla\mathcal{L}(f(\bar{W}^{(k)},\xi)) \left(W_{T_{k+1}} -\bar{W}^{(k)}\right)+
\frac{1}{2}\left(W_{T_{k+1}} -\bar{W}^{(k)}\right)^{\top} H_f^{(k)} (\xi) \left(W_{T_{k+1}} -\bar{W}^{(k)}\right) \quad (k=1,2,\cdots, K)$. 
Since the neural networks possess sufficient capacity to fit the data, the loss at the optimal point can be made negligible. Thus the term $\mathcal{L}(f(\bar{W}^{(k)},\xi))$ can be reasonably neglected. In addition, the first-order condition $\nabla\mathcal{L}(f(\bar{W}^{(k)},\xi))=0$ at the optimal point $\bar{W}^{(k)}$. Thus the loss function can be measured by $
\frac{1}{2}\left(W_{T_{k+1}} -\bar{W}^{(k)}\right)^{\top} H_f^{(k)} (\xi) \left(W_{T_{k+1}} -\bar{W}^{(k)}\right) \quad (k=1,2,\cdots, K)$

Under this quadratic expansion measurement, we define the energy function between the $t$-th step weights and local minimum within segment 
\begin{equation}\label{eq:energy-def-prop}
	E(t):=\frac{1}{2}\|e(t)\|^2, \qquad e(t):= H_k^{1/2} \left( W(t)-\bar W^{(k)}\right).
\end{equation}

with the above assumption, we have 

\textbf{Theorem~\ref{thm:bias-variance-intrinsic}} [Bias-variance-style upper bound via intrinsic time]\textit{
	Consider the $k$-th segment $[T_k,T_{k+1}]$ and the preconditioned stochastic
	gradient flow \eqref{eq:sde-seg-appendix}, 
	the expected energy at the segment end satisfies the bias--variance bound
	\begin{equation}\label{eq:bv-general-appendix}
		\mathbb E[E(T_{k+1})]
		\;\le\;
		\underbrace{\mathbb E[E(T_k)]\,e^{-2\nu_\ell^{(k)}U_k}}_{\text{bias (contraction)}}
		\;+\;
		\underbrace{\frac12\alpha\int_{T_k}^{T_{k+1}}\eta(t)^2\,\text{d} t}_{\text{variance (noise accumulation)}}.
	\end{equation}
	where $\alpha := \sup_{t\in[T_k,T_{k+1}]}\operatorname{tr}\!\big(H_k^{1/2}P(t)\Sigma P(t)H_k^{1/2}\big)$ is the noise intensity bound, and $\lambda_{\min}(M(t))\ge \nu_\ell^{(k)} >0, M(t):=H_k^{1/2}P(t)H_k^{1/2}$ is the spectral lower bound.
}

\begin{proof}
	According to   \eqref{eq:sde-seg-appendix} and local quadratic approximation $\nabla \mathcal L(W(t)) = H_k\left(W(t)-\bar W^{(k)}\right)$, we have 
	\begin{equation}\label{de(t)-appendix}
		\begin{aligned}
			\text{d}e(t)&= H_k^{1/2} \text{d} W(t)\\
			&=-\eta (t) \, H_k^{1/2} \,  P(t)H_k\left(W(t)-\bar W^{(k)}\right) \,\text{d}t
			\;+\;\eta (t)\, H_k^{1/2} \,P(t)\Sigma^{1/2}\,\text{d}B(t)\\
			&=-\eta (t) \, H_k^{1/2} \,  P(t) H_k^{1/2} e(t)\,\text{d}t
			\;+\;\eta (t)\, H_k^{1/2} \,P(t)\Sigma^{1/2}\,\text{d}B(t).
		\end{aligned}
	\end{equation}
	
	Consider $E(t)$ as a function of $e(t) $and apply the It\^{o} formula:
	\begin{equation}\label{dE(t)-appendix}
		\text{d}E(t) = \nabla_e E(t)^\top \text{d}e(t) 
		+ \frac{1}{2} \text{tr}\left[ \text{d}e(t) \text{d}e(t)^\top \nabla_e^2 E(t) \right],
	\end{equation}
	where $\nabla_e E(t) = e(t)$ and $\nabla_e^2 E(t) = I$.
	
	Substituting   \eqref{de(t)-appendix} into \eqref{dE(t)-appendix}, the first term is
	\begin{equation}
		\begin{aligned}
			\nabla_e E(t)^\top \text{d}e(t) 
			&= e(t)^\top \left[ -\eta (t) \, H_k^{1/2} \,  P(t) H_k^{1/2} e(t)\,\text{d}t
			\;+\;\eta (t)\, H_k^{1/2} \,P(t)\Sigma^{1/2}\,\text{d}B(t) \right] \\
			&= -\eta(t) e(t)^\top \, H_k^{1/2} \,  P(t) H_k^{1/2} e(t) \text{d}t 
			+ \eta(t) e(t)^\top H_k^{1/2} P(t) \Sigma^{1/2} \text{d}B(t).
		\end{aligned}
	\end{equation}
	For the part in second term
	\begin{align}
		\text{d}e(t) \text{d}e(t)^\top 
		=&\left[ -\eta (t) \, H_k^{1/2} \,  P(t) H_k^{1/2} e(t)\,\text{d}t
		\;+\;\eta (t)\, H_k^{1/2} \,P(t)\Sigma^{1/2}\,\text{d}B(t) \right]\\
		\cdot & \left[ -\eta (t) \, H_k^{1/2} \,  P(t) H_k^{1/2} e(t)\,\text{d}t
		\;+\;\eta (t)\, H_k^{1/2} \,P(t)\Sigma^{1/2}\,\text{d}B(t) \right]^\top \\
		=& \left[ \eta (t)\, H_k^{1/2} \,P(t)\Sigma^{1/2}\,\text{d}B(t) \right] 
		\left[ \eta (t)\, H_k^{1/2} \,P(t)\Sigma^{1/2}\,\text{d}B(t) \right]^\top \\
		=& \eta^2(t) H_k^{1/2}  P(t) \Sigma^{1/2} \text{d}B(t)\text{d}B(t)^\top \Sigma^{1/2\top} P(t)^\top \left(H_k^{1/2}\right) ^\top\\
		=& \eta^2(t) H_k^{1/2} P(t) \Sigma P(t)^\top H_k^{1/2} \text{d}t.
	\end{align}
	The second equality is because multiplication rules in stochastic differential form. 
	In addition, The second equality is because $\text{d}B(t)\text{d}B(t)^\top = I \text{d} t$. Hence
	\begin{equation}
		\frac{1}{2} \text{tr}\left[ \text{d} e(t) \text{d}e(t)^\top \nabla_e^2 E(t) \right]
		= \frac{1}{2} \eta^2(t) \text{tr}\left[ P(t) \Sigma P(t)^\top  H_k \right] \text{d}t.
	\end{equation}
	Putting together
	\begin{equation}
		\begin{aligned}
			\text{d} E(t)
			&= \left[ -\eta(t) e(t)^\top \, H_k^{1/2} \,  P(t) H_k^{1/2} e(t) \right.\\
			&\quad \left. + \frac{1}{2} \eta^2(t) \text{tr}\left( P(t) \Sigma P(t)^\top  H_k \right) \right] \text{d} t \\
			&\quad + \eta(t) e(t)^\top H_k^{1/2} P(t) \Sigma^{1/2}  \text{d} B(t).
		\end{aligned}
	\end{equation}
	Taking expectations, the martingale term has zero mean, yielding
	\begin{equation}
		\begin{aligned}
			\frac{\text{d}}{\text{d}t}\mathbb E[E(t)]
			&= -\eta (t)\,\mathbb E\!  \left[e(t)^\top H_k^{1/2} \,  P(t) H_k^{1/2} e(t) \right] \\
			&\quad +\frac12\eta (t)^2 \text{tr} \left( P(t) \Sigma P(t)^\top  H_k \right).
		\end{aligned}
	\end{equation}
	
	Since $P(t)$ is symmetric positive definite and $H_k^{1/2}$ is symmetric positive semi-definite, there exists eigenvalue lower bound of $\nu_\ell^{(k)}>0$ of $H_k^{1/2} \,  P(t) H_k^{1/2}$ such that
	\begin{equation}\label{eq:nu-lower-prop}
		e(t)^\top H_k^{1/2} \,  P(t) H_k^{1/2} e(t)  \;\ge\; \nu_\ell^{(k)}\|e(t)\|^2,
		\qquad   \forall t\in[T_k,T_{k+1}).
	\end{equation}
	Therefore
	\begin{equation}\label{linear-inequality}
		\frac{d}{dt}\mathbb E[E(t)]
		\le -2\eta (t) \nu_\ell^{(k)}  \,\mathbb E[E(t)] + \frac12 \alpha \eta(t)^2,
	\end{equation}
	where we denote $\alpha =\text{tr} \left( P(t) \Sigma P(t)^\top  H_k \right) $.
	
	Solving the linear inequality via an integrating factor $ \mu (t):=\exp \left(\int_{T_k}^{t} \eta( \tau) \nu_\ell^{(k)} \text{d} \tau \right) , \quad t \in [T_k, T_{k+1}] $, 
	\begin{equation}
		\frac{\text{d}}{\text{d} t} \mu(t) = \nu_\ell^{(k)} \eta(t) \mu(t).
	\end{equation}

	Multiplying $\mu(t)$ to   \eqref{linear-inequality}, i.e., $\mu(t) \frac{\text{d}}{\text{d} t} \mathbb E[E(t)] + 2\eta(t) \nu_\ell^{(k)} \mu(t) \mathbb E[E(t)] 
	\le \frac{1}{2} \alpha \eta(t)^2 \mu(t)$,we have
	\begin{align}
		\frac{\text{d}}{\text{d} t} \big[ \mu(t) \mathbb E[E(t)] \big] 
		\le \frac{1}{2} \alpha \eta(t)^2 \mu(t).
	\end{align}
	
	Integrate from $T_k$ to $T_{k+1}$, we have
	\begin{equation}
		\mu(T_{k+1}) \mathbb E[E(T_{k+1})] - \mu(T_k) \mathbb E[E(T_k)] 
		\le \frac{1}{2} \alpha \int_{T_k}^{T_{k+1}} \eta(t)^2 \mu(t) \text{d} t.
	\end{equation}
	
	Since $\mu(T_k)=1$ and $U_k = \int_{T_k}^{T_{k+1}} \eta(\tau)\, \text{d}\tau$ , we have
	\begin{equation}
		\begin{aligned}
			\mathbb E[E(T_{k+1})]
			\le
			\mathbb E[E(T_k)]\,e^{-2\nu_\ell^{(k)} U_k }
			+\frac{1}{2} \alpha  \,e^{-\nu_\ell^{(k)} U_k } \int_{T_k}^{T_{k+1}} \eta(t)^2 \mu(t) \text{d} t.
		\end{aligned}
	\end{equation}
	In typical network training, the learning rate satisfies $\eta( t)>0$. Consequently, we have
	$ \mu (t):=\exp \left(\int_{T_k}^{t} \eta( \tau) \nu_\ell^{(k)} \text{d} \tau \right)
	\leq \exp \left( \nu_\ell^{(k)}U_k \right)  $,	
	which is \eqref{eq:bv-general-appendix}.
\end{proof}

\textbf{Corollary~\ref{cor:l1-Uk}} [$\ell_1$-regularized intrinsic-time problem]\textit{
	Under Theorem~\ref{thm:bias-variance-intrinsic} and the constraint	$0\le \eta(t)\le \eta_{\max}$ on $[T_k,T_{k+1}]$, the intrinsic-time selection problem
	reduces to the $\ell_1$-regularized problem, 
	\begin{equation}
		\min_{U\in[0,\eta_{\max}\Delta T_k]}
		\Big\{
		\mathbb E[E(T_k)]e^{-2\nu_\ell^{(k)}U} + \frac12\alpha\eta_{\max}U
		\Big\}.
	\end{equation}
	The unconstrained minimizer admits the closed form
	\begin{equation}\label{eq:l1-Uk-star-appendix}
		U_k^\ast
		=
		\frac{1}{2\nu_\ell^{(k)}}
		\ln\!\left(\frac{4\nu_\ell^{(k)}\mathbb E[E(T_k)]}{\alpha\eta_{\max}}\right).
	\end{equation}
	The constrained minimizer is followed by clamping to $[0,\eta_{\max}\Delta T_k]$.
}

\begin{proof}
	Since $0\le \eta(t)\le \eta_{\max}$, we have $\eta(t)^2\le \eta_{\max}\eta(t)$, hence
	\begin{equation}\label{eq:xxx}
		\int \eta(t)^2dt \le \eta_{\max}\int \eta(t)\text{d}t = \eta_{\max}U_k.
	\end{equation}
	Substituting \eqref{eq:xxx} into \eqref{eq:bv-general-appendix}, we arrive at the simplified bound
	\begin{equation}\label{l1-problem-appendix}
		\begin{aligned}
			\mathbb E[E(T_{k+1})]
			\le
			\mathbb E[E(T_k)]\,e^{-2\nu_\ell^{(k)} U_k }
			+\frac{1}{2} \alpha  \eta_{\max}   U_k.
		\end{aligned}
	\end{equation}
	Consequently, selecting $U_k$ by minimizing the upper bound \eqref{l1-problem-appendix} reduces to the $\ell_1$-regularized problem
	\begin{equation}\label{l1-problem-appendix-2}
		\min_{U\in[0,\eta_{\max}\Delta T_k]}
		\Big\{
		\mathbb E[E(T_k)]e^{-2\nu_\ell^{(k)}U} + \frac12\alpha\eta_{\max}U
		\Big\}.
	\end{equation}
	Differentiating the objective in \eqref{l1-problem-appendix-2} gives the condition
	$2\nu_\ell^{(k)}\mathbb E[E(T_k)]e^{-2\nu_\ell^{(k)}U} = \frac12\alpha\eta_{\max}$.
	Then it gives the explicit form optimal intrinsic time,
	\begin{equation}
		\begin{aligned}
			\hat{  U_k^{\ast}}=\frac{1}{2\nu_\ell^{(k)}} \ln \left( \frac{ 4 \nu_\ell^{(k) }\mathbb E[E(T_k)] }{\eta_{\max}  \alpha}\right).
		\end{aligned}
	\end{equation}
	Note that $U_k\in \left[ 0 , \eta_{\max} \left( T_{k+1}-T_k \right)\right]$, therefore the actual $U_k^{\ast}$ is obtained by clamping $\hat{  U_k^{\ast}}$.
	It rearranges to \eqref{eq:l1-Uk-star-appendix}.
\end{proof}

\textbf{Corollary~\ref{cor:l2-near-constant}} [$\ell_2$ surrogate via near-constant stepsizes]\textit{
	Under Theorem~\ref{thm:bias-variance-intrinsic} and continuity of UBA-shape function, the intrinsic-time selection problem can be approximated as $\ell_2$-regularized problem, 
	\begin{equation}\label{eq:l2-obj-const-appendix}
		\min_{U_k\in[0,\eta_{\max}\Delta T_k]}
		\mathbb E[E(T_k)]e^{-2\nu_\ell^{(k)}U_k}
		+ \frac12\alpha\frac{U_k^2}{\Delta T_k},
	\end{equation}
	whose minimizer is given by the Lambert-$W$ expression
	\begin{equation}\label{L2-Uk-appendix}
		U_k^\ast
		= \frac{1}{2\nu_\ell^{(k)}}\,W\!\left(
		\frac{2(\nu_\ell^{(k)})^2 \mathbb E[E(T_k)] \Delta T_k}{\alpha}
		\right),
	\end{equation}
	followed by clamping to $[0,\eta_{\max}\Delta T_k]$.
}

\begin{proof}

	Given the continuity of UBA-shape function, the stepsize on a short segment $k$
	is slowly varying so that $\eta(t)\approx \bar\eta_k:=U_k/\Delta T_k$.
	Then
	\begin{equation}
		\int_{T_k}^{T_{k+1}}\eta(t)^2\text{d}t \approx \frac{U_k^2}{\Delta T_k}.
	\end{equation}
	Under Theorem~\ref{thm:bias-variance-intrinsic}, substituting into \eqref{eq:bv-general-appendix} yields the $\ell_2$-regularized objective
	\begin{equation}\label{eq:segment-regularized-objective-l2-appendix}
		\min_{U_k\in[0,\eta_{\max}\Delta T_k]}
		\mathbb E[E(T_k)]e^{-2\nu_\ell^{(k)}U_k}
		+ \frac12\alpha\frac{U_k^2}{\Delta T_k},
	\end{equation}

	Differentiating~\eqref{eq:segment-regularized-objective-l2-appendix} with respect to $U_k$ gives
	\begin{equation}
		J_k(U_k)
		= -2\nu_\ell^{(k)} \mathbb E[E(T_k)] \, e^{\big(-2\nu_\ell^{(k)}U_k\big)}
		+ \frac{\alpha}{\Delta T_k} U_k.
	\end{equation}
	Setting $J_k(U_k)=0$ yields
	\begin{equation}\label{eq:Uk-FOC}
		\frac{\alpha}{\Delta T_k} U_k
		= 2\nu_\ell^{(k)} \mathbb E[E(T_k)]\,e^{\big(-2\nu_\ell^{(k)}U_k\big)}.
	\end{equation}
	Multiplying both sides of~\eqref{eq:Uk-FOC} by $2\nu_\ell^{(k)}e^{\big(2\nu_\ell^{(k)}U_k\big)}$ leads to
	\begin{equation}
		2\nu_\ell^{(k)}\frac{\alpha}{\Delta T_k} U_k\,e^{\big(2\nu_\ell^{(k)}U_k\big)}
		= 2\nu_\ell^{(k)}\nu_\ell^{(k)}\mathbb E[E(T_k)].
	\end{equation}
	The above equation becomes
	\begin{equation}
		\left(2\nu_\ell^{(k)}U_k \right) e^{\left(2\nu_\ell^{(k)}U_k \right)} = \frac{2\left(\nu_\ell^{(k)}\right)^2 \mathbb E[E(T_k)] \Delta T_k}{\alpha}.
	\end{equation}
	By definition of the Lambert $W$ function, $2\nu_\ell^{(k)}U_k = W\!\left(\frac{2\left(\nu_\ell^{(k)}\right)^2 \mathbb E[E(T_k)] \Delta T_k}{\alpha}\right)$, and thus
	\begin{equation}
		U_k^\ast
		= \frac{1}{2\nu_\ell^{(k)}}\,W\!\left(
		\frac{2(\nu_\ell^{(k)})^2 \mathbb E[E(T_k)] \Delta T_k}{\alpha}
		\right),
	\end{equation}
	which is~\eqref{L2-Uk-appendix}.
\end{proof}

\paragraph{A computable plug-in estimate of $U_k^\ast$.}
\label{appendix-Regularized-Energy-Evolution-computable plug-in estimate}

Both closed-form solutions in \eqref{eq:l1-Uk-star-appendix} and 
\eqref{L2-Uk-appendix} depend on the quantities 
$\nu_\ell^{(k)}$, $\mathbb E[E(T_k)]$, and $\alpha$. 
Direct estimation of these quantities is computationally expensive. 
We therefore derive a plug-in estimate based on measurable surrogates.

For the $\ell_1$ solution \eqref{eq:l1-Uk-star-appendix}, the logarithmic term 
is dominated by the ratio $\nu_\ell^{(k)}\mathbb E[E(T_k)]/\alpha$. 
For the $\ell_2$ solution \eqref{L2-Uk-appendix}, using the first-order 
approximation $W(x)\approx x$ for small $x$, we obtain
$
U_k^\ast \approx \frac{\nu_\ell^{(k)}\mathbb E[E(T_k)]\Delta T_k}{\alpha}.
$
Both solutions share the common dominant factor
\begin{equation}\label{eq:dominant-factor}
	\frac{\nu_\ell^{(k)}\mathbb E[E(T_k)]}{\alpha}.
\end{equation}

Recall that the energy function is defined as the local quadratic surrogate of 
the true loss $\mathcal L$:
\begin{equation}
	E(t) := \frac{1}{2}\|H_k^{1/2}(W(t)-\bar W^{(k)})\|^2
	\approx \mathcal L(W(t)) - \mathcal L(\bar W^{(k)})
	\approx \mathcal L(W(t)),
\end{equation}
where the last approximation follows from the assumption that the loss at the 
local minimum $\mathcal L(\bar W^{(k)})$ is negligible. 
Under this identification, $E$ and $\mathcal L$ are used interchangeably 
in the following derivations.

For the symmetric positive semidefinite Hessian $H_k = \nabla^2 \mathcal L(\bar W^{(k)})$, 
the Rayleigh--Ritz variational formula characterizes its smallest positive 
eigenvalue as
\begin{equation}
	\nu_\ell^{(k)} = \lambda_{\min}^+(H_k) 
	= \min_{v \perp \ker(H_k)} 
	\frac{v^T H_k v}{\|v\|^2}.
\end{equation}
Consequently, for any update direction $\Delta W$,
\begin{equation}
	\Delta W^T H_k \Delta W 
	\ge \nu_\ell^{(k)} \|\Delta W\|^2.
\end{equation}

Let the actual parameter update be $\Delta W_k = W_{k+1} - W_k =   d_k$, 
where $d_k := \hat{m}_t/(\sqrt{\hat{v}_t}+\epsilon)$ 
is the normalized Adam update direction. 
Using the quadratic expansion of $E$ around $\bar W^{(k)}$, the single-step 
energy decrease satisfies
\begin{equation}
	\Delta E_k := E(W_k) - E(W_{k+1})
	\approx \frac{1}{2} \nu_\ell^{(k)} \|\Delta W_k\|^2.
\end{equation}
Substituting $\Delta W_k =   d_k$ yields
\begin{equation}
	\nu_\ell^{(k)} \approx \frac{2\Delta E_k}{  \|d_k\|^2},
\end{equation}
where $\approx$ denotes less-than or approximately less-than.

Summing the single-step decrease over $k = 0, \dots, T_k-1$ gives
\begin{equation}
	\nu_\ell^{(k)} \sum_{i=0}^{T_k-1} \|\Delta W_i\|^2 
	\approx \sum_{i=0}^{T_k-1} \Delta E_i = \left| E(T_k) - E(0) \right|.
\end{equation}
Assuming the curvature $\nu_\ell^{(k)}$ varies slowly within the segment and 
that the update directions satisfy $\|\Delta W_i\| \approx   \|d_k\|$ 
for $i$ near $T_k$, we obtain
\begin{equation}
	\nu_\ell^{(k)} T_k   \|d_k\|^2 \approx  \left| E(T_k) - E(0) \right| .
\end{equation}
For large $T_k$, the initial energy $E(0)$ is negligible compared with 
$E(T_k)$, so
\begin{equation}
	\nu_\ell^{(k)} E(T_k) \approx \frac{\left| E(T_k) - E(0) \right| E(T_k)}{T_k   \|d_k\|^2}.
\end{equation}

Since $\left| E(T_k) - E(0) \right| E(T_k)$ is related to loss function and the loss decrease, which is a small amount.
Let $T_k$ denote the number of steps required to reach a target accuracy 
$\varepsilon$, so that $T_k \approx \left(\left| E(T_k) - E(0) \right| E(T_k)\right)/\varepsilon$. 
Substituting this relation gives
\begin{equation}
	\frac{\nu_\ell^{(k)} E(T_k)}{\alpha} 
	\approx \frac{1}{\alpha} \cdot \frac{\varepsilon}{  \|d_k\|^2}.
\end{equation}

The quadratic form obtained above, i.e.,
\begin{equation}
	\frac{\nu_\ell^{(k)} E(T_k)}{\alpha} 
	\approx \frac{1}{\alpha} \cdot \frac{\varepsilon}{  \|d_k\|^2},
\end{equation}
provides a theoretically characterization of the dominant factor. 
However, when employed as a plug-in estimate for the intrinsic-time 
scale $U_k^\ast$ in the UBA framework, the inverse quadratic dependence 
on $\|d_k\|$ may lead to numerical instability when $\|d_k\|$ becomes 
small, as the estimate can grow arbitrarily large without saturation.

To address this issue while preserving the qualitative behavior of the 
original estimate, we replace the quadratic inverse $\|d_k\|^{-2}$ 
with the reciprocal-linear surrogate $(1 + \|d_k\|)^{-1}$. 
This substitution is justified by the following considerations:  1)\emph{Monotonicity preservation:} Both forms are strictly 
decreasing functions of $\|d_k\|$, ensuring that larger updates 
correspond to smaller intrinsic-time scales. 2) \emph{Asymptotic consistency:}   3) \emph{Numerical stability:} The reciprocal-linear form 
is naturally bounded in $(0, 1]$, preventing unbounded growth 
when $\|d_k\|$ approaches zero, which is crucial for robust 
estimation in practice.
Thus, 
we obtain
\begin{equation}
	\frac{\nu_\ell^{(k)} E(T_k)}{\alpha} 
	\approx \frac{1}{1 + C \|d_k\|},
\end{equation}
where $C := \sqrt{\varepsilon/(\alpha  )}$ absorbs the constants.

We set $C$ as a hyper-parameter, for the consistency of the initial $\varphi_0$
This yields the stabilized plug-in estimate
\begin{equation}\label{a-computable-plug}
	U_t^{\ast} = \frac{1}{1 + \sqrt{\varphi_0/2} 
		\left\|\frac{\hat{m}_t}{\sqrt{\hat{v}_t}+\epsilon}\right\|
	}
\end{equation}
for the intrinsic-time energy scale.

\section{Derived LRD}\label{Derived-LRD-appendix}

Since the optimal solution of problem \eqref{min-max-adam0} within each short segment is given by a member of  UBA-equivalent family (Definition~\eqref{definition:UBA-family}),  $U_k$ is a deterministic function of the segment-specific shape parameters such as $\eta_{\min}^{(k)},\eta_{\max}^{(k)}$ and $\varphi_k$.

To finalize the design, we calibrate the scale of the UTA LRD by \emph{matching} two forms of intrinsic time:
\begin{itemize}
	\item The energy-optimal intrinsic time $U_k^\ast$ \eqref{varphi_0}  from regularized energy evolution analysis (Subsection~\ref{Regularized Energy Evolution}).
	\item The intrinsic time implied by the UBA-equivalent shape in Definition~\eqref{definition:UBA-family}, i.e., $U_{\varphi_k}=\int_{T_k}^{T_{k+1}} g_{a,b}\left(\varphi,\theta_{t,k}\right) \text{d}t$.
\end{itemize}
Concretely, for desired parameterization of the shape, such as fixing $b=\eta_{\min}^{(k)}$ and $a=\left(\eta_{\max}^{(k)}- \eta_{\min}^{(k)} \right)$, solving for $\varphi_k$, we solve the scalar equation
\begin{equation}\label{eq:Uk-matching-appendix}
	U_{\varphi_k}
	= U_k^\ast ,
\end{equation}
which uniquely determines the remaining free parameter on the segment.
This yields a per-segment effective LRD that is (i) near-optimal in shape for problem~\eqref{min-max-adam0}, by Theorem~\ref{thm:uba-optimal-shape}, and (ii) energy-calibrated in scale, by subsection~\ref{Regularized Energy Evolution}. 

Next, we solve the matching equation \eqref{eq:Uk-matching-appendix} and derive the explicit form of the parameter $\varphi_k$.
Without loss of generality, set \(\eta_{\max}^{(k)}=1\) and \(\eta_{\min}^{(k)}=0\), we proceed as follows:

\begin{theorem}\label{explicit-varphi-k-appendix}
	For a given energy-optimal intrinsic time \(U_k^{\ast}\), there exists an explicit
	solution of \(\varphi_k\) to matching \(U_{\varphi_k}= U_k^\ast\), expressed
	analytically as
	\begin{equation}
		\varphi_k
		= \frac{2\left(1- \psi \right)^2}{\psi^2},
	\end{equation}
	where
	\begin{equation}
		\psi= \frac{U_k^{\ast}}{\Delta T_k}.
	\end{equation}
\end{theorem}

\begin{proof}
	Without loss of generality, let \(\eta_{\max}^{(k)}=1\) and \(\eta_{\min}^{(k)}=0\).
	By definition of \(\theta_k(t)\) and \(\tau:=t-T_k\), we have
	\[
	\frac{d\theta_k}{d\tau}
	= \frac{\pi}{\Delta T_k},
	\qquad
	d\tau = \frac{\Delta T_k}{\pi}\,d\theta_k.
	\]
	Thus
	\begin{equation}\label{eq:Uk-proof-start-normalized}
		\int_{T_k}^{T_{k+1}}\eta_k(t)\,\text{d}t
		= \int_{0}^{\Delta T_k} \eta_k(T_k+\tau)\,\text{d}\tau
		= \frac{\Delta T_k}{\pi}
		\int_{0}^{\pi} g_k(\theta)\,\text{d}\theta,
	\end{equation}
	where
	\begin{equation}
		g_k(\theta)
		:= \frac{2\big(1 + \cos(\theta + (k-1)\pi)\big)}
		{2\varphi_k + (2-\varphi_k)\big(1 + \cos(\theta + (k-1)\pi)\big)}.
	\end{equation}
	Using \(\cos(\theta + (k-1)\pi) = (-1)^{k-1}\cos\theta\) and the substitution
	\(\theta' = \pi - \theta\) (so that \(\cos\theta' = -\cos\theta\) and the
	integration limits remain \([0,\pi]\)), one checks that
	\begin{equation}
		\int_{0}^{\pi} g_k(\theta)\,\text{d}\theta
		= \int_{0}^{\pi} g_1(\theta)\,\text{d}\theta
		=: I(\varphi_k),
	\end{equation}
	so the accumulated step \(U_k\) is independent of \(k\) and it suffices to compute
	\begin{equation}
		I(\varphi)
		= \int_{0}^{\pi}
		\frac{2\big(1 + \cos\theta\big)}
		{2\varphi + (2-\varphi)\big(1 + \cos\theta\big)}\,\text{d}\theta.
	\end{equation}
	Let \(a := \varphi + 2\) and \(b := 2-\varphi\), so that
	\begin{equation}
		I(\varphi)
		= \int_{0}^{\pi} \frac{2(1+\cos\theta)}{a + b\cos\theta}\,\text{d}\theta.
	\end{equation}
	We decompose
	\begin{equation}
		\begin{aligned}
			\int_{0}^{\pi} \frac{1+\cos\theta}{a + b\cos\theta} d\theta
			&=\int_{0}^{\pi} \frac{1}{a+b\cos\theta} \text{d}\theta
			+ \int_{0}^{\pi} \frac{\cos\theta}{a+b\cos\theta} \text{d}\theta\\
			&:=I_1+ I_2.
		\end{aligned}
	\end{equation}
	\begin{equation}
		\begin{aligned}
			\int_{0}^{\pi} \frac{\text{d}\theta}{a+b\cos\theta}
			&= \frac{2}{\sqrt{a^2-b^2}} \arctan\left( \sqrt{\frac{a-b}{a+b}} \tan\frac{\theta}{2} \right) \Big|_{0}^{\pi},\\
			&=\frac{\pi}{ \sqrt{ a^2-b^2 } }.
		\end{aligned}
	\end{equation}
	Since
	\begin{equation}
		I_2=\frac{\pi}{b}-\frac{a}{b} I_1,
	\end{equation}
	and using the standard integral
	\(\int_0^\pi \text{d}\theta/(a + b\cos\theta) = \pi/\sqrt{a^2-b^2}\)
	for \(a>|b|\), we have
	\begin{equation}
		\begin{aligned}
			\int_{0}^{\pi} \frac{1+\cos\theta}{a + b\cos\theta} \text{d}\theta
			&=\frac{\pi}{ \sqrt{ a^2-b^2 } }
			+ \left( \frac{\pi}{b}-\frac{a\pi}{b \sqrt{ a^2-b^2 }} \right),\\
			&=\pi \left( \frac{2-\sqrt{2 \varphi}}{ 4-2\varphi } \right).
		\end{aligned}
	\end{equation}
	Since \(a^2-b^2 = 8\varphi\), we obtain
	\begin{equation}
		I(\varphi)
		= 2\left[
		\frac{\pi}{b}
		+ \Big(1 - \frac{a}{b}\Big)\frac{\pi}{\sqrt{a^2-b^2}}
		\right]
		= \pi\,\frac{2 - \sqrt{2\varphi}}{2 - \varphi}.
	\end{equation}
	Substituting this into~\eqref{eq:Uk-proof-start-normalized} yields
	\begin{equation}\label{aaaaa-appendix-normalized}
		U_k
		= \Delta T_k\,\frac{2 - \sqrt{2\varphi_k}}{2 - \varphi_k}.
	\end{equation}
	Substituting the optimal \(U_k^{\ast}\) in each segment into~\eqref{aaaaa-appendix-normalized} gives
	\begin{equation}
		\varphi_k
		= \frac{2\left(1- \psi \right)^2}{\psi^2},
	\end{equation}
	where
	\begin{equation}
		\psi= \frac{U_k^{\ast}}{\Delta T_k}.
	\end{equation}
\end{proof}

Therefore, substituting the derived \(\varphi_k\) from
Theorem~\ref{explicit-varphi-k-appendix} into the UBA-equivalent family
(Definition~\eqref{definition:UBA-family}), match the iteration-wise extended formula \(\varphi_t\)(under iteration-wise framework, $\Delta T_k=1$) with the scale by \eqref{a-computable-plug}, we obtain the optimal step size
at \(t\in[T_k,T_{k+1}]\) on segment \(k\) as follows:
\begin{equation}
	\begin{aligned}
		\eta_k(t)
		= \frac{2\big(1 + \cos\theta_k(t)\big)}
		{2\varphi_k + (2-\varphi_k)\big(1 + \cos\theta_k(t)\big)},
	\end{aligned}
\end{equation}
where
\[
\theta_k(t)
:= \frac{\pi(t - T_k)}{\Delta T_k} + (k-1)\pi,
\]
and the curvature parameter is obtained by substituting the plug-in estimate
$U_t^{\ast} = \frac{1}{1 + \sqrt{\varphi_0/2} 
	\left\|\frac{\hat{m}_t}{\sqrt{\hat{v}_t}+\epsilon}\right\|}$ into the matching
solution $\varphi_t = \frac{2(1-U_t^\ast)^2}{(U_t^\ast)^2}$ from
Theorem~\ref{explicit-varphi-k-appendix}, which simplifies to the closed form
\begin{equation}
	\varphi_t
	= \varphi_0 \left\|\frac{\hat{m}_t}{\sqrt{\hat{v}_t} + \epsilon}\right\|^2.
\end{equation}

\paragraph{Global LRD by segment-wise UTA.} \label{Global LRD by segment-wise concatenation}
The above derivation specifies an optimal \emph{segment-wise} shape
$\eta_k(t)$ under a locally approximation on $[T_k,T_{k+1}]$. 
We extend it to a learning-rate LRD on the entire horizon $[0,T]$ by
concatenation:
\begin{equation}\label{eq:global-LRD-piecewise}
	\eta(t) := \eta_k(t), \qquad t\in[T_k,T_{k+1}],\quad k=1,\dots,K.
\end{equation}
This yields a well-defined global LRD whose shape is
min-max optimal \emph{within each segment} and whose scale is calibrated by
matching the intrinsic time $U_k^\ast$.

\begin{proposition}
	\label{prop:robust-concat}
	Consider the  segment dynamics on $[T_k,T_{k+1})$, the end-of-segment state satisfies
	\begin{equation}\label{eq:seg-bound}
		\|W_{T_{k+1}} -\bar{W}^{(k)}\|
		\le
		\rho_k\|W_{T_{k}} -\bar{W}^{(k)}\|
		+
		\Big(\varepsilon_{\varsigma}^{(k)} + \eta_{\max}\varepsilon_M^{(k)}\|W_{T_{k}} -\bar{W}^{(k)}\|\Big)
		\Delta T_k \rho_k ,
	\end{equation}
	where  $\eta_{\max}:=\max_{t\in[T_k,T_{k+1})}\eta_t$,  $
	\rho_k := \| \prod_{t=T_k}^{T_{k+1}-1}(I-\eta_t \bar{H_k})\|$, $\bar{W}^{(k)}$ is the local minimum in $k$-th segment. 
	$\varsigma_t$ is an additive perturbation
	collecting by (i) local quadratic residuals and (ii) momentum or preconditioner
	approximation errors. Throughout the segment we assume, $
	\|M_t-\bar{H_k}\|\le \varepsilon_M^{(k)}$
	and
	$\|\varsigma_t\|\le \varepsilon_{\varsigma}^{(k)} $, where $M_t$ denotes the preconditioned Hessian estimate used at each step.
\end{proposition}

\begin{proof}
	In the ideal segment-wise steady model, the dynamic can be estimated by segment-wise  $\bar{H_k}$ for each segment, as   \eqref{eq4}, 
	\begin{equation}
		\begin{aligned}
			&W_{T_{k+1}} -\bar{W}^{(k)}
			=\left(I-\eta_{T_{k+1}-1} \bar{H_k}\right)\cdots \left(I-\eta_{T_{k}+1} \bar{H_k}\right)  \left( W_{T_{k}+1 } -\bar{W}^{(k)}\right).
		\end{aligned}
	\end{equation}
	In a real system, the true segment dynamics on $[T_k,T_{k+1})$ is
	\begin{equation}\label{eq:true-dyn-proof}
		W_{t+1} -\bar{W}^{(k)} = (I-\eta_t M_t)\,\left(W_{t} -\bar{W}^{(k)}\right) +  \varsigma_t,
		\qquad t\in[T_k,T_{k+1}),
	\end{equation}
	where $M_t\succeq 0$ and $M_t = \bar{H_k} + \Delta M_t$, 	
	and $\varsigma_t$ is an additive perturbation
	collecting by (i) local quadratic residuals and (ii) momentum or preconditioner
	approximation errors.

	Substitute $M_t = \bar{H_k} + \Delta M_t$ into \eqref{eq:true-dyn-proof}:
	\begin{equation}\label{eq:dyn-pertxxxx}
		W_{t+1} -\bar{W}^{(k)} = (I-\eta_t \bar{H_k})\,\left(W_{t} -\bar{W}^{(k)}\right) +    \hat{\varsigma}_t,
		\qquad t\in[T_k,T_{k+1}),
	\end{equation}
	where $\hat{\varsigma}_t:=\varsigma_t - \eta_t\Delta M_t \left(W_{t} -\bar{W}^{(k)}\right)$.
	
	Iterating \eqref{eq:dyn-pertxxxx} over the segment gives
	\begin{equation}\label{eq:unrollxxx}
		W_{T_{k+1}} -\bar{W}^{(k)}
		=
		A_k\,\left(W_{T_{k}} -\bar{W}^{(k)}\right)
		+
		\sum_{t=T_k}^{T_{k+1}-1}
		\left(
		\prod_{s=t+1}^{T_{k+1}-1}(I-\eta_s \bar{H_k})
		\right)\hat{\varsigma}_t,
	\end{equation}
	where $A_k$ is defined as
	\begin{equation}
		A_k := \prod_{t=T_k}^{T_{k+1}-1}(I-\eta_t \bar{H_k}),
		\qquad
		\rho_k := \|A_k\|.
	\end{equation}.

	Taking norms in \eqref{eq:unrollxxx} yields
	\begin{equation}\label{eq:norm-bound1}
		\|W_{T_{k+1}} -\bar{W}^{(k)}\|
		\le
		\|A_k\|\cdot\|W_{T_{k}} -\bar{W}^{(k)}\|
		+
		\sum_{t=T_k}^{T_{k+1}-1}
		\left\|
		\prod_{s=t+1}^{T_{k+1}-1}(I-\eta_s \bar{H_k})
		\right\|\cdot \|\hat{\varsigma}_t\|.
	\end{equation}
	By definition $\|A_k\|=\rho_k$, it gives,
	\begin{equation}\label{eq:norm-bound2}
		\|W_{T_{k+1}} -\bar{W}^{(k)}\|
		\le
		\rho_k\|W_{T_{k}} -\bar{W}^{(k)}\|
		+
		\Delta T_k \rho_k\|\hat{\varsigma}_t\|.
	\end{equation}

	By definition $\hat{\varsigma}_t:=\varsigma_t - \eta_t\Delta M_t \left(W_{t} -\bar{W}^{(k)}\right)$ and triangle inequality,
	\[
	\|\hat{\varsigma}_t\|
	\le
	\|\varsigma_t\| + \eta_t\|\Delta M_t\|\cdot\|W_{t} -\bar{W}^{(k)}\|.
	\]
	
	Assume  that for $t\in[T_k,T_{k+1}) $, $
	\|\Delta M_t\|\le \varepsilon_M^{(k)}$
	and
	$\|\varsigma_t\|\le \varepsilon_{\varsigma}^{(k)} $. We have
	\begin{equation}
		\|\hat{\varsigma}_t\|
		\le
		\varepsilon_{\varsigma}^{(k)} + \eta_{\max}\varepsilon_M^{(k)}R_k.
	\end{equation}
	where $R_k:=\max_{t\in[T_k,T_{k+1}]}\|W_{t} -\bar{W}^{(k)}\|$ is the segment radius.
	
	Substituting this into \eqref{eq:norm-bound2} gives
	\[
	\|W_{T_{k+1}} -\bar{W}^{(k)}\|
	\le
	\rho_k\|W_{T_{k}} -\bar{W}^{(k)}\|
	+
	\Big(\varepsilon_{\varsigma}^{(k)} + \eta_{\max}\varepsilon_M^{(k)}\|W_{T_{k}} -\bar{W}^{(k)}\|\Big)
	\Delta T_k \rho_k .
	\]
	
	This shows the deviation from the segment prediction is controlled
	by $(\varepsilon_{\varsigma}^{(k)},\varepsilon_M^{(k)})$ and the segment length $\Delta T_k$.
\end{proof}

Thus, concatenating segment-wise near-optimal LRDs (which minimize $\rho_k$)
yields a global LRD whose deviation from the  prediction
vanishes as $\Delta T_k\to 0$ and $(\varepsilon_{\varsigma}^{(k)},\varepsilon_M^{(k)})\to 0$.

\begin{corollary}
	\label{cor:global-concat-bound}
	Under the conditions of Proposition~\ref{prop:robust-concat}, the concatenated
	LRD $\eta(t)=\eta_k(t)$ for $t\in[T_k,T_{k+1})$ yields the global bound
	\begin{equation}\label{eq:global-bound-concat}
		\|u_T\|
		\le
		\Big(\prod_{k=1}^{K}\rho_k\Big)\|u_0\|
		+
		\sum_{i=1}^{K}\Big(\Delta_i\prod_{j=i+1}^{K}\rho_j\Big),
	\end{equation}
	where $\rho_k:=\big\|\prod_{t=T_k}^{T_{k+1}-1}(I-\eta_t \bar{H_k})\big\|$ and $\Delta_k:=\Big(\varepsilon_{\varsigma}^{(k)} + \eta_{\max}\varepsilon_M^{(k)}\|W_{T_{k}} -\bar{W}^{(k)}\|\Big)
	\Delta T_k \rho_k$ is
	the segment-wise accumulated perturbation term in
	Proposition~\ref{prop:robust-concat}.
	In particular, when $\Delta_k$ is small (short segments), minimizing $\rho_k$
	segment-wise yields a near-optimal global contraction.
\end{corollary}

\paragraph{Justification of concatenation.}
Proposition~\ref{prop:robust-concat} and Corollary~\ref{cor:global-concat-bound} further show that under mild non-stationarity,
the concatenated LRD remains near-optimal, with deviation controlled by
the segment length, curvature and preconditioner drift.
The segment length $\Delta T_k$ controls the approximation error. Shorter segments yield smaller trajectory radius $R_k$ and smaller curvature
drift. 
Hence, the piecewise construction becomes more accurate as
$\max_k\Delta T_k$ decreases.
Consequently, to achieve the best possible approximation within our framework, we set the segment length to its limit value, $\Delta T_k = 1$. 
This implies that each segment corresponds to a single optimization step, thereby equating the segment index $k$ with the time index $t$ (so that, for instance, $\varphi_k$ becomes $\varphi_t$). The resulting final UTA LRD is therefore defined per time step as follows.

\paragraph{Derived UTA LRD.}
Although the parameters $(\eta_{\min}^{(k)},\eta_{\max}^{(k)},\varphi_k)$ may vary
across segments, the LRD on every segment belongs to the same
UBA-equivalent function family. Therefore, the resulting global LRD
\eqref{eq:global-LRD-piecewise} is a \emph{piecewise UBA} LRD on $[0,T]$.
To avoid discontinuities at segment boundaries, we set the starting endpoint of the next segment to match the ending value of the previous segment, 
$\eta_{\min}^{(k)}:=\eta_{\max}^{(k+1)}$ for decay LRDs ( analogously for increasing LRDs).
This keeps $\eta(t)$ continuous on $[0,T]$ while preserving the segment-wise
optimal shape.

The derived LRD on the entire horizon $[0,T]$ admits the closed form
\begin{equation}\label{eta-closed-form-appendix}
	\begin{aligned}
		\eta(t)=g_{a,b}(\varphi,\theta(t))
		= \eta_{\min}
		+ \big(\eta_{\max} - \eta_{\min}\big)\,
		\frac{2\big(1 + \cos\theta(t)\big)}
		{2\varphi_t + (2-\varphi_t)\big(1 + \cos\theta(t)\big)},
	\end{aligned}
\end{equation}
where $a=\eta_{\min}$, $b=\eta_{\max} - \eta_{\min}$, 
$
\theta(t)
:= \frac{\pi(t )}{T},
$
the curvature parameter	$\varphi_t$ is determined by the exact intrinsic‑time matching condition
\begin{equation}\label{varphi-t-simplify-appendix}
	\varphi_t
	= \varphi_0 \left\|\frac{\hat{m}_t}{\sqrt{\hat{v}_t} + \epsilon}\right\|^2 .
\end{equation}

\section{Theoretical analysis}\label{Theoretical Analysis}

\subsection{Functional scaling laws for LRDs}

To theoretically validate the effectiveness of the UTA, we analyze it under the framework of the Functional Scaling Law(FSL) \citep{li2025functional}.
This FSL theory characterizes how the entire loss trajectory scales with training resources under given LRD. Formally, for a model with capacity exponent $\beta$ and task difficulty $s$, the expected loss $\mathcal{E}_t$ at intrinsic training time $t$ follows:
\begin{equation}
	\begin{aligned}
		\mathcal{E}_t \eqsim \underbrace{M^{-s\beta}}_{\substack{\text{model} \\ \text{error}}}+\underbrace{e(t)}_{\substack{\text{model} \\ \text{learning}}} 
		+ \underbrace{\int_0^t \mathcal{K}(t-z)[e(z)+\sigma^2]\gamma(z) \text{d}z}_{\text{noise accumulation}},
	\end{aligned}
\end{equation}
where $M$ denotes model size, $\sigma^2$ is label noise, $\mathcal{K}(\cdot)$ is a forgetting kernel that decays as $\mathcal{K}(t):=\left(1+t\right)^{-\left(2-\frac{1}{\beta}\right)}$, $e(t):=(1+t)^{-s}$ denotes the learning term, and $\gamma(z)$ encapsulates the joint effect of learning rate and batch size LRD.
In addition, the notation $\eqsim$ denotes equivalence up to a constant factor.

The FSL reveals that different learning rate decays fundamentally affect training efficiency through two competing mechanisms: (i) accumulating sufficient \textit{intrinsic time}  $t = \int_0^\tau \eta(r)\text{d}r$ for signal learning, and (ii) controlling noise injection via the LRD-dependent term $\gamma(z)$. This framework has enabled rigorous comparison of LRD strategies, establishing that WSD achieves strictly better scaling efficiency than both constant and exponential decay LRDs \citep{li2025functional}.


\subsection{Analytical framework for our LRD}

Given the widespread adoption of cosine LRD and the fact that the cosine function forms a core component of our proposed LRD approach, we first analyze the theoretical effect of the cosine LRD.

Within the functional scaling law framework, the model error term is determined solely by model capacity. Therefore, the differences between LRDs manifest in the model learning term and noise accumulation term. Analysis reveals that: (i) Compared to cosine LRD, UTA 
achieves smaller values in both terms, provides superior control of expected loss $\mathcal{E}_t$ at intrinsic training time. (ii) Compared to WSD FSL \citep{li2025functional}, UTA exhibits equivalent scaling behavior since its model learning and noise accumulation terms are of the same order.

Our analysis proceeds through the following sequence of steps. First, we formulate the learning rate decay in physical time $\phi(\tau)$. Next, we compute the transformation to the  intrinsic time $\bar{T}(\tau)=\int_{0}^{\tau} \phi(z) \text{d}z $ and examine its properties, thereby circumventing the need for an explicit inverse function. We then express the noise intensity $\gamma(z)$ within the intrinsic time coordinate. Finally, we substitute these results into the FSL formulation to derive both the signal learning term $e(t)$ and the noise accumulation term $\int_0^t \mathcal{K}(t-z)[e(z)+\sigma^2]\gamma(z)\text{d}z$.

Since cosine function is a part of our LRD function, we first analyze the theoretical efficiency of cosine learning rate decay.


\begin{theorem}[Cosine LRD under FSL]
\label{prop:FSL-cosine}
Under the FSL assumptions and in the hard-learning regime
$s\le 1-1/\beta$, consider the cosine-annealed learning-rate
LRD
\[
\phi_{\cos}(\tau)
= \eta_{\min}
+\frac12 (\eta_{\max}-\eta_{\min})\bigl(1+\cos(\pi\tau/T_{\max})\bigr),
\qquad \tau\in[0,T_{\max}],
\]
where $T_{\max}$ is the total training step number.
Let $\bar{T}(\tau)=\int_0^\tau \phi_{\cos}(r)\,\text{d}r$ be the intrinsic time 
and $\bar{T}(T_{\max})$ the final intrinsic time.
Then the FSL  risk satisfies the form
\begin{equation}\label{eq:FSL-cosine-struct-appendix}
	\mathcal{E}_t
	\eqsim
	M^{-s\beta}
	+ (1+aT_{\max})^{-s}
	+ \frac{\sigma^2}{B}\left[
	\left(a-bT_{\max}\right)
	+ b\,
	\frac{\min\{M,\left(aT_{\max}\right)^{1/\beta}\}}{aT_{\max}}
	\right],
\end{equation}
where $B$ denotes the batch size, $a= \frac{1}{2} \left( \eta_{\max} + \eta_{\min} \right) $ , $b=\frac{1}{2} \left( \eta_{\max} - \eta_{\min} \right)$.
\end{theorem}

\begin{proof}
By direct integration,
\[
\bar{T}(\tau)=\int_0^\tau\phi_{\cos}(r)\,\text{d}r
= a\tau + b\frac{T_{\max}}{\pi}\sin(\pi\tau/T_{\max}),
\]
with $a= \frac{1}{2}(\eta_{\max} + \eta_{\min})$, $b=\frac{1}{2}(\eta_{\max} - \eta_{\min})$. 
Hence the final intrinsic time is $\bar{T}(T_{\max})=aT_{\max}$.

Since $\sin(\cdot)$ is bounded, from the pointwise bounds
$\eta_{\min}\leq\phi_{\cos}(\tau)\leq\eta_{\max}$ we obtain
\begin{equation}\label{eq:intrinsic-comparability}
	\eta_{\min}\tau
	\leq\bar T(\tau)
	\leq\eta_{\max}\tau,
	\qquad
	0\leq\tau\leq T,
\end{equation}
which guarantees monotonicity and gives the corresponding inverse bounds
\begin{equation}
	\frac{t}{\eta_{\max}}
	\leq\bar T^{-1}(t)
	\leq\frac{t}{\eta_{\min}}.
\end{equation}

Then, under the Theorem 4.4 and Assumption 5.1 in FSL rules \citep{li2025functional}, the FSL can be simplified to
\[
\mathcal{E}_t
\eqsim M^{-s\beta}+e(t)+\frac{\sigma^2}{B}\mathcal N\left(\phi\right),
\]
where $e(t):=(1+t)^{-s}$ ,
$\mathcal N[\gamma]:=\int_0^t \mathcal{K}_M (t-z) \phi (\bar{T}^{-1}(t)) \,\text{d}z$.
It gives the signal term as
\[
e(t)=(1+aT_{\max})^{-s}.
\]
Moreover, the forgetting kernel can be approximated as
\[
\mathcal{K}_M (t)
\eqsim \sum_{j=1}^{M} j^{-2\beta} e^{-2j^{-\beta}t}
\eqsim \int_{1}^{M} x^{-2\beta} e^{-2x^{-\beta}t} \text{d}x
\eqsim \int_{M^{-\beta}}^{1} u^{1-1/\beta} e^{-2ut} \text{d}u.
\]
In addition, the noise intensity $\gamma$ within the intrinsic time coordinate is
$\gamma(t)=\frac{\phi_{\cos}(\bar{T}^{-1}(t))}{B}$ , and $ a+b(1-\frac{\pi t}{T_{\max}})\leq\phi (\bar{T}^{-1}(t)) \leq a+b(\pi-\frac{\pi t}{T_{\max}})$.

It gives $N\left(\phi\right)$ as
\begin{equation*}
	\begin{aligned}
		\int_0^T \mathcal{K}_M (T-t) \phi_{\cos} (t) \,dt \eqsim  &\int_0^T \int_{M^{-\beta}}^{1} u^{1-1/\beta} e^{-2u\left(T-t\right)} \text{d}u  \left( \hat{a}-\hat{b} t \right)   \,\text{d}t  \\
		= &\int_{M^{-\beta}}^{1} u^{1-1/\beta} e^{-2uT} \left( \int_{0}^{T} e^{2ut}(\hat{a}-\hat{b} t) \, \text{d}t \right) \text{d}u\\
		= &\int_{M^{-\beta}}^{1} u^{1-1/\beta} e^{-2uT} \left[ \frac{\hat{a}}{2u} (e^{2uT} - 1) - \frac{\hat{b}}{2u} \left( Te^{2uT} - \frac{e^{2uT} - 1}{2u} \right) \right] \text{d}u\\
		= &\int_{M^{-\beta}}^{1} \left[ \frac{\hat{a}}{2u^{1/\beta}} - \frac{\hat{a}e^{-2uT}}{2u^{1/\beta}} - \frac{\hat{b} T}{2u^{1/\beta}} + \frac{\hat{b} (1 - e^{-2uT})}{4u^{1+1/\beta}} \right] du\\
		= &\int_{M^{-\beta}}^{1} \left[ \frac{\hat{a} - \hat{b} T}{2u^{1/\beta}} - \frac{(\hat{a} - \hat{b} T + \hat{b} T)e^{-2uT}}{2u^{1/\beta}} + \frac{\hat{b} (1 - e^{-2uT})}{4u^{1+1/\beta}} \right] \text{d}u\\
		= & (\hat{a} - \hat{b} T) \int_{M^{-\beta}}^{1} \frac{1 - e^{-2uT}}{2u^{1/\beta}} \text{d}u + \hat{b} T \int_{M^{-\beta}}^{1} \left( -\frac{e^{-2uT}}{2u^{1/\beta}} + \frac{1 - e^{-2uT}}{4Tu^{1+1/\beta}} \right) \text{d}u,
	\end{aligned}
\end{equation*}
where $\hat{a}:= \pi\eta_{\max} \eqsim (a+b)\pi $ , $\hat{b} = \frac{1}{2} \left( \eta_{\max} - \eta_{\min} \right)=b\frac{\pi}{T_{\max}}$ and total intrinsic time $\bar{T}(T_{\max})=a T_{\max}$.

Then, according to Lemma F.5 and Lemma F.6 in FSL \citep{li2025functional}, the noise term can be estimated as
\begin{equation*}
	\begin{aligned}
		N\left(\phi\right) \eqsim &\left[\eta_{\max}-\frac{1}{2} \left( \eta_{\max}^2 - \eta_{\min}^2 \right) \right]\pi + \frac{1}{2} \left( \eta_{\max} - \eta_{\min} \right) \pi \frac{\min\{M,(aT_{\max})^{1/\beta} \}}{aT_{\max}},
	\end{aligned}
\end{equation*}
which gives
\begin{equation*}
	\begin{aligned}
		\mathcal{E}_t
		\eqsim
		M^{-s\beta}
		+ (1+aT_{\max})^{-s}
		+ \frac{\sigma^2}{B}\left[
		\left(a+b-\frac{ab}{2}\right)\pi
		+ b \pi\,
		\frac{\min\{M,\left(aT_{\max}\right)^{1/\beta}\}}{aT_{\max}}
		\right].
	\end{aligned}
\end{equation*}

Combining these contributions yields \eqref{eq:FSL-cosine-struct-appendix}.
\end{proof}

\begin{theorem}
\label{theorem4-2}
Under the FSL assumptions and in the hard-learning regime
$s\le 1-1/\beta$.
Let $\bar{T}(\tau)=\int_0^\tau \phi(r)\,\text{d}r$ be the intrinsic time 
and $\bar{T}(T_{\max})$ the final intrinsic time. When a sufficiently small $\varphi_k$ is sufficiently small that causes the LRD to maintain a consistently high learning rate over a prolonged duration, analogous to WSD-like period, the FSL  risk satisfies the form
\begin{equation*}
	\begin{aligned}
		\mathcal{E}_t
		\eqsim
		M^{-s\beta}
		+   \left[\frac{2 - \sqrt{2\varphi}}{2 - \varphi}   \left(T_1+T_{2}\right)\right] ^{-s}
		+ \frac{\sigma^2}{B}\left[
		\left(a+b-\frac{ab}{2}\right)\pi
		+ b \pi\,
		\frac{\min\{M,\left(c \frac{2 - \sqrt{2\varphi}}{2 - \varphi} T_{2}\right)^{1/\beta}\}}{aT_{2}}
		\right].
	\end{aligned}
\end{equation*}
where $B$ denotes the batch size.
\end{theorem}

\begin{proof}

According to   \eqref{eq:Uk-proof-start-normalized}, it yields the integration of our proposed LRD in $k$-th segment,
\begin{equation*}
	U_k
	= \eta_{\min}^{(k)}\Delta T_k
	+ \big(\eta_{\max}^{(k)} - \eta_{\min}^{(k)}\big)\,
	\Delta T_k\,\frac{2 - \sqrt{2\varphi_k}}{2 - \varphi_k}.
\end{equation*}

As $\frac{2 - \sqrt{2\varphi_k}}{2 - \varphi_k}$ decreases from 1 to 0, a sufficiently small $\varphi_k$ causes the LRD to maintain a consistently high learning rate over a prolonged duration, analogous to the stable period observed in WSD \citep{hu2024minicpm}. Assume that when $\varphi_k\leq\mathrm{Z}$, 
the learning rate enters a stable phase, i.e.,$\mathbb{E} (\eta_t)= \frac{2 - \sqrt{2\varphi}}{2 - \varphi}\eta_{\max}$, where $r\leq 1$ is a constant close to 1; 
when $\varphi_k>\mathrm{Z}$, 
the learning rate decay enters a phase of significant decay. Let $T_1:=|T_{\varphi_k\leq \mathrm{Z}}|$ denote the duration of stable stage, $T_2:=T_{\max}-T_{\varphi_k>\mathrm{Z}}$ 
denote the duration of this pronounced decay stage.

For $T_1$ stage, since $\mathbb{E} (\eta_t)=\frac{2 - \sqrt{2\varphi}}{2 - \varphi}\eta_{\max}$, according to FSL of constant LRD,
\begin{equation*}
	\begin{aligned}
		\mathcal{E}_t
		\eqsim
		M^{-s\beta}
		+ (\frac{2 - \sqrt{2\varphi}}{2 - \varphi}\eta_{\max} T_{1})^{-s}
		+ \frac{\frac{2 - \sqrt{2\varphi}}{2 - \varphi}\eta_{\max}\sigma^2}{B}.
	\end{aligned}
\end{equation*}
For $T_2$ stage, we  analyze the case where $\varphi_k$ is invariant with $k$. Denote $\varphi_k \equiv \varphi$ and $K_{\tau}=\{K:\tau\in (T_K,T_{K+1}]\}$. If $\eta_{\max}^{(k)}=\eta_{\min}^{(k)}+\epsilon$ where $\epsilon$ is a very small constant, and $\Delta T_k=1$, it gives
\begin{equation*}
	\begin{aligned}
		\bar{T}(\tau)
		= &\sum_{k=1}^{K_{\tau}} 		\eta_{\min}^{(k)}\Delta T_k
		+ \big(\eta_{\max}^{(k)} - \eta_{\min}^{(k)}\big)\,
		\Delta T_k\,\frac{2 - \sqrt{2\varphi}}{2 - \varphi}\\
		= &\frac{\tau}{2} \left[ \left( \eta_{\min}^{(1)} + \eta_{\min}^{(K_{\tau})}\right)
		+ \left( \eta_{\max}^{(1)} + \eta_{\max}^{(K_{\tau})}-\eta_{\min}^{(1)} + \eta_{\min}^{(K_{\tau})} \right)\,
		\,\frac{2 - \sqrt{2\varphi}}{2 - \varphi} \right]\\
		\eqsim & c \frac{2 - \sqrt{2\varphi}}{2 - \varphi} \tau, 
	\end{aligned}
\end{equation*}
where $c$ is a constant and $c\leq \eta_{\max}$.
Thus we have 
\begin{equation*}
	\begin{aligned}
		\tau =\bar{T}^{-1}(t)
		= & \frac{2 - \varphi}{c\left(2 - \sqrt{2\varphi}\right)} t.
	\end{aligned}
\end{equation*}
Furthermore, the total intrinsic time is
\begin{equation*}
	\begin{aligned}
		\bar{T}(T_{\max})
		\eqsim & c \frac{2 - \sqrt{2\varphi}}{2 - \varphi} T_{\max}.
	\end{aligned}
\end{equation*}
In addition, since 
\begin{equation}
	\begin{aligned}
		&\phi
		= \eta_{\min}^{(k)}
		+ \big(\eta_{\max}^{(k)} - \eta_{\min}^{(k)}\big)\,
		\frac{2\big(1 + \cos\theta_k(t)\big)}
		{2\varphi_k + (2-\varphi_k)\big(1 + \cos\theta_k(t)\big)},
	\end{aligned}
\end{equation}
where 
$
\theta_k(t)
:= \frac{\pi(t - T_k)}{\Delta T_k} + (k-1)\pi,
$
we have the LRD-adjusted function in intrinsic time is given by
\begin{equation*}
	\begin{aligned}
		\phi\left( \bar{T}^{-1}(t)\right) = \eta_{\min}^{(k)}
		+ \big(\eta_{\max}^{(k)} - \eta_{\min}^{(k)}\big)\,
		\frac{2\big(1 + \cos\hat{ \theta}_k(t)\big)}
		{2\varphi_k + (2-\varphi_k)\big(1 + \cos\hat{ \theta}_k(t)\big)} .
	\end{aligned}
\end{equation*}
where 
$
\hat{ \theta}_k(t)
:= \frac{\pi\left(\frac{2 - \varphi}{c\left(2 - \sqrt{2\varphi}\right)} \left(t - T_k\right)\right)}{\Delta T_k} + (k-1)\pi.
$

Similar to the proof in Theorem \ref{prop:FSL-cosine}, the noise intensity $\gamma$ within the intrinsic time coordinate is
$\gamma(t)=\frac{\phi_{\cos}(\bar{T}^{-1}(t))}{B}$ , and $ a+b(1-\frac{2 - \varphi}{c\left(2 - \sqrt{2\varphi}\right)}  \frac{t\pi}{\Delta T_k})\leq\phi (\bar{T}^{-1}(t)) \leq a+b(\pi-\frac{2 - \varphi}{c\left(2 - \sqrt{2\varphi}\right)}\frac{t\pi}{\Delta T_k})$.
The noise term can be estimated as
\begin{equation*}
	\begin{aligned}
		N\left(\phi\right) \eqsim &\left[\eta_{\max}-\frac{1}{2} \left( \eta_{\max}^2 - \eta_{\min}^2 \right) \right]\pi + \frac{1}{2} \left( \eta_{\max} - \eta_{\min} \right) \pi \frac{\min\{M,(c \frac{2 - \sqrt{2\varphi}}{2 - \varphi} T_{2})^{1/\beta} \}}{aT_{2}},
	\end{aligned}
\end{equation*}
which gives
\begin{equation*}
	\begin{aligned}
		\mathcal{E}_t
		\eqsim
		M^{-s\beta}
		+ (1+c \frac{2 - \sqrt{2\varphi}}{2 - \varphi} T_{2})^{-s}
		+ \frac{\sigma^2}{B}\left[
		\left(a+b-\frac{ab}{2}\right)\pi
		+ b \pi\,
		\frac{\min\{M,\left(c \frac{2 - \sqrt{2\varphi}}{2 - \varphi} T_{2}\right)^{1/\beta}\}}{aT_{\max}}
		\right].
	\end{aligned}
\end{equation*}
Following the Theorem 5.4 in FSL \citep{li2025functional}, the FSL risk of entire training dynamic under our proposed LRD is as follow
\begin{equation*}
	\begin{aligned}
		\mathcal{E}_t
		\eqsim
		M^{-s\beta}
		+   \left[\frac{2 - \sqrt{2\varphi}}{2 - \varphi}   \left(T_1+T_{2}\right)\right] ^{-s}
		+ \frac{\sigma^2}{B}\left[
		\left(a+b-\frac{ab}{2}\right)\pi
		+ b \pi\,
		\frac{\min\{M,\left(c \frac{2 - \sqrt{2\varphi}}{2 - \varphi} T_{2}\right)^{1/\beta}\}}{aT_{2}}
		\right].
	\end{aligned}
\end{equation*}

\end{proof}

\section{Experimental implementation details}\label{appendix-implement-details}

\subsection{Baselines}\label{baselines-appendix}

We detail the baselines as follows, 
\begin{itemize}
\item \textbf{Cosine LRD}, originally proposed for SGD with restarts \citep{loshchilov2017sgdr}, has become a de facto standard in modern deep learning and large scale model training due to its smooth transition and robust performance \citep{loshchilov2017sgdr,schaipp2025surprising}. The LRD follows 
\begin{equation*}
	\eta_t = \eta_{min} + \frac{1}{2}(\eta_0 - \eta_{min})(1 + \cos(\frac{t\pi}{T})),
\end{equation*} 
where $T$ is total iterations, $\eta_{0}$ and $\eta_{\min}$ are the initial and final learning rates.
Its widespread adoption in modern deep learning and large language models establishes it as an essential classical baseline.

\item \textbf{Linear decay LRD} represents the simplest form of LRDd annealing, providing a transparent and easily interpretable comparison point. Following the implementation in \citep{li2019budgeted}, we define
\begin{equation*}
	\eta_t = \eta_0 \times (1 - t/T),
\end{equation*} 
where $T$ is total iterations,  
after an optional warmup phase. Despite its simplicity, linear decay remains competitive in resource-constrained scenarios \citep{li2019budgeted} and serves as an important baseline for assessing whether more complex LRDs provide meaningful advantages over straightforward linear reduction.

\item \textbf{WSD LRD} exemplifies modern empirical practices in large-scale model training, particularly for language models \citep{zhai2022scaling,hu2024minicpm}. The LRD comprises three distinct phases, (1)Warmup phase: linear ramp from $\eta_{\text{start}}$ to $\eta_{\text{peak}}$ over $T_{\text{warmup}}$;  (2) Stable phase: maintain $\eta_{\text{peak}}$ for $T_{\text{stable}}$ steps.
(3) Decay phase: Linear or exponential decay to $\eta_{\min}$ (we use the WSD LRDr with 10\% steps in the decay
phase following previous works \citep{hu2024minicpm,hagele2024scaling}).
Its empirical success in large language models makes WSD a critical baseline for evaluating LRDs in realistic large-scale scenarios.

\item \textbf{UBA LRD} is a theoretically principled and empirically strong latest baseline. It derives from a min-max problem on optimization dynamics. The LRD takes the form:
\begin{equation*}
	\eta_{t}= \left(\eta_{\max}-\eta_{\min}\right)\frac{2\left[ 1+\cos  \left(  \frac{ \left[ 2\left(t-T_{k}\right) -1\right] \pi}{2\left(T_{k+1}-T_{k}\right)   }    \right) \right]}{ 2\varphi+ (2-\varphi)  \left[ 1+\cos  \left(  \frac{ \left[ 2\left(t-T_{k}\right) -1\right] \pi}{2\left(T_{k+1}-T_{k}\right)   }    \right) \right]  } + \eta_{\min},
\end{equation*}
where $\varphi$ is a hyper-parameter. We use the default setting \citep{tang2026stepsize}, fixing the $\varphi=0.5$. 
Recent systematic evaluations \citep{tang2026stepsize} rank UBA among the top-performing LRDs across diverse architectures and tasks, making it a particularly strong benchmark for theoretical LRDs.

\end{itemize}

\subsection{LRD setting}\label{LRD-Setting-appendix}

For our UTA LRD, we fix the unique hyper-parameter $\varphi_0=0.15$. 
For UBA LRD, we use the default setting \citep{tang2026stepsize}, fixing the $\varphi=0.5$. 
For WSD LRD, we use 10\% steps in the decay phase following the setting in previous works \citep{hu2024minicpm,hagele2024scaling}.

\subsection{Model architectures and hyper-parameters}\label{Model Architectures and Hyper-parameters}

\paragraph{Settings on LLaMA}

Following \citet{lialin2023stack} and \citet{zhao2024galore}, we pretrain vanilla LLaMA models from scratch on the C4 dataset \citep{raffel2020exploring}. We consider four model scales with 60M, 130M, 350M, and 1B parameters. For each scale, we evaluate four training budgets of 10K, 15K, 20K, and 25K steps. The corresponding architectural hyperparameters and token budgets are summarized in Tab.~\ref{tab:llama-arch}. Across all experiments, we set the batch size to 512 and the maximum sequence length to 1024, and train the models using AdamW with $\beta_1/\beta_2=0.9/0.95$. We use a peak learning rate of 1e-3 for the 60M–350M models, and 5e-4 for the 1B model. For a fair comparison across LRDs, we apply a learning rate warmup over the first 10\% of training steps for all experiments, and enforce a minimum learning rate equal to 10\% of the peak learning rate for all LRDrs. Since UTA has the phase drift mechanism that keeps ascending phase at early stage,  we also preserve its original warmup-like behavior.

\begin{table}[!ht]
\centering
\caption{LLaMA architecture hyper-parameters for evaluation. Data amounts are specified in tokens.}
\label{tab:llama-arch}
\small
\setlength{\tabcolsep}{6pt}
\begin{tabular}{ccccccc}
	\toprule
	Params & Hidden Dim & Intermediate Dim& Heads & Layers & Training Steps & Data amount (Tokens) \\
	\midrule
	60M    & 512 & 1376 & 8 & 8& 10k  &  5.2B    \\
	130M  & 768 & 2048 & 12 & 12 & 15k  &  7.9B    \\
	350M   & 1024 & 2736&  16 & 24 & 20k  & 10.5B     \\
	1 B    & 2048 & 5461 & 24 & 32 & 25k  & 13.1B     \\
	\bottomrule
\end{tabular}
\end{table}

\paragraph{Settings on ResNet.}
For ResNet, we adopt the A2 training recipe introduced by the work \citet{wightman2021resnet}, which employs stronger data augmentation and has been shown to substantially improve ResNet performance. The augmentation pipeline includes random cropping, random horizontal flipping, Mixup ($\alpha=0.1$), CutMix ($\alpha=1.0$) applied with 50\% probability, and RandAugment ($M=7$, $N=2$, $\text{MSTD}=0.5$). For all experiments, We train the model 300 epochs with 2048 batch size. For training loss, we use a binary cross-entropy loss. We use AdamW with a weight decay of 0.05 and $\beta_1/\beta_2 = 0.9/0.999$, and set the peak learning rate to 1e-3. Following Win \citep{zhou2024win}, we incorporate a 5-epoch warmup phase for all baseline LRDs, starting from an initial learning rate of 1e-8. 
For LRDs that natively include an ascending phase (e.g., the WSD), we preserve their original warmup-like behavior during that period. Since UTA has the phase drift mechanism that keeps ascending phase at early stage,  we also preserve its original warmup-like behavior. 
Additionally, we apply a 10-epoch cooldown phase at the end of training, with a minimum learning rate of 1e-5. 

\paragraph{Settings on ViT.}
For ViT, we follow the widely-used official training setting of ViTs described in the work \citep{touvron2021training,yu2022metaformer}. For this setting, data augmentation  consists of  random crop, horizontal flipping with probability, Mixup with parameter 0.8 \citep{zhang2017mixup}, CutMix with parameter
1.0 and probability 0.5 \citep{yun2019cutmix}, RandAugment \citep{cubuk2020randaugment} with M = 9, N = 2
and MSTD = 0.5, and Random Erasing with parameter p = 0.25. For all experiments, We train the models under 300 epochs with 1024 batch size. For training loss, we use cross entropy loss. We use the AdamW optimizer with a weight decay of 0.1 and $\beta_1/\beta_2 = 0.9/0.999$. The peak learning rate is set to 1e-3. Additionally, We adopt the same warmup and cooldown settings for baselines as in the ResNet experiments.

\subsection{Hardware and software environment}

All experiments are conducted on hardware consisting of  NVIDIA H800.

\subsection{Datasets}\label{dataset}

\paragraph{C4.} 
C4 dataset is a colossal, cleaned version of Common Crawl’s web crawl corpus, which is mainly
intended to pre-train language models and word representations \citep{raffel2020exploring}. To best simulate the practical
pre-training scenario, we train without data repetition over
a sufficiently large amount of data.

\paragraph{ImageNet.} 
To assess the scalability of our LRD on more demanding and realistic visual recognition tasks, we perform extensive experiments on the ImageNet dataset. 
Widely recognized as the de-facto benchmark for large-scale image classification, ImageNet comprises approximately 1.28 million training images and 50,000 validation images spanning 1,000 diverse classes, and is accompanied by an official, standardized data split. 

\paragraph{CIFAR100.} CIFAR-100 is a widely recognized benchmark for evaluating the performance and efficiency of deep learning models. The dataset consists of 60,000 color images, each with a resolution of 32$\times$32 pixels, categorized into 100 fine-grained classes. These classes are further organized into 20 super-classes, with each class containing 600 images. Every image is annotated with both a fine label, indicating its specific class, and a coarse label, indicating its super-class).

\subsection{Code and reproducibility}

Our implementation strictly follows the PyTorch official API, i.e. it is inherited from the base class
\begin{verbatim}
torch.optim.lr_LRDr.LRDcheduler.
\end{verbatim} It maintains full compatibility with all \begin{verbatim}
torch.optim
\end{verbatim}
optimizers. The complete production-ready code is publicly available.

The code of UBA LRD is available at \url{https://anonymous.4open.science/r/UTA-LRD-7B39/UTA.py}.

We conduct an extensive evaluation comprising \textbf{40-60 experimental runs}, \textbf{requiring approximately 5000 H800 GPU-hours}. It is a substantial commitment for a single-study paper (context: DeepSeek-V3 consumed ~3M H800 GPU-hours). This computational scale ensures statistically reliable conclusions and validates the robustness and scalability of our findings.  The systematic, large-scale validation on a modern Transformer architecture provides strong evidence for UTA’s practical utility and impact in real-world training scenarios.

\section{Experimental results}

\subsection{Performance on language model-LLaMA}

We evaluate on four scales of the LLaMA architecture: 60M, 130M, 350M, and 1B parameters. 
Performance is measured by perplexity on a held-out validation set, which directly reflects the model’s generative uncertainty and language modeling quality.


\begin{figure}[htbp]
\centering
\begin{subfigure}[b]{0.48\textwidth}
	\centering
	\includegraphics[width=\textwidth]{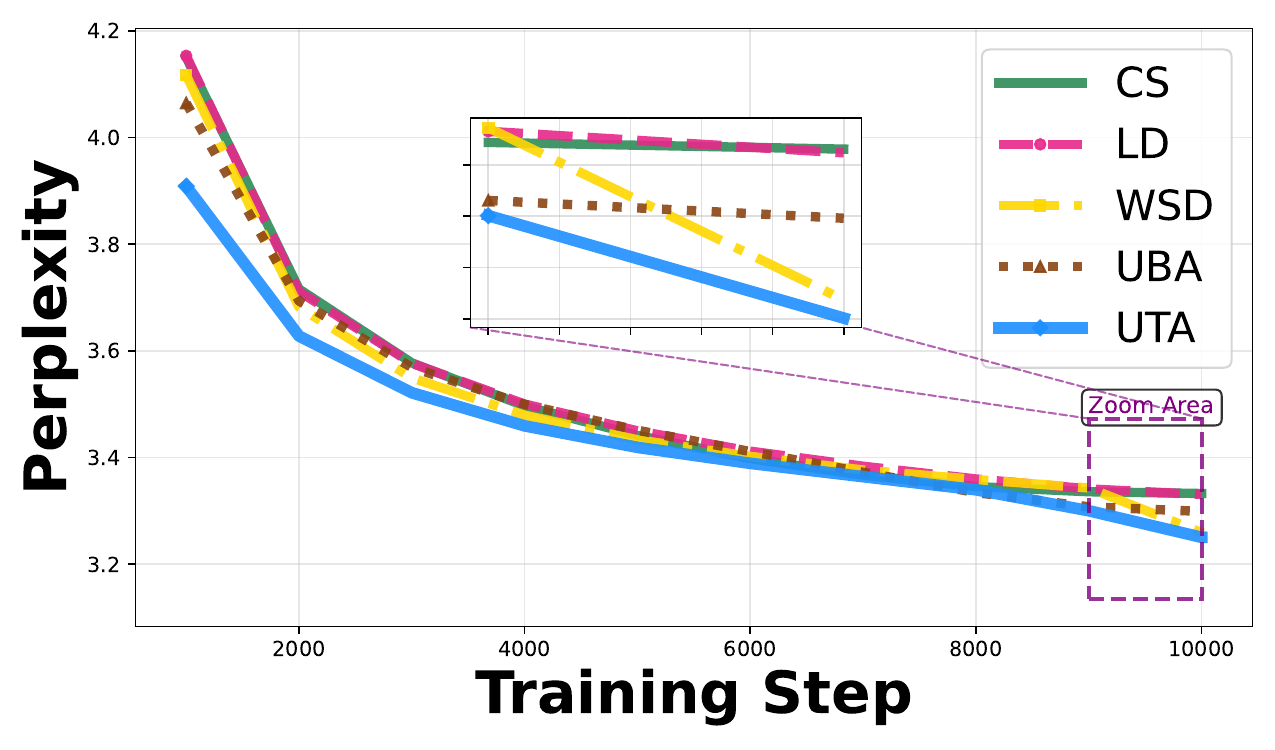}
	\caption{Perplexity of LLaMA-60M.}
	\label{Performance on Language Model-LLaMA60M-appendix}
\end{subfigure}
\begin{subfigure}[b]{0.48\textwidth}
	\centering
	\includegraphics[width=\textwidth]{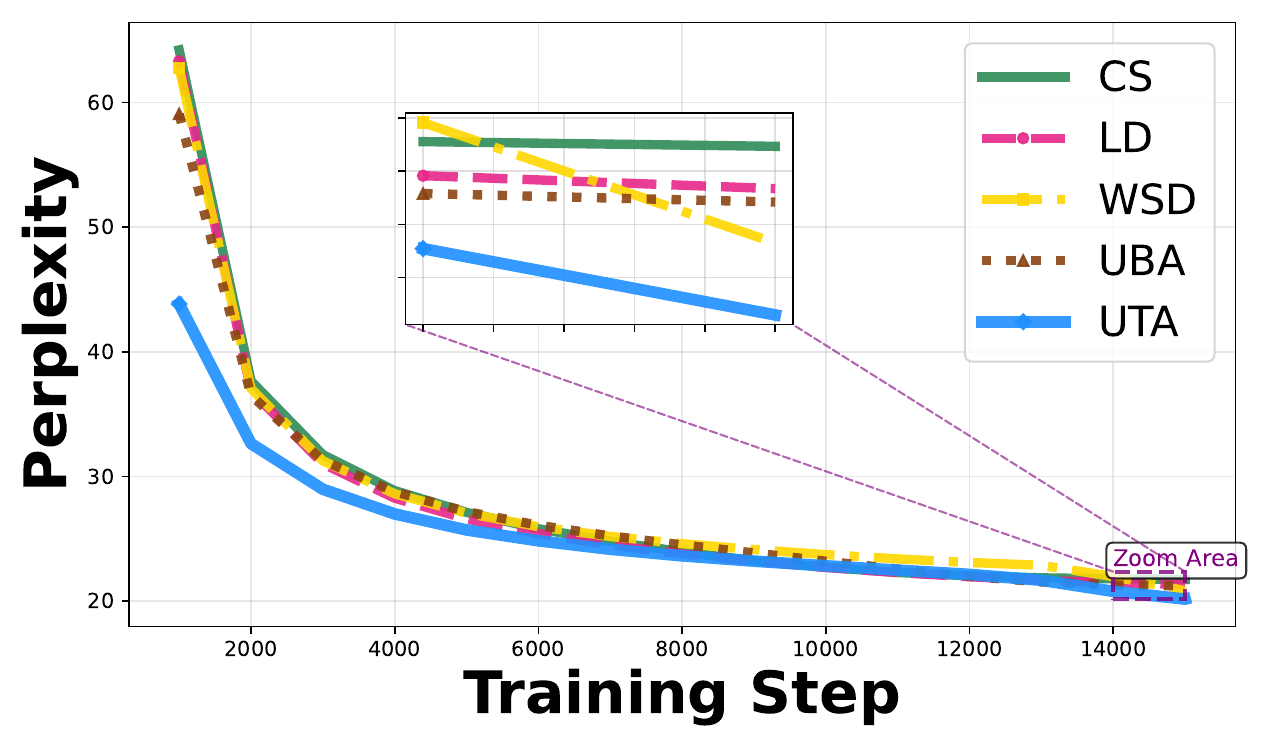}
	\caption{Perplexity of LLaMA-130M.}
	\label{Performance on Language Model-LLaMA130M-appendix}
\end{subfigure}
\begin{subfigure}[b]{0.48\textwidth}
	\centering
	\includegraphics[width=\textwidth]{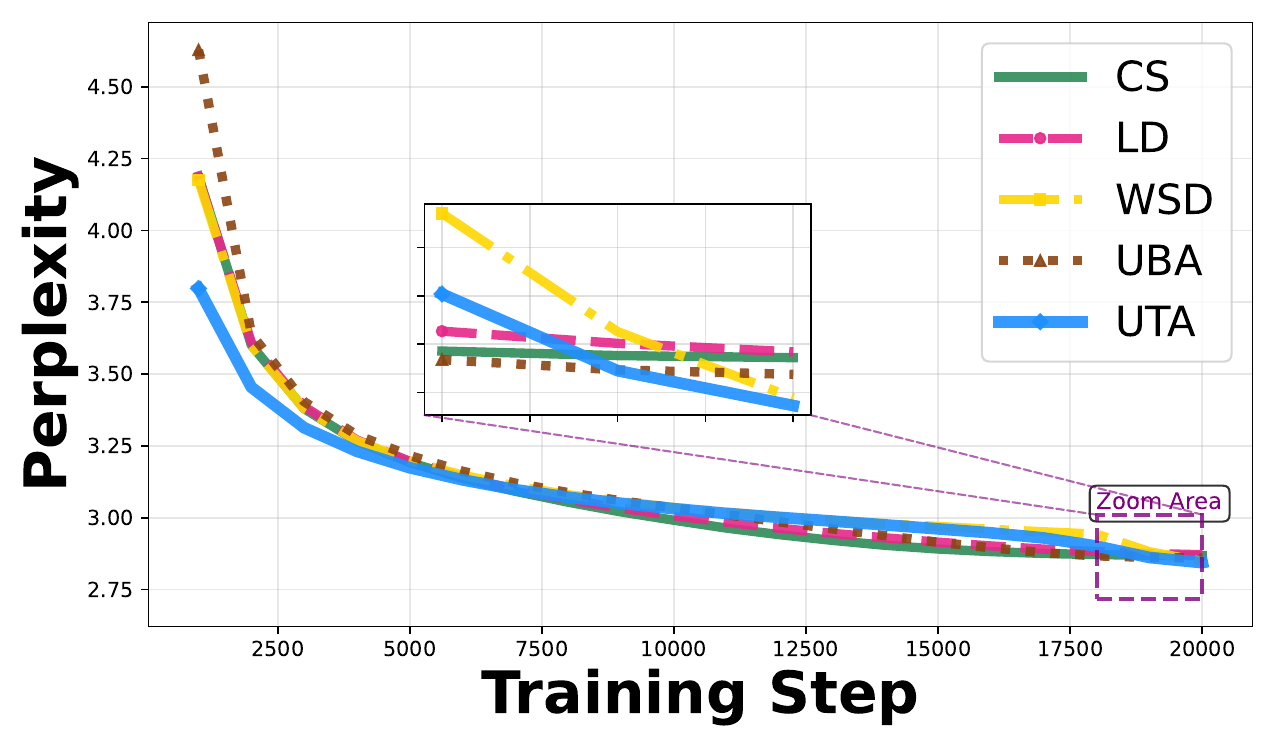}
	\caption{Perplexity of LLaMA-350M.}
	\label{Performance on Language Model-LLaMA350M-appendix}
\end{subfigure}
\begin{subfigure}[b]{0.48\textwidth}
	\centering
	\includegraphics[width=\textwidth]{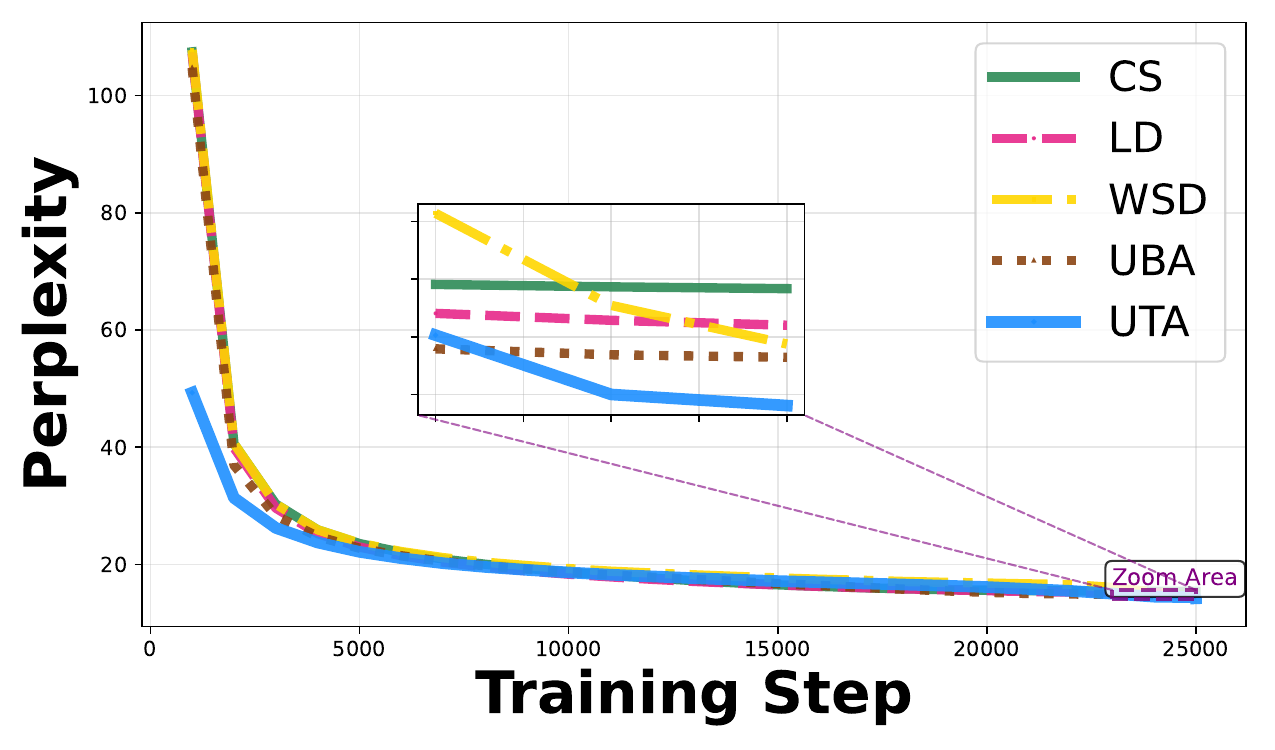}
	\caption{Perplexity of LLaMA-350M.}
	\label{Performance on Language Model-LLaMA1B-appendix}
\end{subfigure}
\caption{Perplexity curves of LLaMA models.}
\label{perplexity-llama-appendix}
\end{figure}

%

\begin{table}[!ht]
\centering
\small
\caption{Perplexity ($e^{\text{validation loss}}$) on LLaMA models. Lower is better (\textdownarrow).  `D-10' and `D-0' denote final learning rates decaying to 10\% and 0\% of the peak rate, respectively. }\label{llama-model-Perplexity-appendix-min_lr_ratio_0}

\setlength{\tabcolsep}{4pt}
\renewcommand{\arraystretch}{1.1}
\rowcolors{5}{gray!8}{white}
\begin{tabular}{@{}lcccc@{}}
	\toprule
	\multirow{2}{*}{LRD} & \multicolumn{4}{c}{LLaMA Model Size} \\
	\cmidrule(lr){2-5}
	& \multicolumn{2}{c}{60M} & \multicolumn{2}{c}{130M}  \\
	& D-10 & D-0 & D-10 & D-0  \\
	\midrule
	CS & 28.03 & 28.27 & 21.76 & 21.99 \\
	LD & 27.99 &   28.21 & 21.43 & 21.89 \\
	WSD & 26.03 &  26.93 & 20.83 &  21.18 \\
	UBA & 27.01 & 27.23  & 21.12 & 21.00 \\
	\rowcolor{cyan!8} UTA (Ours)& \textbf{25.80} & \textbf{26.77}   & \textbf{20.16} & \textbf{20.68} \\
	\bottomrule
\end{tabular}
\end{table}

%

\subsection{Performance on ImageNet}

We further validate the generalization of UTA on the ImageNet classification benchmark using ResNet-50 and ViT-B architectures. This selection tests generalization across model depth, width, and fundamental architectural biases. 
Tab.~\ref{imagenet-tab-appendix} reports the test top-1 accuracy. UTA consistently achieves the highest accuracy across both models. 
Specifically, it improves ResNet-50 accuracy to 78.73\%, outperforming the strongest baseline UBA (78.38\%) by 0.35\%. For ViT-B, UTA reaches 82.30\%, surpassing the competitive UBA LRD (82.05\%) by 0.25\%. 
These gains are statistically meaningful given the scale and complexity of ImageNet
These consistent outperformance demonstrate that UTA generalizes beyond language modeling to vision tasks. It has robust and architecture-agnostic advantage in large-scale, real-world training scenarios. and generalizes beyond language modeling to vision tasks.
\begin{table}[!ht]
\centering
\small
\caption{ImageNet test top-1 Acc. (\%) of different LRDs on ResNet50 and ViT-B under 300 epochs. 
	Higher values indicates better performance (\textuparrow). Best results in bold.}\label{imagenet-tab-appendix}

\renewcommand{\arraystretch}{1.15}
\rowcolors{3}{gray!8}{white}
\begin{tabular}{lcc}
	\toprule
	LRD & ResNet-50 & ViT-B \\
	\midrule
	CS & 78.34$_{\pm 0.024}$ & 79.92$_{\pm 0.057}$ \\
	LD & 77.90$_{\pm 0.028}$ & 79.96$_{\pm 0.104}$ \\
	WSD & 78.30$_{\pm 0.024}$ & 81.14$_{\pm 0.048}$ \\
	UBA & 78.38$_{\pm 0.014}$ & 82.05$_{\pm 0.041}$ \\
	\rowcolor{cyan!8} UTA (Ours) & \textbf{78.73}$_{\pm 0.029}$ & \textbf{82.30}$_{\pm 0.037}$ \\
	\bottomrule
\end{tabular}

\vspace{-0.5em}
\end{table}

We also plot the training loss and validation accuracy across training steps for LRDs on ImageNet using ResNet‑50 and ViT‑B models in Fig. \ref{fig:multi-vision-appendix}. A key and consistent observation is that while our proposed UTA LRDr does not always achieve the lowest training loss, it attains the highest final validation accuracy. 
Specifically, the training loss curves in Fig.  \ref{fig:multi-vision-appendix}(a) and (c) show that, unlike empirically tuned LRDs, such as CS and LD, that aggressively push training loss to very low values, UTA maintains a steady descent accompanied by moderate fluctuations. Its training loss is never the absolute minimum. This is not a sign of inadequate optimization, but a direct consequence of the robust‑optimization design of UTA \eqref{min-max-adam0}. Derived from a min‑max framework that minimizes the learning rate impact under worst-case conditions, UTA implicitly introduces a regularization effect that prevents over‑fitting to the training data. Consequently, although training loss is not minimized as aggressively, UTA consistently surpasses all baselines on validation accuracy as training progresses, demonstrating enhanced generalization rooted in its robust formulation.

Moreover, the learning rate trajectory of UTA exhibits noticeable fine‑grained fluctuations, more pronounced than those of static or predefined LRDs. These fluctuations are not noise. They reflect a state‑aware, active adaptation mechanism that responds to the internal dynamics of training. By dynamically adjusting the learning rate according to local landscape stability, UTA achieves a more robust optimization trajectory, ultimately leading to superior generalization performance.

\begin{figure}[htbp]
\label{Validation Accuracy on ImageNet-ResNet50 appendix}
\centering
\begin{subfigure}[b]{0.48\textwidth}
	\centering
	\includegraphics[width=\textwidth]{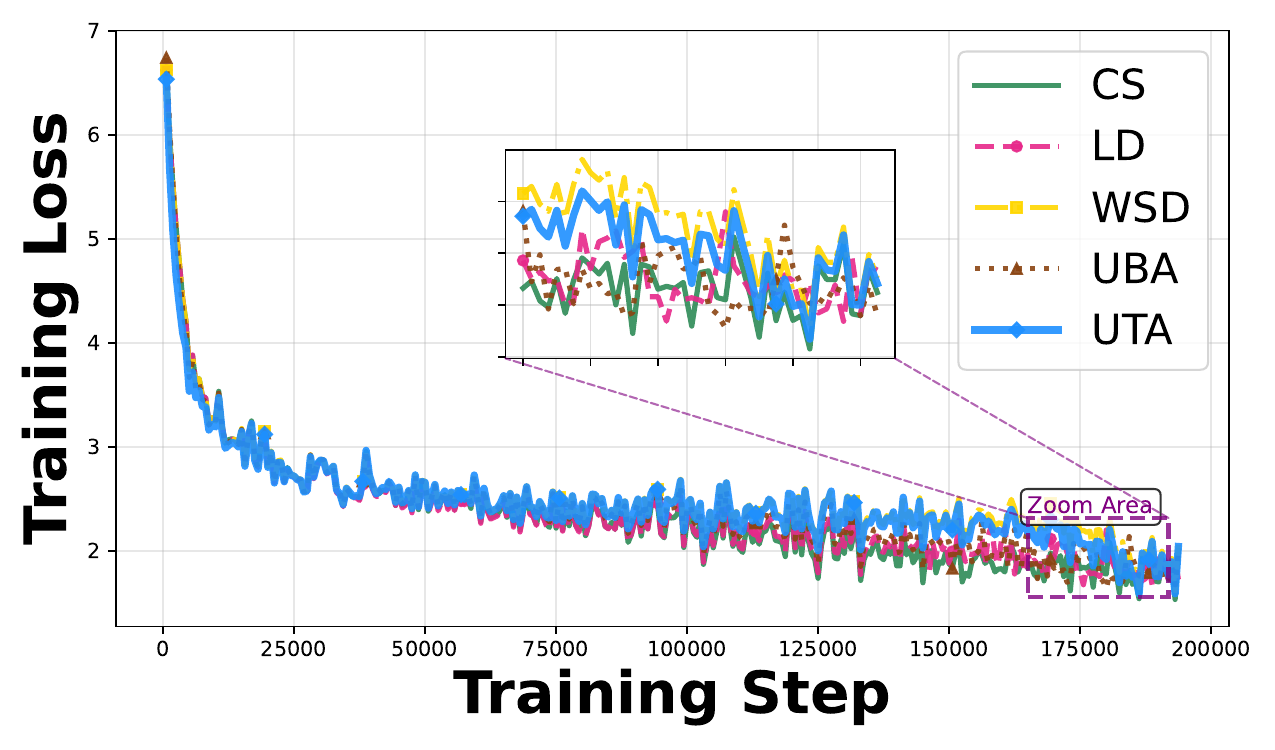}
	\caption{Training Loss on ImageNet-ResNet50}
	\label{fig:sub1-appendix}
\end{subfigure}
\begin{subfigure}[b]{0.48\textwidth}
	\centering
	\includegraphics[width=\textwidth]{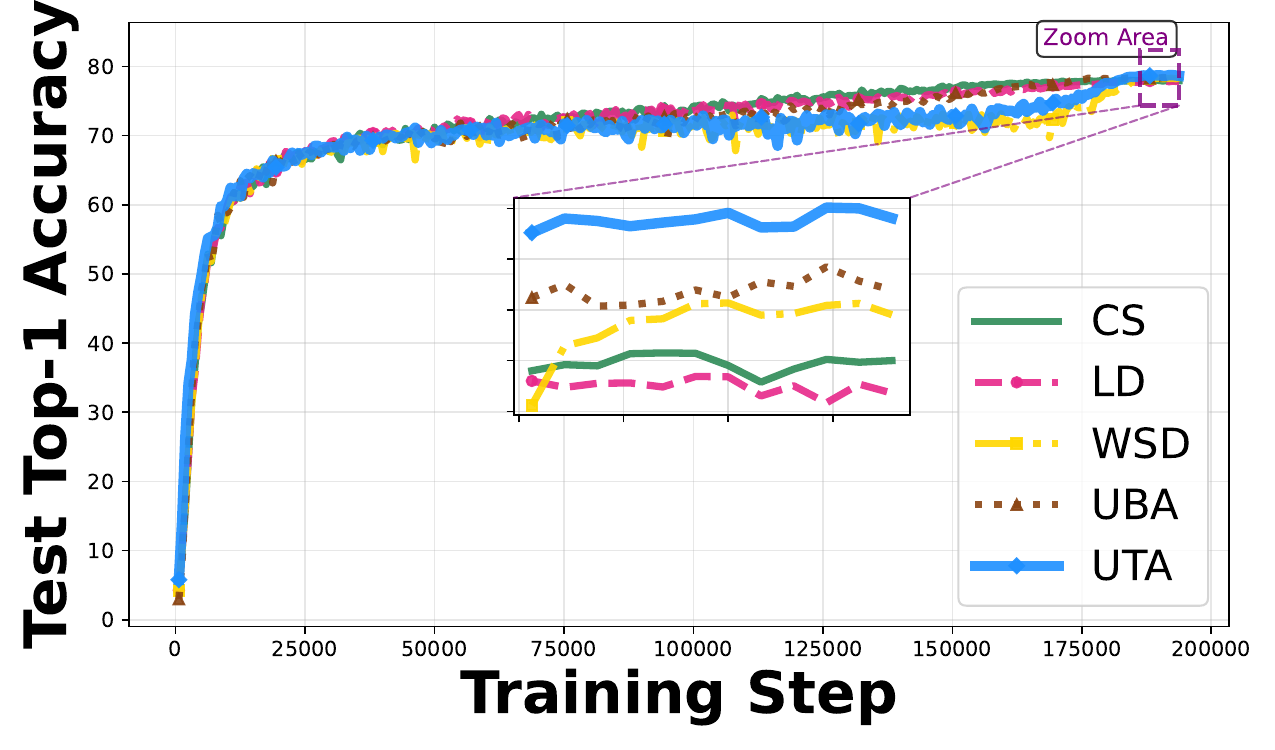}
	\caption{Validation Accuracy on ImageNet-ResNet50.}
	\label{fig:sub2-appendix}
\end{subfigure}
\begin{subfigure}[b]{0.48\textwidth}
	\centering
	\includegraphics[width=\textwidth]{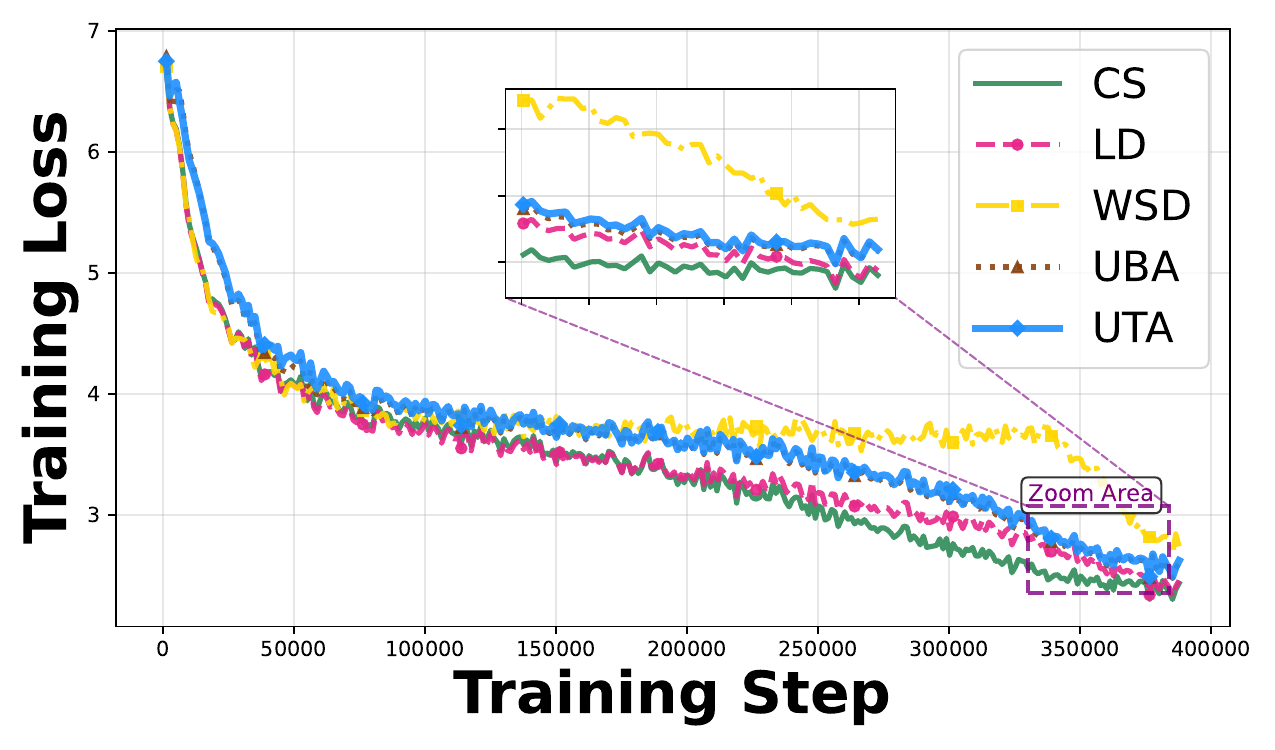}
	\caption{Training Loss on ImageNet-ViT-B.}
	\label{fig:sub3-appendix}
\end{subfigure}
\begin{subfigure}[b]{0.48\textwidth}
	\centering
	\includegraphics[width=\textwidth]{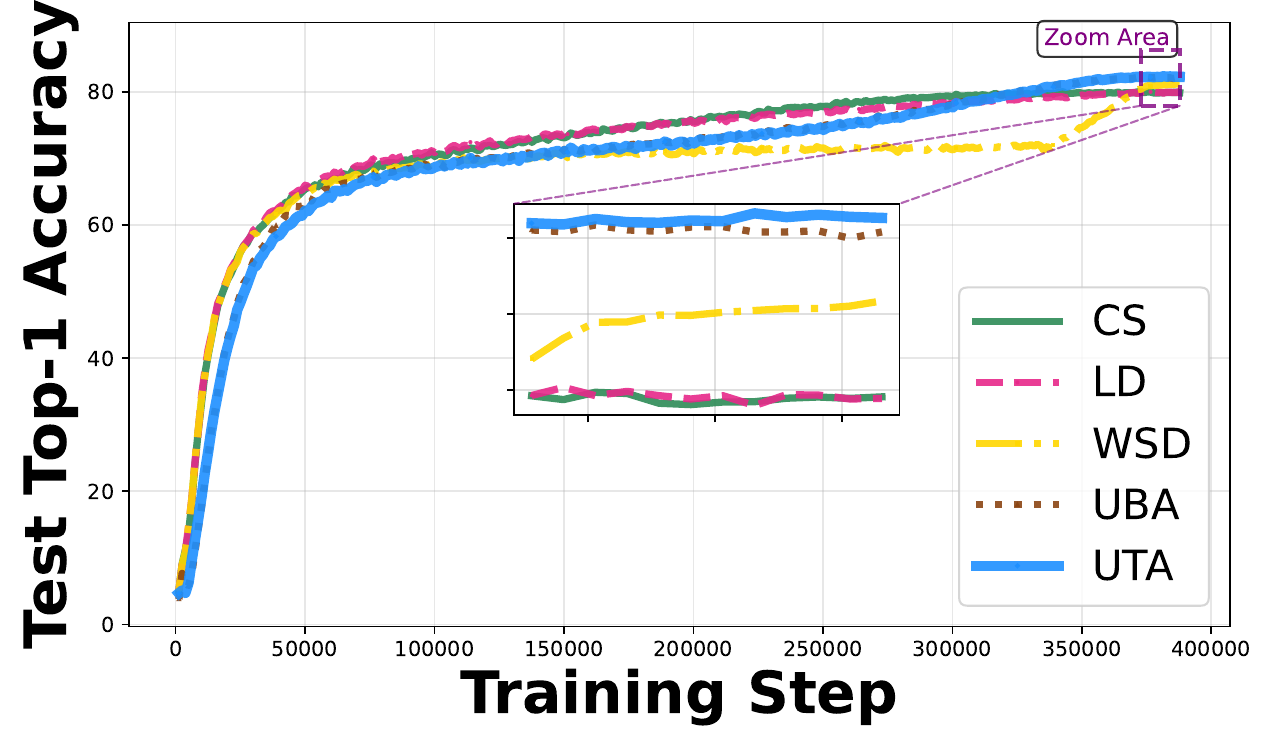}
	\caption{Validation Accuracy on ImageNet-ViT-B.}
	\label{fig:sub4-appendix}
\end{subfigure}
\caption{Training loss and validation accuracy on ImageNet using ResNet50 and ViT-B.}
\label{fig:multi-vision-appendix}
\end{figure}

\subsection{Ablation study}\label{Ablation-Study123}

\paragraph{Peak LR sensitivity analysis.}\label{}

To make a fair comparison, our main experiments used a fixed peak LR across all LRDs. Although these configurations are widely adopted defaults that have been empirically validated by prior works, to prevent peak LR from disadvantaging certain baselines if their optimal peak LRs differ, we conduct a systematic learning rate sweep on LLaMA-60M and 130M models. For 60M, we sweep peak LRs from \(5\times10^{-4}\) to \(1\times10^{-2}\); for 130M, from \(1\times10^{-3}\) to \(7\times10^{-3}\). All LRDs are trained under the same step budgets as in the main experiments, and we report perplexity (\(e^{\text{validation loss}}\)) for direct comparison. Results are summarized in Tab.~\ref{tab:llama-60m-130m-full-lr}.

First, \textbf{Optimal peak LRs differ across LRDs.} As expected, we observe that different LRDs indeed favor different peak learning rates. On LLaMA-60M, CS, LD and UBA achieve their best at \(7\times10^{-3}\), WSD at \(1\times10^{-3}\), while UTA achieves its best jointly at \(1\times10^{-3}\) and \(7\times10^{-3}\) with a perplexity of \(25.80\). On LLaMA-130M, CS and LD peak at \(5\times10^{-3}\), WSD also at \(5\times10^{-3}\), whereas UBA and UTA achieve their best at \(3\times10^{-3}\) and \(1\times10^{-3}\), respectively. This confirms that a fixed LR across all LRDs is not ideal and underscores the necessity of per-LRD tuning for fair comparison.

Second, \textbf{UTA achieves the best overall performance under its optimal LR.} When comparing the best result per LRD (reported in the ``Best from LR Sweep'' column), UTA consistently outperforms all baselines on both model scales. On 60M, UTA reaches a best perplexity of \textbf{25.80}, surpassing the best baseline WSD (26.03) by \(0.23\) and UBA (26.26) by \(0.46\). On 130M, UTA achieves \textbf{20.16}, outperforming WSD (20.82) by \(0.66\) and UBA (20.78) by \(0.62\). These margins, while modest in absolute terms, are consistent and non-trivial in the context of language modeling validation loss.

\paragraph{Parameter analysis of $\varphi_0$}\label{Ablation-Study-d-appendix}

\begin{table}[!t]
\centering
\footnotesize
\caption{Cross-$\varphi_0$ ablation studies: Perplexity of LLaMA-60M using UTA under different $\varphi_0$.}\label{Cross-d-table-appendix1}
\resizebox{\linewidth}{!}{%
	\begin{tabular}{lllllllllll}
		\toprule
		\multirow{2}{*}{Network}&    \multicolumn{10}{c}{ UTA }\\
		\cmidrule(r){2-11}
		&$\varphi_0$=0.01& $\varphi_0$=0.05 &$\varphi_0$=0.1&$\varphi_0$=0.15&$\varphi_0$=0.2 &$\varphi_0$=0.25&  $\varphi_0$=0.5 & $\varphi_0$=1  & $\varphi_0$=2.5 & $\varphi_0$=5 \\
		\midrule
		LLaMA-60M& 26.18 &25.94 & 25.85& 25.81& 25.80 & 25.80 &  25.92 & 26.04  &  26.39 & 26.78\\
		\bottomrule
	\end{tabular}
}
\end{table}

\begin{table}[!t]
\centering
\footnotesize
\caption{Cross-$\varphi_0$ ablation studies: Perplexity of LLaMA-130M using UTA under different $\varphi_0$.}\label{Cross-d-table-appendix2}
	\begin{tabular}{llllll}
		\toprule
		\multirow{2}{*}{Network}&    \multicolumn{5}{c}{ UTA }\\
		\cmidrule(r){2-6}
		&$\varphi_0$=0.15&$\varphi_0$=0.2 &$\varphi_0$=0.25&  $\varphi_0$=0.5  & $\varphi_0$=2.5  \\
		\midrule
		LLaMA-130M& 20.63 &20.65 & 20.66& 20.67& 20.99 \\
		\bottomrule
	\end{tabular}
\end{table}


In this subsection, we perform a sensitivity analysis of the hyper-parameter $\varphi_0 \in \{0.01,0.05,0.1,0.15,0.2 0.25, 0.5, 1.0, 2.5,5\}$ in   \eqref{varphi_0} , which controls the variation speed of learning rate. 
To efficiently identify general parameter trends at low cost, then verify their scalability and robustness on larger models.
We further examine $\varphi_0 \in \{0.1,0.15,0.2 0.25,2.5\}$ on a larger 130M model. We narrow the range of $\varphi_0$ to focus computational resources on the most promising intervals, filtering out ineffective values identified in the initial broad search.
The results on LLaMA-60M listed in Tab. \ref{Cross-d-table-appendix1}. 
The results on LLaMA-130M listed in Tab. \ref{Cross-d-table-appendix2}). 

Our sensitivity analysis reveals three key findings. \textbf{First, UTA exhibits remarkable robustness to $\varphi_0$ over a wide range.} On LLaMA-60M, UTA achieves perplexity below 26.0 for \textbf{7 out of 10} tested $\varphi_0$ values (0.05 through 0.5), and the performance variation across this broad range is remarkably small, spanning only 25.80 to 25.92. Even at the extremes ($\varphi_0=0.01$ and $\varphi_0=5$), the degradation remains modest (26.18 and 26.78, respectively). On LLaMA-130M, UTA maintains perplexity between 20.63 and 20.67 across $\varphi_0 \in [0.15, 0.5]$, which is a variation of merely 0.04, and only degrades to 20.99 at $\varphi_0=2.5$. This \textbf{flat response surface} indicates that UTA's performance is largely insensitive to the choice of $\varphi_0$, a desirable property for practical deployment where extensive hyper-parameter tuning is often computationally prohibitive.

Second, \textbf{the performance trend aligns with our theoretical interpretation.} As $\varphi_0$ decreases from 0.5 to 0.15, performance improves slightly (from 25.92 to 25.80 on 60M, and from 20.67 to 20.63 on 130M). This is consistent with our continuous-time analysis: a smaller $\varphi_0$ corresponds to a larger intrinsic step $U^{\ast}$, enabling more thorough exploration of the loss landscape. However, the gains saturate beyond $\varphi_0=0.15$, suggesting that once the intrinsic step is sufficiently large, further expansion yields diminishing returns—a phenomenon that aligns with the diminishing sensitivity predicted by our theoretical framework.

Third, and most importantly, \textbf{UTA demonstrates significantly superior robustness compared to the baseline UBA.} While both methods share a similar functional form of the learning rate decay, UBA exhibits a much narrower effective range: its performance declines sharply for $\varphi_0 < 0.5$, and further reduction beyond this point brings little improvement or even degradation. In contrast, UTA maintains strong performance across a \textbf{10-fold wider range} of $\varphi_0$ (from 0.05 to 0.5 on 60M), with a performance spread of only 0.24, compared to UBA's much larger degradation outside its narrow optimal region. We attribute this advantage to UTA's learning rate adaptation mechanism, which dynamically adjusts the intrinsic time scale in response to optimization geometry, allowing the optimizer to fully exploit the expanded exploration capacity afforded by larger $U^{\ast}$ while automatically compensating when $\varphi_0$ is suboptimal. This adaptive coupling between the learning rate and the preconditioned gradient norm, which is absent in UBA, ensures that increased intrinsic time translates effectively into improved convergence and generalization, regardless of the specific $\varphi_0$ choice.

Overall, this sensitivity analysis confirms that UTA is not only performant but also \textbf{robust and easy to use}: it achieves near-optimal performance across a broad spectrum of $\varphi_0$ values, substantially reducing the burden of hyper-parameter tuning compared to both UBA and other schedule-based methods.

%

\paragraph{The effect of final LR in decay} Regarding the LRD  implementations, some practices use decay to 10\% of the peak LR, some use 0\%. To investigate whether this choice affects performance. We conduct this ablation on LLaMA-60M and 130M models, comparing both settings for all LRD methods. Results are reported in Tab.~\ref{llama-model-Perplexity-appendix-min_lr_ratio_0} in Appendix. 

We observe that decaying to 0\% consistently yields worse perplexity than decaying to 10\% across nearly all methods and model scales. For instance, on LLaMA-130M, linear decay degrades from 21.43 to 21.89 (0.46), and even our UTA increases from 20.16 to 20.68 (0.52) when using 0\% decay. Aggressively decaying the learning rate to zero prematurely reduces the model's ability to adapt, as the optimizer enters a near-stagnant regime before the model has fully converged. 

Importantly, regardless of the decay ratio, UTA consistently outperforms all baselines. Under the 0\% setting, UTA achieves 26.77 on 60M and 20.68 on 130M, still surpassing the best baselines (WSD at 26.93 and UBA at 21.00) by margins of 0.16 and 0.32, respectively. This confirms that our main conclusions are robust to the choice of final learning rate, and that the default 10\% decay setting does not unfairly advantage UTA.

\paragraph{The role of dynamic state-aware mechanism}\label{ablation-Dynamic State-Aware Mechanism-appendix}

To validate the contribution of the dynamic mechanism in our UTA LRD, we conduct an ablation study comparing the full version (Ours-Full) against a static variant (Ours-Static). We conduct the ablation study on the CIFAR-100 dataset with a ResNet-34 architecture. 
The CIFAR-100 dataset with ResNet-34 provides a balanced testbed. It is sufficiently complex to reveal LRD differences yet computationally tractable for ablation studies. We intentionally keep all other factors identical to isolate the dynamic mechanism's effect.
In the static variant, the dynamic part is removed. All other hyper-parameters and training settings remain identical.

\begin{table}[!ht]
	\centering
	\small
	\caption{Ablation Study: Impact of Dynamic State-Aware Mechanism.}\label{ablation-tab-1-appendix}
	\setlength{\tabcolsep}{4pt}
	\renewcommand{\arraystretch}{1.15}
			\begin{tabular}{cccc}
				\toprule
				& Ours-Static & Ours-Full & Improvement\\
				
				\midrule
				CIFAR100-ResNet34(\textuparrow)   &76.86    &  77.81  &  0.95 \\
				ImageNet-ResNet50(\textuparrow)  &78.23    &  78.55   &   0.32 \\
				LLaMA60M(\textdownarrow) & 27.38 &   26.78  &   0.60 \\
				LLaMA130M(\textdownarrow) & 21.02 &   20.15  &  0.87  \\
				\bottomrule
			\end{tabular}
	\end{table} \vspace{-0.5em}
	
	\paragraph{Comparison with automatic LR methods.}
	While our primary baselines focus on LRDs, another line of LR research is automatic LR methods such as D-Adaptation~\citep{defazio2023learning} and Prodigy~\citep{mishchenko2023prodigy}. These methods dynamically adjust the learning rate during training \textit{without predefined schedules}. Although these methods are not the main focus of our work, we conduct a comparison to position UTA against them.
	
	Specifically, we evaluate D-Adaptation and Prodigy on LLaMA-60M, 130M, and 350M under a fixed 10K-step training budget, following their default hyperparameter configurations. Results are reported in Tab.~\ref{tab:vs-auto-lrs-appendix}. As shown, UTA consistently outperforms both automatic schedulers across all three model scales, achieving average perplexity reductions of 1.2 and 0.9 over D-Adaptation and Prodigy, respectively. We hypothesize that UTA's advantage stems from its schedule-aware shape design, which explicitly accounts for the remaining training horizon, whereas automatic methods rely solely on gradient statistics and may lack such long-term planning capability.
	
	Given the clear superiority of UTA and the high computational cost of scaling automatic methods to larger models, we exclude the 1B model from this comparison and focus our main experiments on carefully designed schedule baselines.
	
	\begin{table}[!ht]
		\centering
		\small
		\caption{Perplexity ($e^{\text{validation loss}}$) comparison between automatic LR methods (D-Adaptation, Prodigy) and our UTA on LLaMA models. Lower is better ($\downarrow$).}
		\setlength{\tabcolsep}{5pt}
		\renewcommand{\arraystretch}{1.15}
		\label{tab:vs-auto-lrs-appendix}
		\begin{tabular}{@{}lccc@{}}
			\toprule
			Method 		& 60M & 130M & 350M \\
			\midrule
			D-Adaptation&  25.80 & 20.65 & 17.64 \\
			Prodigy & 27.37 & 21.26 & 17.97 \\
			UTA (Ours)& \textbf{25.80} & \textbf{20.16} & \textbf{17.17} \\
			\bottomrule
		\end{tabular}
	\end{table}	
	
	\paragraph{Curvature alignment between UTA and UBA.}\label{Curvature-alignment-UTA-UBA-appendix}
	
	We observe that UTA belongs to UBA family, and the effective scale of UTA hyperparameter is modulated by the RMS norm of the preconditioned update (approximately 0.2) \citep{liu2025muon,kosson2023rotational}, whereas UBA’s hyperparameter $\varphi$ acts directly. To isolate and evaluate the contribution of UTA’s state‑adaptive mechanism—independent of hyperparameter scaling—we conduct a parameter‑alignment experiment. This ensures that the underlying trend of the learning‑rate curve is comparable between UTA and UBA, allowing us to directly attribute performance differences to UTA’s dynamic fluctuations.

	Specifically, we align the effective learning‑rate scales by setting $\varphi$ for UBA and $\varphi_0$ for UTA such that $\varphi\approx \varphi_0 \times 0.2 $. We then train models under identical settings and compare both the macroscopic learning rate trajectories and the final performance. The results are listed in Tab. \ref{ablation-tab-2-appendix} and plotted in Fig. \ref{Curvature alignment between UTA and UBA}.
	
	The results confirm that after parameter alignment, UTA and UBA produce LRD with nearly identical macroscopic shapes. However, UTA consistently achieves superior final accuracy. This performance gain can be directly attributed to the fine‑grained, state‑dependent fluctuations unique to UTA, which are absent in the static UBA curve. These fluctuations act as a form of online regularization, allowing the optimizer to navigate the loss landscape more robustly. 
	
	\begin{table}[!ht]
		\centering
		\small
		\caption{Curvature alignment between UTA and UBA.}\label{ablation-tab-2-appendix}
		\setlength{\tabcolsep}{4pt}
		\renewcommand{\arraystretch}{1.15}
				\begin{tabular}{cccc}
					\toprule
					& UBA($\varphi=0.03$) & UTA($\varphi_0=0.15$) & Improvement\\
					
					\midrule
					ImageNet-ResNet50(\textuparrow)  &78.47   &  78.64   &   0.17 \\
					LLaMA130M(\textdownarrow) & 21.08 &   20.15  &  0.93  \\
					LLaMA1B(\textdownarrow) & 14.81 &   14.53  &  0.28  \\
					\bottomrule
				\end{tabular}
		\end{table} \vspace{-0.5em}

		\begin{figure}[htbp]
			\centering
			\begin{subfigure}[b]{0.88\textwidth}
				\centering
				\includegraphics[width=\textwidth]{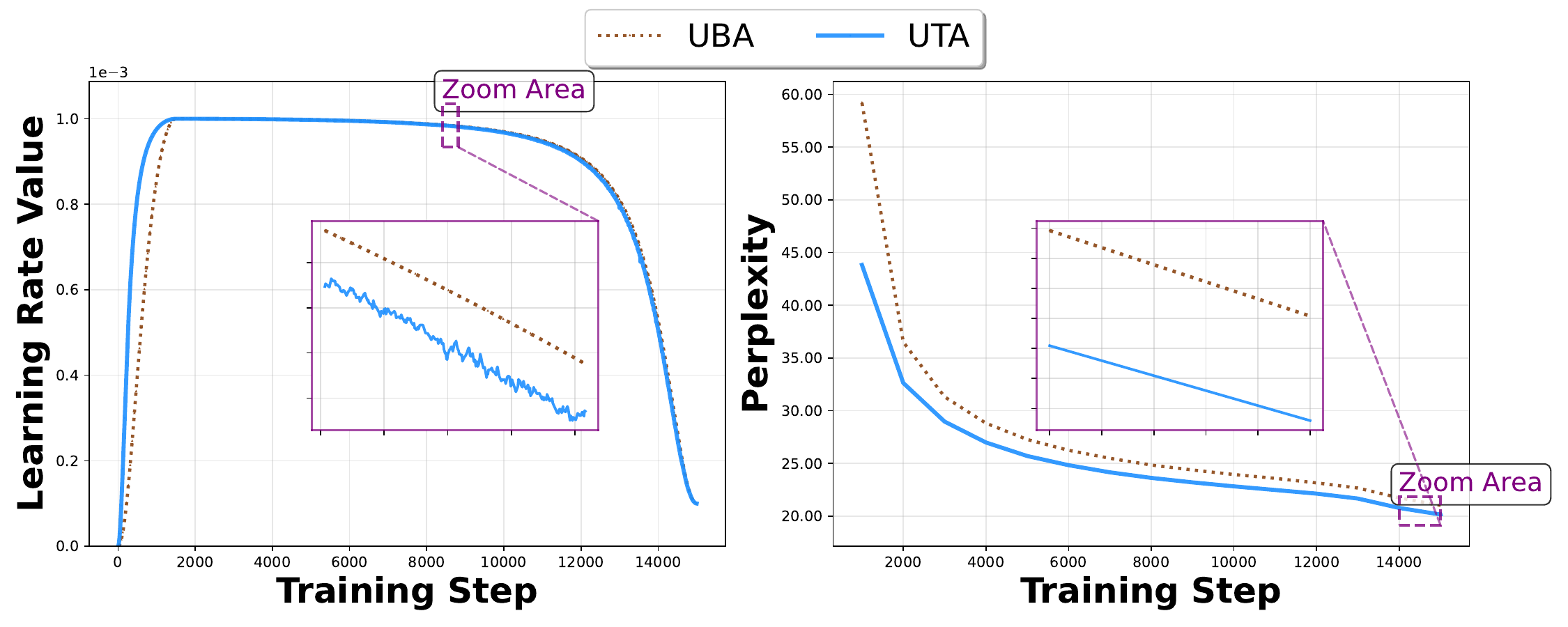}
				\caption{LLaMA-130M.}
				\label{UBA-UTA-130M-3e-3}
			\end{subfigure}
			\begin{subfigure}[b]{0.88\textwidth}
				\centering
				\includegraphics[width=\textwidth]{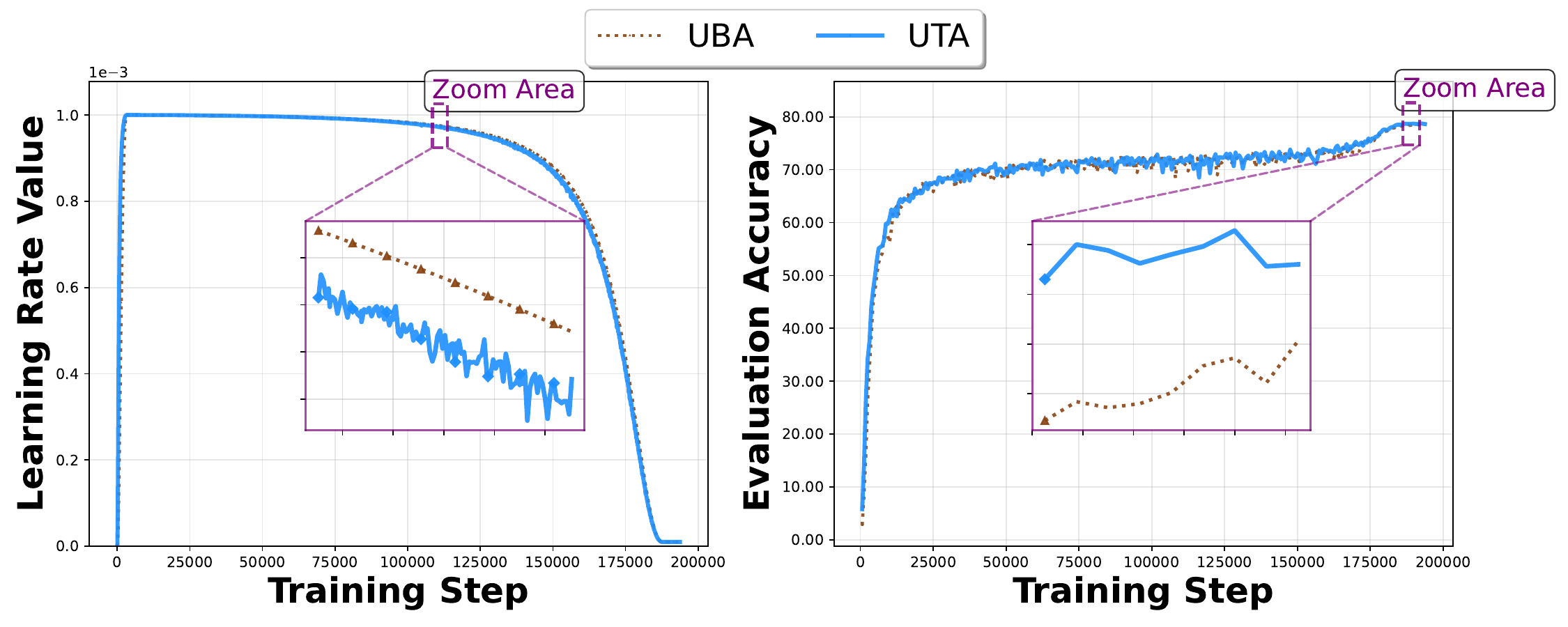}
				\caption{ImageNet-ResNet50.}
				\label{UBA-UTA-imagenet-resnet50-3e-3}
			\end{subfigure}
			\caption{Curvature alignment between UTA and UBA.}
			\label{Curvature alignment between UTA and UBA}
		\end{figure}

		\paragraph{Generalization to parameter-efficient fine-tuning}\label{PEFT-scenario-appendix}
		
		
		To demonstrate the versatility of UTA beyond full-parameter optimization, we evaluate it in parameter-efficient fine-tuning (PEFT) scenarios. PEFT presents distinct challenges: only a small subset of parameters (e.g., LoRA adapters) are updated, creating a different optimization landscape where learning-rate scheduling may have amplified effects.
		
		Experimental Setup. We fine-tune a pre-trained ViT-B model on two datasets: (1) \textbf{CIFAR-100} (generic object classification) and (2) \textbf{OxfordPets} (fine-grained pet breed classification). We employ \textbf{LoRA} with rank $r=16$ applied to the query, key and value projections in attention layers, and dense layer. We compare UTA against the baseline UBA LRD.
		All methods use the same total training budget (10 epochs), batch size, and optimizer AdamW setting such as 0.01 weight decay.

		
		\begin{table}[!ht]
			\centering
			\small
			\caption{UTA for LoRA with ViT-B.}\label{ablation-tab-lora-appendix}
			\setlength{\tabcolsep}{4pt}
			\renewcommand{\arraystretch}{1.15}
			\begin{tabular}{ccccccc}
				\toprule
				\cmidrule{2-5}
				Network & Dataset  & CS & LD & WSD & UBA & UTA \\
				\midrule
				\multirow{2}{*}{ViT-B}& CIFAR100 & 91.93 & 92.45 & 92.62 &  92.58 & \textbf{92.81}   \\
				& OxfordPets   & 92.59 & 92.67& 92.86 &  93.05 &  \textbf{93.43}\\
				\bottomrule
			\end{tabular}
		\end{table} \vspace{-0.5em}

		\subsection{Discussion of UTA characteristics}
		To illustrate the characteristics of UTA in practice, we plot its learning rate trajectory during the training of ImageNet-ResNet50 in Fig.  \ref{learning-rate-figure-appendix}.
		
		Three key observations emerge. 
		
		(i) \textbf{Overall Trajectory.} The general shape of UTA closely resembles that of the widely adopted WSD. It exhibits a prolonged stable‑like phase followed by decay. Notably, however, the duration of the stable phase and and the decay rate are automatically determined by UTA, rather than being preset by hyperparameters. 
		
		(ii) Although the average learning rate of WSD is higher than that of UTA, UTA consistently achieves better performance. This confirms that the performance gains are not merely due to a higher average learning rate, but stem from UTA being near-optimal in both shape and scale, a property grounded in the preceding methodological and theoretical analysis.
		
		(iii) \textbf{Training Adaptive Fluctuations.} Unlike most conventional LRDs that follow a static, predefined curve, UTA introduces fine‑grained fluctuations in the learning rate. These adjustments are driven by real‑time feedback from the training state, reflecting the optimizer’s internal dynamics. This responsiveness demonstrates UTA’s strong adaptive capability in practical training scenarios.	
		
		\begin{figure}[!t]
			\centering
			{
				\includegraphics[width=0.8\linewidth]{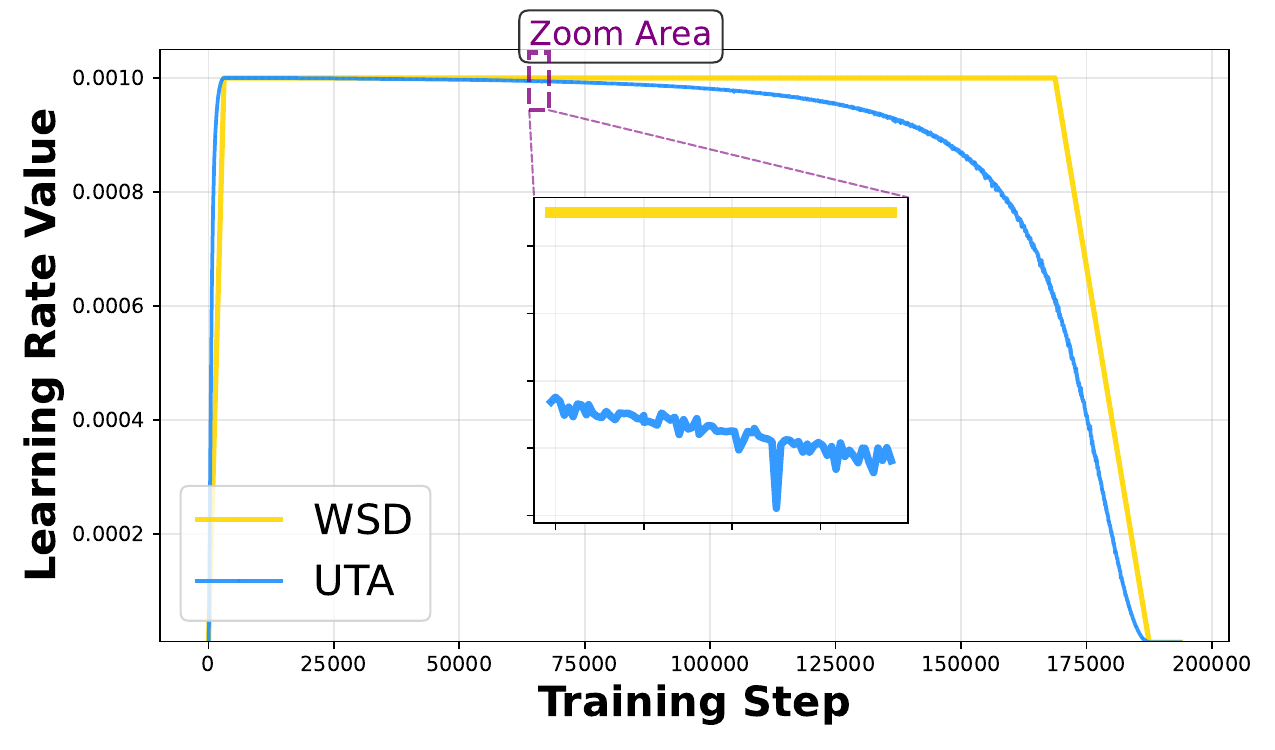}
				
			}
			\caption{The learning rate curve in UTA LRD across training steps.}
			\label{learning-rate-figure-appendix} 
		\end{figure}

		\section{Related work}\label{related-work-appendix}
		
		The LRD plays a crucial role in deep learning, and prior research has investigated a wide range of approaches. 
		Step LRD is a typical instance that reduces the learning rate by a decay factor after a predefined number of epochs \citep{he2016deep}. 
		A common enhancement is the warmup phase, which gradually raises the learning rate early in training to stabilize optimization \citep{he2016deep}.
		Then Loshchilov et al. \citep{loshchilov2017sgdr} propose the cosine LRD, which is widely regarded as the de-facto  standard choice in large-scale model training \citep{loshchilov2017sgdr,schaipp2025surprising}.
		To improve the performance under budgeted training scenarios, some LRDs include the cyclical LRD \citep{smith2017cyclical}, alternative LRD under budgeted training setting \citep{li2019budgeted} and REX LRD \citep{chen2022rex}, yet there is no consensus on the optimal choice.
		
		Recently, research in this area has evolved along two main lines. 
		A line of work focus on designing LRD primarily through empirically driven approaches, and effectively enhance accuracy.
		Pan et al. \citep{pan2022extremebert} introduce elastic step decay LRD to accelerate the pre-training for Berts.
		As an alternative on large model training, Warmup-Stable-Decay (WSD) LRD \citep{zhai2022scaling,hu2024minicpm} uses an extended constant learning rate phase before late-stage decay. It has become popular in LLM pre-training \citep{hagele2024scaling,hu2024minicpm} and is deployed in training industrial-scale LLMs such as DeepSeek-V3 \citep{liu2024deepseek} and Kimi-K2 \citep{team2025kimi}. Yet its mechanism remains unclear, recent works \citep{wen2024understanding,schaipp2025surprising,li2025functional} offer some insights, showing that the advantage of WSD can be explained and quantified through scaling efficiency comparisons with other LRDs.

		Another line of research is more theoretically driven, seeking to establish principled frameworks and scaling laws for designing or analyzing LRD.
		Pan et al. \citep{pan2022eigencurve} propose Eigencurve, and demonstrate that it has minimax-optimal convergence on skewness Hessian quadratic objectives.
		Defazio et al. \citep{defazio2023optimal} establish optimal linear decay LRD and refinements through rigorous mathematical derivation. 
		Some studies \citep{tissue2025scaling,li2025functional,luo2025multi} establish scaling law with LRD, providing theoretical explanation for empirical results observed in this area. 
		Recently, Tang et al. propose a theoretically grounded LRD called UBA LRD \citep{tang2026stepsize}, derives from a constructed UBA framework. This framework adopts a min‑max formulation that explicitly accounts for robustness to variations in loss landscape curvature, thereby guaranteeing performance under worst‑case conditions. The UBA LRD eliminates the need for heuristic design while maintaining strong and consistent performance across diverse tasks, architectures, model scales, and training budgets. Consequently, it removes the costly trial‑and‑error process typically required to select an appropriate LRD for a given network and training setup.

\end{document}